\documentclass[10pt]{article}
\usepackage{amsmath}
\usepackage[utf8]{inputenc}

\newcommand{\umax}{\|u\|_{L^\infty(\mu)}}
\newcommand{\Linfnorm}[2]{\norm{#1}_{L^\infty(#2)}}
\newcommand{\rmax}{r_{\mathrm{max}}}
\usepackage{bbm}
\usepackage{caption}
\usepackage{rotating}
\usepackage[maxfloats=100]{morefloats}
\usepackage{listings}
\usepackage{booktabs}
\usepackage{xcolor}
\usepackage[title]{appendix}
\usepackage{makecell} 
\usepackage{multirow}

\usepackage[symbol]{footmisc}

\usepackage{longtable}
\usepackage{subcaption}
\usepackage[numbers]{natbib}
\usepackage{geometry}
\usepackage{amssymb}
\usepackage{amsthm}
\usepackage{enumitem}
\usepackage{verbatim}
\usepackage{xcolor}
\usepackage{graphicx}
\usepackage{enumitem}
\usepackage{comment}
\usepackage{subfiles}
\usepackage{amsmath}
\usepackage{todonotes}
\usepackage{tikz-qtree,tikz-qtree-compat}
\usepackage{scrextend}
\usepackage{framed}
\usepackage{placeins}
\usepackage{setspace}
\usepackage{algorithmic}
\usepackage{algorithm}
\usepackage{mathtools,collcell,eqparbox}
\usepackage{ dsfont }
\usepackage{multicol}
\usepackage{pgfplots}
\usepackage{tikz}
\usepackage{listings}
\usepackage{courier}
\usepackage{url}
\usepackage{color}
\usepackage[title]{appendix}
\usepackage{authblk}
\usepackage[pagebackref = true]{hyperref}
\hypersetup{
    colorlinks = true,
    linkcolor = red,
    anchorcolor = blue,
    citecolor = blue,
    filecolor = blue,
    urlcolor = blue
    }

\pgfplotsset{compat=1.15}
\newcounter{BMatrix}

\newcommand{\setmaxwd}[1]{%
  \eqmakebox[BM-\theBMatrix][\BMalign]{$#1$}%
}
\MHInternalSyntaxOn

\MHInternalSyntaxOff
\makeatother

\newtheorem{theorem}{Theorem}
\newtheorem{assumption}{Assumption}

\newtheorem{proposition}{Proposition}[section]
\newtheorem{corollary_prop}{Corollary}[proposition]
\newtheorem{corollary}{Corollary}[theorem]
\newtheorem{lemma}{Lemma}
\theoremstyle{definition}
\newtheorem{definition}{Definition}
\theoremstyle{remark}
\newtheorem*{remark}{Remark}

\DeclareMathOperator*{\esssup}{ess\,sup}
\DeclareMathOperator*{\essinf}{ess\,inf}
\newcommand{\inv}{^{-1}}

\newcommand{\1}{\mathds{1}}

\newcommand{\R}{\mathbb{R}}

\newcommand{\Z}{\mathbb{Z}}

\newcommand{\var}{\mathrm{Var}}
\newcommand{\bias}{\mathrm{Bias}}

\newcommand{\cov}{\mathrm{Cov}}
\newcommand{\rhs}{\mathrm{RHS}}
\newcommand{\lhs}{\mathrm{LHS}}
\newcommand{\del}{\partial}
\newcommand{\da}{\downarrow}
\newcommand{\ra}{\rightarrow}

\newcommand{\ua}{\uparrow}

\newcommand{\cd}{\cdot}
\newcommand{\ds}{\dots}

\newcommand{\bd}[1]{\mathbf{#1}}
\newcommand{\mrm}[1]{\mathrm{#1}}

\newcommand{\cB}{\mathcal{B}}

\newcommand{\cF}{\mathcal{F}}

\newcommand{\cH}{\mathcal{H}}

\newcommand{\cL}{\mathcal{L}}
\newcommand{\cM}{\mathcal{M}}
\newcommand{\cN}{\mathcal{N}}

\newcommand{\cP}{\mathcal{P}}

\newcommand{\cT}{\mathcal{T}}

\newcommand{\fr}[1]{\mathfrak{#1}}

\newcommand{\spnorm}[1]{\left|#1\right|_\mathrm{span}}
\newcommand{\set}[1]{\left\{{#1}\right\}}

\newcommand{\ceil}[1]{\left\lceil{#1}\right\rceil}

\newcommand{\norm}[1]{\left\|#1\right\|}

\newcommand{\norminf}[1]{\left\|#1\right\|_{\infty}}

\newcommand{\abs}[1]{\left|#1\right|}
\newcommand{\sqbk}[1]{\left[ #1 \right]}
\newcommand{\sqbkcond}[2]{\left[ #1 \middle| #2 \right]}

\newcommand{\crbk}[1]{\left( #1 \right)}

\newcommand{\crbkcond}[2]{\left( #1 \middle| #2 \right)}

\newcommand{\bmx}[1]{\begin{bmatrix} #1 \end{bmatrix}}

\newcommand{\argmax}[1]{\underset{#1}{\operatorname{arg}\,\operatorname{max}}\;}

\definecolor{codegreen}{rgb}{0,0.4,0}
\definecolor{codeblue}{rgb}{0.1,0.1,0.7}
\definecolor{codegray}{rgb}{0.5,0.5,0.5}
\definecolor{codepurple}{rgb}{0.58,0,0.82}
\definecolor{backcolour}{rgb}{0.97,0.97,0.97}
 
\lstdefinestyle{mystyle}{
    backgroundcolor=\color{backcolour},   
    commentstyle=\color{codegreen},
    keywordstyle=\color{magenta},
    numberstyle=\tiny\color{codegray},
    stringstyle=\color{codepurple},
    basicstyle=\scriptsize\ttfamily,
    identifierstyle=\color{codeblue},
    breakatwhitespace=false,         
    breaklines=true,                 
    captionpos=b,                    
    keepspaces=true,                 
    numbers=left,                    
    numbersep=4pt,                  
    showspaces=false,                
    showstringspaces=false,
    showtabs=true,                  
    tabsize=3
}
\numberwithin{equation}{section}
\numberwithin{equation}{section}

\begin{document}

\title{Sample Complexity of \\
Variance-Reduced Distributionally Robust Q-Learning }
\author[1]{Shengbo Wang}
\author[2]{Nian Si}
\author[1]{Jose Blanchet}
\author[3]{Zhengyuan Zhou}
\affil[1]{MS\&E, Stanford University}
\affil[2]{IEDA, HKUST}
\affil[3]{Stern School of Business, New York University}
\date{August, 2024}

\maketitle

\begin{abstract}
Dynamic decision-making under distributional shifts is of fundamental interest in theory and applications of reinforcement learning:  The distribution of the environment in which the data is collected can differ from that of the environment in which the model is deployed. This paper presents two novel model-free algorithms, namely the distributionally robust Q-learning and its variance-reduced counterpart, that can effectively learn a robust policy despite distributional shifts. These algorithms are designed to efficiently approximate the $q$-function of an infinite-horizon $\gamma$-discounted robust Markov decision process with Kullback-Leibler ambiguity set to an entry-wise $\epsilon$-degree of precision.  Further, the variance-reduced distributionally robust Q-learning combines the synchronous Q-learning with variance-reduction techniques to enhance its performance. Consequently, we establish that it attains a minimax sample complexity upper bound of $\tilde O(|\bd{S}||\bd{A}|(1-\gamma)^{-4}\epsilon^{-2})$, where $\bd{S}$ and $\bd{A}$ denote the state and action spaces.  This is the first complexity result that is independent of the ambiguity size $\delta$, thereby providing new complexity theoretic insights. Additionally, a series of numerical experiments confirm the theoretical findings and the efficiency of the algorithms in handling distributional shifts.
\end{abstract}

\section{Introduction}
Reinforcement learning (RL) \citep{sutton2018reinforcement} focuses on how agents can learn to make optimal decisions in uncertain and dynamic environments. It is based on the principle of trial-and-error learning, where the agent interacts with the environment, receives rewards or penalties for its actions, and adjusts its behavior to maximize the expected long-term reward.

A significant obstacle in RL is the limited interaction between the agent and the environment, often due to factors such as data-collection costs or safety constraints. To overcome this, practitioners often rely on historical datasets or simulation environments to train the agent. However, this approach can suffer from distributional shifts \citep{quinonero2008dataset} between the real-world environment and the data-collection/simulation environment, potentially leading to suboptimal learned policies when deployed in the actual environment. It is also observed in RL environments that an agent trained this way could be vulnerable to adversarial attacts \citep{lin2017tactics, pan2019characterizing}. 

To tackle these challenges, distributionally robust reinforcement learning (DR-RL) \citep{zhou21, yang2021,liu22DRQ,ShiChi2022,Wang2023MLMCDRQL} has emerged as a promising approach. DR-RL seeks to learn policies that are robust to distributional shifts in the environment by explicitly considering a family of possible distributions that the agent may encounter during deployment. This approach allows the agent to learn a policy that performs well across a range of environments, rather than just the one it was trained on. 

These benefits of distributionally robust policies motivate the exploration of a critical question: \textit{Can we construct efficient reinforcement learning algorithms that achieve the desired robustness properties while also providing provable guarantees on their sample complexity?}

A growing body of literature aims to understand the sample complexities of distributionally robust reinforcement learning. Specifically, we are interested in a robust tabular Markov Decision Process (MDP) with state space \(\bd{S}\) and action space \(\bd{A}\), in the discounted infinite-horizon setting with discount factor \(\gamma\). To account for uncertainty, we use an ambiguity set based on Kullback-Leibler (KL) divergence with ambiguity size \(\delta\), which is arguably the most natural and challenging divergence in distributionally robust literature. Previous research has mainly focused on the \textit{model-based} approach, where a specific model of the environment is estimated, and value iteration (VI) is run on the estimated model. Table \ref{tab:sample_complexity_model_based} shows the worst-case sample complexity of model-based distributionally RL, with \citet{ShiChi2022} proposing a method with state-of-the-art sample complexity in terms of \(|\bd{S}|,|\bd{A}|,1-\gamma,\epsilon\).

\begin{table}[htb]
\centering
\begin{tabular}{lll}
\toprule Algorithm & Sample Complexity & Origin \\ 
\midrule 
DRVI & $\tilde
O(|\bd{S}|^2|\bd{A}|e^{O(1-\gamma)^{-1}}(1-\gamma)^{-4}\epsilon^{-2}\delta^{-2})$ & \citet{zhou21} \\ 
REVI/DRVI & $\tilde
O(|\bd{S}|^2|\bd{A}|e^{O(1-\gamma)^{-1}}(1-\gamma)^{-4}\epsilon^{-2}\delta^{-2})$ & \citet{Panaganti2021} \\ 
DRVI & $\tilde
O(|\bd{S}|^2|\bd{A}|(1-\gamma)^{-4}\epsilon^{-2}\fr{p}_{\wedge}^{-2}\delta^{-2})$ & \citet{yang2021} \\ 
DRVI-LCB & $\tilde O(|\bd{S}||\bd{A}|(1-\gamma)^{-4}\epsilon^{-2}\fr{p}_{\wedge}^{-1}\delta^{-2})$   & \citet{ShiChi2022} \\
\bottomrule
\end{tabular}\caption{Summary of sample complexity upper bounds for finding an $\epsilon$-optimal robust  policy in  \textit{model-based} distributionally robust RL ($\fr{p}_{\wedge}$ is the minimal support probability of the nominal MDP; see, Def. \ref{def:min_supp_prob}).}
\label{tab:sample_complexity_model_based}
\end{table}

\subsection{Our Motivation}

The emerging line of work mentioned above reflects the growing interest and fruitful results in the pursuit of sample-efficient distributionally robust reinforcement learning. At the same time, a closer scrutiny of the results suggests that two fundamental aspects of the problem are inadequately addressed. 

For one thing, the complexity bounds of existing results exhibit $\tilde O(\delta^{-2})$ dependence as $\delta\da 0$.  This increase in the complexity bounds appears to reflect an increased need for learning the training environment as the training and adversarial environments become more alike. At the surface level, this makes sense: in the extreme case where $\delta$ is approaching $\infty$, then (assuming known support of the distributions) no sample is needed to find an optimal distributionally robust policy. Nevertheless, such bounds have failed to align with the continuity property of the robust MDP: the robust value function should converge to the non-robust optimal cumulative reward as $\delta\da 0$. Therefore, for all sufficiently small $\delta$ that may depend on the training environment and $\epsilon$, the robust value function can be approximated by the output of a classical RL algorithm. Specifically, we expect an algorithm and analysis with a $\tilde O(1)$ dependence as $\delta\da 0$. This is presently absent in the literature.

Additionally, with the exception of~\citet{Wang2023MLMCDRQL} (discussed in more detail in the next subsection), all the existing distributionally robust policy learning algorithms that have finite-sample guarantees (such as the ones 
mentioned above~\citep{zhou21, Panaganti2021, yang2021, ShiChi2022}) are model-based, which estimates the underlying MDP first before provisioning some policy from it. Although model-based methods are often more sample-efficient and easier to analyze, their drawbacks are also well-understood~\citep{sutton2018reinforcement,franccois2018introduction}: they are computationally intensive; they require more memory to store MDP models and often do not generalize well to non-tabular RL settings. These issues limit the practical applicability of model-based algorithms,  which stand in contrast to model-free algorithms that learn to select actions without first learning an MDP model. Such methods are often more computationally efficient, have less storage overhead, and better generalize to RL with function approximation. In particular, $Q$-learning~\citep{watkins1992q}, as the prototypical model-free learning algorithm, has widely been both studied theoretically and deployed in practical applications. However, $Q$-learning is not robust (as demonstrated in our simulations), and the policy learned by $Q$-learning in one environment can perform poorly in another under a worst-case shift (with bounded magnitude). 

As such, the above discussion naturally motivates the following research question:

\textit{Can we design a variant of $Q$-Learning that is distributionally robust, where the sample complexity has the right scaling with $\delta$?}


\subsection{Our Contributions}
We answer the above question affirmatively and contribute to the existing literature on the worst-case sample complexity theory of \textit{model-free} distributionally robust RL. We propose two distributionally robust variants of the Q-learning algorithm \citep{watkins1992q}, namely DR Q-learning (Algorithm \ref{alg:q-learning}) and variance-reduced DR Q-learning (Algorithm \ref{alg:vr_q-learning}), which effectively solve the DR-RL problem under the KL ambiguity set. 

The proposed algorithms operate efficiently under the assumption of limited power of the adversary (as per Assumption \ref{assump:delta_small}), which is realistic in many real-world applications. We prove that both algorithms have near-optimal worst-case sample complexity guarantees in this regime. Additionally, the variance-reduced version exhibits superior complexity dependence on the effective horizon \((1-\gamma)^{-1}\), as shown in Table \ref{tab:sample_complexity}. To the best of our knowledge, both algorithms and their worst-case sample complexity upper bounds represent state-of-the-art results in model-free distributionally robust RL. Moreover, our sample complexity upper bound for variance-reduced DR Q-learning matches the best-known upper bound for this DR-RL problem in \citet{ShiChi2022} in terms of \(\epsilon^{-2}\) and \((1-\gamma)^{-4}\) dependence.

\begin{table}[htb]
\centering
\begin{tabular}{lll}
\toprule Algorithm & Sample Complexity & Origin \\ 
\midrule MLMC DR Q-learning & $\tilde
O(|\bd{S}||\bd{A}|(1-\gamma)^{-5}\epsilon^{-2}\fr{p}_{\wedge}^{-6}\delta^{-4})$ & \citet{Wang2023MLMCDRQL} \\ 
\midrule DR Q-learning & $\tilde
O(|\bd{S}||\bd{A}|(1-\gamma)^{-5}\epsilon^{-2}\fr{p}_{\wedge}^{-3})$ & Theorem \ref{thm:ql_sample_complexity}\\ 
Variance-reduced DR Q-learning & $\tilde
O(|\bd{S}||\bd{A}|(1-\gamma)^{-4}\epsilon^{-2}\fr{p}_{\wedge}^{-3})$ & Theorem \ref{thm:vrql_sample_complexity} \\ 
\bottomrule
\end{tabular}\caption{Summary of sample complexity upper bounds for finding an $\epsilon$-optimal robust  policy in \textit{model-free} distributionally robust RL ($\fr{p}_{\wedge}$ is the minimal support probability of the nominal MDP; see, Def. \ref{def:min_supp_prob}).}
\label{tab:sample_complexity}
\end{table}

The DR Q-learning Algorithm \ref{alg:q-learning} is a direct extension of mini-batch Q-learning. Compared to the MLMC DR Q-learning method proposed by \citet{Wang2023MLMCDRQL}, Algorithm \ref{alg:q-learning} is easier to implement in real-world applications. Additionally, this approach allows for the design of a more sophisticated variant, the variance-reduced DR Q-learning, which provides a provable enhancement of the worst-case sample complexity guarantee of DR Q-learning. To achieve this improvement, we leverage Wainwright's variance reduction technique and algorithm structure \citep{wainwright2019}, adapting it to the DR-RL context and redesigning the variance reduction scheme accordingly.

Both the DR Q-learning and its variance-reduced version use a stochastic approximation (SA) step to iteratively update the estimator of the optimal DR \( q \)-function towards the fixed point of the population DR Bellman operator. However, both algorithms involve a bias that must be controlled at the algorithmic and iterative update levels. Our contribution to the literature lies in the near-optimal analysis of the biased SA resulting from DR Q-learning and its variance-reduced version. This analysis also generalizes to settings where the biased stochastic version of the contraction mapping is a monotonic contraction.

We highlight that these are the first algorithmic complexity results showing that the worst-case complexity dependence on the uncertainty set size \(\delta\) is \(O(1)\) as \(\delta \to 0\) for the DR-RL problem with a KL ambiguity set. This resolves the issue of worst-case complexity bounds blowing up as \(\delta\) approaches 0, a problem present in all previous works, including both model-based and model-free approaches \citep{yang2021,Panaganti2021,ShiChi2022,Wang2023MLMCDRQL}. 

The significance of this characteristic lies in its theoretical illustration that as the adversary's power \(\delta\) approaches 0, not only does the solution to the DR-RL problem converge to that of the non-robust version, but so does the sample complexity required to solve it. This sheds light on the connection between robust and non-robust RL problems, indicating that in more general settings and real-world applications, DR-RL problems with function approximation may be efficiently addressed by utilizing variants of the corresponding approach for non-robust RL problems.

\subsection{Literature Review}

This section is dedicated to reviewing the literature that is relevant to
our work. The literature on RL and MDP is extensive. One major line of
research focuses on developing algorithms that can efficiently learn
policies to maximize cumulative discounted rewards. When discussing RL and
MDP problems, we will concentrate on this infinite horizon discounted reward
formulation.

\textbf{Minimax Sample Complexity of Tabular RL}: Recent years have seen significant developments in the worst-case sample complexity theory of tabular RL. Two principles, namely model-based and model-free, have motivated distinct algorithmic designs. In the model-based approach, the controller aims to gather a dataset so as to construct an empirical model of the underlying MDP and solve it using variations of the dynamic
programming principle. Research \citep{azar2013,sidford2018near_opt,
agarwal2020,li2022settling} have proposed model-based algorithms and proven
optimal upper bounds for achieving $\epsilon$, with a matching lower bound $\tilde\Omega(|\bd{S}||\bd{A}|(1-\gamma)^{-3}\epsilon^{-2})$ proven in \citet{azar2013}.
In contrast, the model-free approach involves maintaining only
lower-dimensional statistics of the transition data, which are iteratively
updated. As one of the most well-known model-free algorithms, the sample
complexity of Q-learning has been extensively studied \citep{even2003learning,wainwright2019l_infty,chen2020,li2021QL_minmax}.
However, \citet{li2021QL_minmax} have shown that the Q-learning has a minimax
sample complexity of $\tilde \Theta(|\bd{S}||\bd{A}|(1-\gamma)^{-4}\epsilon^{-2})$,
which doesn't match the lower bound $\tilde\Omega(|\bd{S}||\bd{A}|(1-\gamma)^{-3}\epsilon^{-2})$. Nevertheless, variance-reduced variants of the Q-learning,
such as the one proposed in \citet{wainwright2019}, achieve the
aforementioned sample complexity lower bound. Other algorithmic techniques such as Polyak-Ruppert averaging \citep{li2023statistical} have been shown to result in optimal sample complexity. 

\textbf{Finite Analysis of SA:} The classical theory of asymptotic
convergence for SA has been extensively studied, as seen in \citet{kushner2013SA}. Recent progress in the minimax and instant dependent sample
complexity theory of Q-learning and its variants has been aided by advances
in the finite-time analysis of SA. Traditional RL research focuses on
settings where the random operator is unbiased. \citet{wainwright2019l_infty}
demonstrated a sample path bound for the SA recursion, which enables the use
of variance reduction techniques to achieve optimal learning rates. In
contrast, \citet{chen2020,chen2022} provided finite sample guarantees for SA
only under a second moment bound on the martingale difference noise
sequence. Additionally, research has been conducted on non-asymptotic
analysis of SA procedures in the presence of bias, as documented in \citep{karimi2019biased_SA,Wang2022biased_SA}.

\textbf{Robust MDP and RL:} Our work draws upon the theoretical framework of
classical max-min control and robust MDPs, as established in previous works \citep{gonzalez-trejo2002,Iyengar2005,nilim2005robust,
wiesemann2013,huan2010,shapiro2022, wang2024foundation}. These works have established the concept of distributional robustness in dynamic decision making. In particular, \citet{gonzalez-trejo2002,Iyengar2005,nilim2005robust} established the distributionally robust dynamic programming principles for SA-rectangular adversaries under symmetric information structures, while \citet{wiesemann2013,wang2024foundation} studies asymmetric settings, leading to the same the DR Bellman equation.

Recent research has shown great interests in learning DR policies from data \citep{si2020,zhou21, yang2021,liu22DRQ,ShiChi2022,Wang2023MLMCDRQL,yang2023avoiding}. For
instance, \citep{si2020} studied the contextual bandit setting, while \citep{zhou21, Panaganti2021,yang2021, ShiChi2022} focused on the model-based
tabular RL setting. On the other hand, \citep{liu22DRQ,Wang2023MLMCDRQL,yang2023avoiding}
tackled the DR-RL problem using a model-free approach\footnote{\citet{liu22DRQ}'s algorithm is infeasible: it requires an infinite number of samples in expectation for \textit{each iteration}, and only asymptotic convergence is established with an infinite number of iterations.}. Before our work, the
best worst-case sample complexity upper bound for DR-RL under the KL
ambiguity set was established for the model-based DRVI-LCB algorithm, as
proposed and analyzed by \citet{ShiChi2022}. Their analysis showed that the
worst-case sample complexity has an upper bound of $\tilde
O(|\bd{S}||\bd{A}|(1-\gamma)^{-4}\epsilon^{-2}\delta^{-2}\fr{p}_{\wedge}^{-1})$.

\section{Distributionally Robust Reinforcement Learning}

\subsection{Classical Tabular Reinforcement Learning}

Let \(\cM_0 = \left(\bd{S}, \bd{A}, \bd{R},  P_0,  N_0, \gamma\right)\) be a Markov decision process (MDP), where \(\bd{S}\), \(\bd{A}\), and \(\bd{R}\subsetneq \R_+\) are finite state, action, and reward spaces\footnote{We assume a finite reward space for simplicity. However, our results can be extended to continuous reward spaces by imposing a minimum density assumption, as described in \citet{si2020}., respectively. Let \(\cP(\bd{U})\), where \(\bd{U} = \bd{S}, \bd{A}, \bd{R}\), denote the set of probability measures on the power set \(2^{\bd{U}}\). Then \(P_0 = \{p_{s,a} \in \cP(\bd{S}), s \in \bd{S}, a \in \bd{A}\}\) and \(N_0 = \{\nu_{s,a} \in \cP(\bd{R}), s \in \bd{S}, a \in \bd{A}\}\) are the sets of transition and reward distributions, respectively. \(\gamma \in (0,1)\) is the discount factor. Define \(r_{\max} = \max \{r \in \bd{R}\}\) as the maximum reward. 

\par At each time $t$, given the state process is at $S_t$ and the decision maker takes action $A_t$, the subsequent state is determined by the conditional distribution \(S_{t+1}\sim p_{S_t,A_t}\). Then, a randomized reward $R_t\sim \nu_{S_t,A_t}$ will be collected, independent of the history. }

\par Let $\Pi$ be the history-dependent policy class (see \citep{wang2024foundation} for a rigorous construction). For $\pi\in \Pi$, the value function $v^{\pi}(s)$ is defined
as: 
\begin{equation*}
v^{\pi}(s) := E\sqbkcond{\sum_{t=0}^\infty \gamma^{t}R_t }{S_0 = s}.
\end{equation*}
\par The optimal value function is
\begin{equation*}
v^*(s) \coloneqq \max_{\pi\in\Pi} v^\pi(s),
\end{equation*}
$\forall s\in\bd{S}$. It is well known that the optimal value function is the
unique solution of the following Bellman equation: 
\begin{equation*}
v^*(s) = \max_{a\in\bd{A}} \crbk{E_{\nu_{s,a}}[R] + \gamma E_{
p_{s,a}}[v^*(S)]}.
\end{equation*}
where the expectations are taken over $R\sim \nu_{s,a}$ and $S\sim p_{s,a}$, respectively. 
\par An important implication of the Bellman equation is that it suffices to optimize within the stationary Markovian deterministic policy class.

\par We define the optimal $q$-function as
\begin{equation*}
q^*(s,a) := E_{ \nu_{s,a}}[R] + \gamma E_{
p_{s,a}}[v^*(S)]. 
\end{equation*}
It is well-know that $q^*$ satisfies its Bellman equation 
\begin{equation*}
q^*(s,a)= E_{ \nu_{s,a}}[R] + \gamma E_{ p_{s,a}}\sqbk{\max_{b\in\bd{A}} q^*(S,b)}.
\end{equation*}
An optimal policy can be constructed as $\pi^*(s) = \arg\max_{a\in\bd{A}}q^*(s,a)$. Therefore, policy
learning in RL environments can be achieved if we can learn a good estimate
of $q^*$.

\subsection{Kullback-Leibler Divergence Constrained DR-RL}

We consider a DR-RL setting where the adversary is constrained to perturb both transition probabilities and rewards within a KL divergence ball of radius $\delta$. Specifically, for probability measures $Q$ is absolutely continuous w.r.t. $P$ on some measurable space $(\Omega,\cF)$, denoted by $Q\ll P$, define 
\begin{equation}\label{eqn:def_KL_div}
    D_{\text{KL}}(Q\|P) := \int_\Omega
\log\crbk{\frac{dQ}{dP}(\omega)}P(d\omega),
\end{equation}where $\frac{dP}{dQ}$ is the Radon-Nikodym derivative. 
\par For each $(s,a)\in \bd{S}\times\bd{A}$ and $\delta > 0$, we define KL ambiguity set that are centered at $p_{s,a}\in P_0$ and $\nu_{s,a}\in N_0$ of radius $\delta$ by 
\begin{equation}\label{eqn:def_KL_delta_balls}\begin{aligned}
\cP_{s,a}(\delta)&\coloneqq \left\lbrace p : D_{\text{KL}}\left(p\|p_{s,a}\right)\leq \delta\right\rbrace, \\
\cN_{s,a}(\delta)&\coloneqq \left\lbrace \nu: D_{\text{KL}}(\nu\|\nu_{s,a})\leq\delta\right\rbrace. \end{aligned}
\end{equation}
These ambiguity sets represent the possible distributional shifts from the reference model $P_0, N_0$. In particular, the parameter $\delta > 0$ controls the size of the ambiguity sets, quantifying the power of the adversary. 
\par With these definitions in mind, we define the DR optimal value function as the solution to a fixed point equation--a.k.a. the DR Bellman equation--which serves as the learning objective of this paper. 
\begin{definition}\label{def:DRvalue}
The DR Bellman operator $\cB_\delta$ for the value function is defined as the mapping
\begin{equation} \label{def:robust_Bellman_opt} \cB_\delta(v)(s):=
\max_{a\in\bd{A}} \inf_{\mbox{$\begin{subarray}{c} p\in\cP_{s,a}(\delta),\\
		\nu\in \cN_{s,a}(\delta)\end{subarray}$}}
\crbk{   E_{ \nu} [R] + \gamma E_{p}\left[v(S)\right]}.
\end{equation}
Define the DR optimal value function $v^*_\delta$ as the solution of the DR Bellman equation: 
\begin{equation}\label{def:robust_Bellman}
    v^*_\delta =  \cB_\delta(v^*_\delta)
\end{equation}
\end{definition}
Moving forward, we will suppress the explicit dependence on $\delta$.

The DR Bellman equation has a unique solution as the fixed point of $\cB$, which is a consequence of $\cB$ being a contraction operator. Furthermore, the solution is equal to the max-min control optimal value of a \textit{SA-rectangular} distributionally robust MDP (DRMDP) \citep{Iyengar2005,nilim2005robust,wiesemann2013}. Specifically, this max-min optimal value is given by 
\begin{equation}
\label{eqn:maxmin_control_val}u^*(s) := \sup_{\pi\in\Pi}\inf_{\kappa\in\mathrm{K}} E^{\pi,\kappa }
\sqbkcond{\sum_{t=0}^\infty \gamma^t R_t }{s_0 = s}
\end{equation}
where $\Pi$ is the history-dependent policy class, and the adversary chooses a policy $\kappa$ from an adversarial ambiguity set $\mathrm{K}$ that is induced by the KL ambiguity sets in \eqref{eqn:def_KL_delta_balls}. 

\par Intuitively, this value represents the optimal reward in the following adversarial environment: When the controller selects a policy \(\pi\), an adversary observes this policy and then chooses a counter-policy that determines the sequence of reward and transition distributions. The adversary's choice is constrained such that the reward and transition distributions induced by the counter-policy lie within the ambiguity set \eqref{eqn:def_KL_delta_balls} of radius $\delta$. The decisions made by both the controller and the adversary uniquely specify the law of the state-action-reward process, thereby determining the value of the policy pair \((\pi,\kappa)\). 

\par The equivalence of the max-min control optimal value \eqref{eqn:maxmin_control_val} and the solution to the DR Bellman equation \eqref{def:robust_Bellman} shows the optimality of stationary deterministic Markov control policies and stationary Markovian adversarial distribution choices. This equivalence, known as the \textit{dynamic programming principle} (DPP), is explored in detail in \citet{wang2024foundation}, where the adversary and controller can have asymmetric information structures. For those interested, we refer you to this paper.

\par We note that \citet{wang2024foundation} considers a setting where the reward is not randomized, i.e., \(\cN_{s,a} = \{\delta_{r(s,a)}\}\) for some reward function \(r: \bd{S} \times \bd{A} \to [0,1]\). However, it is straightforward to generalize the DPP to include randomized rewards in the SA-rectangular setting.

\subsection[Dual and q-Function Formulations]{Dual and $q$-Function Formulations}
The right-hand side of \eqref{def:robust_Bellman_opt} can be challenging to work with because the measure underlying the expectations is not directly accessible. To address this, we use strong duality to reveal the dependence of the value on the reference transition and reward distributions, \(P_0\) and \(N_0\). Specifically, we consider the dual representation: 
\begin{lemma}[\citet{Hu2012KLDRO}, Theorem 1] \label{lemma:strong_dual}
	Let $X$ be a random variable and $\mu_0$ be a probability measure on $(\Omega,\cF)$ s.t. $X$ has a finite moment generating function in a neighborhood of zero. Then for any $\delta >0$, 
	\begin{equation*}
	\inf_{\mu: D_{\emph{KL}}(\mu\|\mu_0)\leq \delta} E_\mu X=\sup_{\alpha\geq 0}\left\lbrace -\alpha\log E_{\mu_0}\left[e^{-X/\alpha}\right] -\alpha\delta\right\rbrace.
	\end{equation*}
\end{lemma}
Since the reward and values are bounded, directly apply Lemma \ref{lemma:strong_dual} to the r.h.s. of \eqref{def:robust_Bellman}, the DR value function $v^{*}$ in fact
satisfies the following \textit{dual form} of the DR Bellman's equation. 
$$ v^{ *}(s) = \max_{a\in\bd{A}} \left\lbrace \sup_{\alpha\geq
0}\left\lbrace -\alpha\log E_{ \nu_{s,a}}\left[e^{-R/\alpha}\right] -
\alpha\delta\right\rbrace \right. + \left. \gamma\sup_{\beta\geq
0}\left\lbrace -\beta\log E_{ p_{s,a}}\left[e^{-v^*(S)/\beta}\right]
- \beta\delta\right\rbrace \right\rbrace.$$

Similar to the traditional RL policy learning approach, we utilize the
optimal DR state-action value function, also known as the $q$-function, to
address the DR-RL problem. The $q$-function assigns real numbers to pairs of
states and actions, and can be represented as a matrix $q\in \R^{\bd{S}\times\bd{A}}$. From now on, we will assume this representation. To simplify notation, let
us define 
\begin{equation}
v(q)(s) := \max_{b\in\bd{A}}q(s,b),  \label{def:v_operator}
\end{equation}
which is the value function induced by the $q$-function $q(\cdot,\cdot)$.

We proceed to rigorously define the optimal $q$-function and its Bellman
equation. 
\begin{definition}
The optimal DR $q$-function is defined as
\begin{equation}\label{eqn:def_q*_func}
q^*(s,a) := \inf_{\mbox{$\begin{subarray}{c} p\in\cP_{s,a}(\delta),\\
		\nu\in \cN_{s,a}(\delta)\end{subarray}$}}
    \crbk{   E_{\nu} [R] + \gamma E_{p}\left[v^*(S)\right] }
\end{equation}
where $v^*$ is the DR optimal value function in Definition \ref{def:DRvalue}.
\end{definition}
Similar to the Bellman operator, we can define the DR Bellman operator for
the q-function as follows: 
\begin{definition}
    Given $\delta>0$ and $q\in\R^{ \bd{S}\times\bd{A}}$, the \textit{primal form} of the DR Bellman operator $\cT:\R^{ \bd{S}\times \bd{A}}\to \R^{ \bd{S}\times \bd{A}}$ is defined as
    \begin{equation}\label{eqn:def_bellman_op_primal}
    \cT(q)(s,a):= \inf_{\mbox{$\begin{subarray}{c} p\in\cP_{s,a}(\delta),\\
		\nu\in \cN_{s,a}(\delta)\end{subarray}$}}
    \crbk{  E_{\nu} [R] + \gamma E_{p}\left[v(q)(S)\right] }
    \end{equation}
    The \textit{dual form} of the DR Bellman operator is
    \begin{equation}\label{eqn:def_bellman_op_dual}
    \begin{aligned}\cT(q)(s,a) &= \sup_{\alpha\geq 0}\left\lbrace -\alpha\log E_{\nu_{s,a}}\left[e^{-R/\alpha}\right] - \alpha\delta\right\rbrace\\
    &\quad + \gamma\sup_{\beta\geq 0}\left\lbrace -\beta\log E_{p_{s,a}}\left[e^{-v(q)(S)/\beta}\right] - \beta\delta\right\rbrace.
    \end{aligned}
    \end{equation}
\end{definition}

The equivalence of the primal and dual form follows from Lemma \ref{lemma:strong_dual}. We remark that the dual form is usually easier to work
with, as the outer supremum is a 1-d optimization problem and the dependence
on the reference measures $\nu_{s,a}$ and $p_{s,a}$ are explicit.

Note that by definition \eqref{eqn:def_q*_func} and the Bellman equation \eqref{def:robust_Bellman}, we have $v(q^*) = v^*$. So, our definition
implies that $q^*$ is a fixed point of $\cT$ and the following Bellman
equation for the $q^*$-function holds: 
\begin{equation}  \label{eqn:dr_bellman_eqn_q}
q^* = \cT(q^*).
\end{equation}
The uniqueness of the fixed point $q^*$ of $\cT$ follows from the
contraction property of the operator $\cT$; c.f. Lemma \ref{lemma:monotone_contraction}.

The optimal DR policy can be extracted from the optimal $q$-function by $\pi^*(s) = \arg\max_{a\in\bd{A}}q^*(s,a).$ Hence the goal the DR-RL paradigm is
to learn the DR $q$-function and extract the corresponding robust policy.

\subsection{Synchronous Q-Learning and Stochastic Approximation}

The Q-learning estimates the optimal $q$-function by iteratively update the
estimator $\set{q_k:k\geq 0}$ using samples generated by the reference
measures. The classical synchronous Q-learning proceeds as follows. At
iteration $k\in \Z_{\geq 0}$ and each $(s,a)\in \bd{S}\times\bd{A}$, we draw samples $R_{k+1}\sim \nu_{s,a}$ and $S_{k+1}\sim p_{s,a}$. Then perform the Q-learning
update 
\begin{equation}  \label{eqn:non-drql}
q_{k+1}(s,a)= (1-\lambda_k)q_k(s,a) + \lambda_k(R_{k+1} +\gamma
v(q_k)(S_{k+1}))
\end{equation}
for some chosen step-size sequence $\set{\lambda_k}$.

The synchronous Q-learning can be analyzed as a stochastic approximation
(SA) algorithm. SA for the fixed point of a contraction operator $\cL$
refers to the class of algorithms using the update 
\begin{equation}  \label{eqn:sa_update}
X_{k+1} = (1-\lambda_k)X_{k} + \lambda_k\cL(X_k) + W_{k+1}.
\end{equation}
$\set{W_k}$ is a sequence satisfying $E[W_{k}|W_{k-1},\ds,W_1] = 0$ and some
higher order moment conditions, thence is known as the martingale difference noise. The asymptotics of the above recursion are
well-understood in the literature, as discussed in \citet{kushner2013SA}. The
recent developments of finite-time/sample behavior of SA is discussed in the
literature review. The Q-learning recursion in \eqref{eqn:non-drql} can be represented as an SA update if we notice that given any $q$-function, $R + \gamma v(q)(S) $ is an \textit{unbiased} estimator of the population Bellman operator applied to $q$. However, the DR Q-learning and the variance-reduced version
cannot be formulated in the same way as \eqref{eqn:sa_update} with
martingale difference noise, as there is bias present in the former
algorithms. Consequently, to achieve the nearly optimal sample complexity
bounds, we must conduct a tight analysis of these algorithms as biased SA,
as we will explain in the subsequent sections.

\section{The DR Q-Learning and Variance Reduction}

This section introduces two model-free algorithms, the DR Q-learning
(Section \ref{subsec:dr_ql}) and its variance-reduced version (Section \ref{subsec:dr_vrql}), for learning the optimal $q$-function of a robust MDP. We
also present the upper bounds on their worst-case sample
complexity. In addition, we outline the fundamental ideas behind the proof
of the sample complexity results in Section \ref{sec:analysis_of_algo_overview}. 

Prior to presenting the algorithms, we introduce several notations. Let $\nu_{s,a,n}$ and $p_{s,a,n}$ denote the empirical measure of $\mu_{s,a}$ and 
$p_{s,a}$ formed by $n$ i.i.d. samples respectively; i.e. for $f:\bd{U}\ra \R$,
where $\bd{U}$ could be the $\bd{S}$ or $\bd{R}$, 
\begin{equation}\label{eqn:empirical_meas_exp}
E_{\mu_{s,a,n}}f(U) := \frac{1}{n}\sum_{j=1}^nf(U_i)
\end{equation}
for $\mu = \nu,p$ and $U_i = R_i,S_i$ are i.i.d. across $i$.

Assuming access to a simulator, we are able to draw samples and
construct an empirical version of the DR Bellman operator. 
\begin{definition}\label{def:empirical_Bellman_operator}
Define the \textit{empirical DR Bellman operator} on $n$ i.i.d. samples by
\begin{equation}\label{eqn:empirical_Bellman_operator}
\begin{aligned}
\mathbf{T}(q)(s,a)&\coloneqq \sup_{\alpha\geq 0}\left\lbrace -\alpha\log E_{\nu_{s,a,n}}\left[e^{-R/\alpha}\right] - \alpha\delta\right\rbrace\\
&\quad +\gamma\sup_{\beta\geq 0}\left\lbrace -\beta\log E_{p_{s,a,n}}\left[e^{-v(q)(S)/\beta}\right] - \beta\delta\right\rbrace.
\end{aligned}
\end{equation}
\end{definition}

Note that $\mathbf{T}$ is a random operator whose randomness is coming
from on the samples that we used to construct $\set{\nu_{s,a,n},p_{s,a,n}:(s,a)\in \bd{S}\times\bd{A}}$.

Definition \ref{def:empirical_Bellman_operator} presents the empirical DR
Bellman operator in its dual form. Lemma \ref{lemma:strong_dual} establishes
that this definition is equivalent to the DR Bellman operator $\cT$ in \eqref{eqn:def_bellman_op_primal} where the sets $\cP_{s,a}(\delta)$ and $\cN_{s,a}(\delta)$ are replaced with their empirical counterparts: $\{p: D_{\text{KL}}(p\|p_{s,a,n})\leq \delta\}$ and $\{\nu: D_{\text{KL}}(\nu\|\nu_{s,a,n})\leq \delta\}$.

The dual formulation of the empirical DR Bellman operator implies that it is
generally a biased estimator of the population DR Bellman operator $\cT$ in
the sense that $E\left[\mathbf{T}(q)\right]\neq \cT(q)$ for a generic $q\in \R^{\bd{S}\times\bd{A}}$. This bias poses a significant challenge in the design of model-free
algorithms and the analysis of sample complexities. Previous works \citet{liu22DRQ} and \citet{Wang2023MLMCDRQL}
eliminates this bias by using a randomized multilevel Monte Carlo estimator.
However, the randomization procedure requires a random (and heavy-tailed) sample size.
Therefore, the complexity bound is stated in terms of the expected number of
samples. Also, this complex algorithmic design limits its generalizability.
In contrast, this paper takes a different approach by directly analyzing the
DR Q-learning and its variance-reduced version as biased SA. To achieve
near-optimal sample complexity guarantees, the bias of the empirical DR
Bellman operator and the propagation of the systematic error it causes are
tightly controlled, and samples are optimally allocated so that the
stochasticity is in balance with the cumulative bias.  A detailed discussion
of this approach is provided in Section \ref{sec:analysis_of_algo_overview}.

To state the key assumption which constraint the operating regime of our
algorithm, we introduce the following complexity metric parameter: 
\begin{definition}\label{def:min_supp_prob}
    Define the \textit{minimum support probability} as
    \begin{equation}\label{eqn:min_supp_prob}
    \fr{p}_{\wedge} := \min_{s,a\in \bd{S}\times\bd{A}}\min\set{\min_{r\in\bd{R}:\nu_{s,a}(r) > 0}\nu_{s,a}(r),\min_{s'\in \bd{S}:p_{s,a}(s')>0} p_{s,a}(s')}.
    \end{equation}
\end{definition}

The intuition behind the dependence of the MDP complexity on the minimal
support probability is that in order to estimate the DR Bellman operator
with high accuracy in the worst case, it is necessary to know the entire
support of the transition and reward distributions. As a result, at least $1/\fr{p}_{\wedge}$ samples are required, as discussed in \citet{Wang2023MLMCDRQL}.

We are now prepared to present the main assumption that defines the
operating regime for which our algorithms are optimized. 

\begin{assumption}[Limited Adversarial Power] \label{assump:delta_small}
Suppose the adversary's power $\delta$ satisfies $\delta < \frac{1}{24}\fr{p}_{\wedge}$.
\end{assumption}

It should be noted that the constant $1/24$ is only for mathematical convenience and can potentially be improved.

Under this assumption, the adversary cannot collapse the support of the transition or reward distributions to a singleton, preventing them from completely restricting possible transition events under \(P_0\). This assumption regime is of practical significance because overly conservative policies can be produced if \(\delta\) is large. Furthermore, the support of the reward and transition measures often encode physical constraints intrinsic to the real environment, which the adversary should not be allowed to violate. 

We also make the following simplifying assumption. 
\begin{assumption}[Reward Bound] \label{assump:max_rwd}
The reward $\bd{R}\subset [0,1]$. 
\end{assumption}
This assumption is straightforward to remove given that the results of the empirical Bellman operator hold for $\bd{R}\subset \R_{\geq 0}$. We assume it so as to clarify our presentation. 

\subsection{The Distributionally Robust Q-learning}

\label{subsec:dr_ql} First, we proposed the DR Q-learning Algorithm \ref{alg:q-learning}, a robust version of the classical Q-learning that is based
on iteratively update the $q$-function by applying the $n$-sample empirical
Bellman operator. 
\begin{algorithm}[ht]
   \caption{Distributionally Robust  Q-Learning}
   \label{alg:q-learning}
\begin{algorithmic}
	\STATE {\bfseries Input:} the total times of iteration $k_0$ and a batch size $n_0$.
   \STATE {\bfseries Initialization:} $q_1 \equiv 0$; $k = 1$. 
   \FOR{$1\leq k\leq k_0$}
   \STATE Sample $\mathbf{T}_{k+1}$ the $n_0$-sample empirical DR Bellman operator as in Definition \ref{def:empirical_Bellman_operator}. 
   \STATE Compute the Q-learning update
   \begin{gather}\label{eqn:q-learning_update}
   q_{k+1} = (1-\lambda_{k})q_{k} + \lambda_k\mathbf{T}_{k+1}(q_{k})
   \end{gather}
   with stepsize $\lambda_k = 1/(1+(1-\gamma)k)$. 
   \ENDFOR
   \RETURN $ q_{k_0+1}$.
\end{algorithmic}
\end{algorithm}

Algorithm \ref{alg:q-learning} can be viewed as a \textit{biased} SA: We can rewrite the update \eqref{eqn:q-learning_update} as
\begin{equation*}
q_{k+1} = (1-\lambda_k)q_{k} + \lambda_k\cT(q_k) + \lambda_k(\mathbf{T}_{k+1}(q_k) -\cT(q_k) ).
\end{equation*}
This is in the form
of \eqref{eqn:sa_update}. However, notice that $E[\mathbf{T}_{k+1}(q_k) -\cT(q_k)|q_k]\neq 0$. Moreover, we note that the update \eqref{eqn:q-learning_update} involves computing \(\bd{T}_{k+1}(q_k)(s,a)\) for all \((s,a) \in \bd S \times \bd A\). Unlike a model-based algorithm, which requires storing the entire empirical kernel and reward distributions \(\{p_{s,a,n}, \nu_{s,a,n} : (s,a) \in \bd S \times \bd A\}\), the update rule \eqref{eqn:q-learning_update} can be implemented separately for each state-action pair. This allows \(p_{s,a,n}\) and \(\nu_{s,a,n}\) to be discarded immediately after the update, significantly reducing the memory requirements for running Algorithm \ref{alg:q-learning} when the state space is large. 

It turns out that, by leveraging the fact that the empirical Bellman
operators are monotone contractions w.p.1 (as proven in Lemma \ref{lemma:monotone_contraction}), we can perform a stronger pathwise analysis
of Algorithm \ref{alg:q-learning} instead of treating it as a variant of the
SA update in \eqref{eqn:sa_update}. As a result, we will prove in Section \ref{subsec:proof_drql} that the DR Q-learning algorithm satisfies the following error bound in Proposition \ref{prop:q-learning_err_high_prob_bd}. 
\par To simplify notation, we define the dimensionality parameter $d:= |\bd{S}||\bd{A}|(|\bd{S}|\vee|\bd{R}| )$. It will only show up inside the $\log(\cdot)$ term in our complexity bounds because of the use of union bound techniques.

\begin{proposition}\label{prop:q-learning_err_high_prob_bd}
Suppose that Assumptions \ref{assump:delta_small} and \ref{assump:max_rwd} are satisfied. The output $q_{k_0+1}$ of the distributionally robust Q-learning satisfies
\[
\|q_{k_0+1}-q^*\|_\infty \leq c\crbk{\frac{1}{(1-\gamma)^3k_0} + \frac{1}{\fr{p}_{\wedge}^3(1-\gamma)^2 n_0} + \frac{1}{\fr{p}_{\wedge}(1-\gamma)^{5/2}\sqrt{n_0k_0}}}(\log\crbk{3dk_0/\eta})^2.
\]
with probability at least $1-\eta$, where $c$ is an absolute constant.
\end{proposition}
By ``absolute constant", we mean a constant that does not depend on the complexity metric parameters \(\epsilon, \fr{p}_\wedge, (1-\gamma)^{-1}, \eta, d\). Although the logarithmic term in the above proposition can be further improved, we will not focus on optimizing the logarithmic dependence in this paper. For clarity, we adjust the constant in the logarithmic factor using the inequality for \(C_1 \geq 1, C_2 \geq e\), \(\log(C_1C_2) = \log(C_1) + \log(C_2) \leq C_1 \log(C_2)\), and incorporate \(C_1\) into \(c\). These adjustments are applied to all subsequent convergence results.

\par The proof of this Proposition, which is outlined in Section \ref
{sec:analysis_of_algo_overview}, will be postponed to Section \ref
{subsec:proof_drql}.

Proposition \ref{prop:q-learning_err_high_prob_bd} provides an upper bound
on the terminal error in the estimator after $k_0$ iterations of Algorithm 
\ref{alg:q-learning}. This bound is given by three terms that decay with
rate $\tilde O(k_0^{-1})$, $\tilde O(n_0^{-1})$, and $\tilde
O((k_0n_0)^{-1/2})$, respectively, where the first and third terms resemble the upper bounds for standard Q-learning and the second term arises because of the bias. We optimize the algorithm parameters to
balance these three terms and ensure that the right-hand side of the
probability bound in Proposition \ref{prop:q-learning_err_high_prob_bd} is
less than $\epsilon$. One way to achieve this is by selecting the parameters 
$n_0$ and $k_0$ as follows: 
\begin{corollary}
\label{cor:ql_param_err_high_prob_bd}
Assume Assumptions \ref{assump:delta_small} and \ref{assump:max_rwd}. Running Algorithm \ref{alg:q-learning} with parameters
\begin{align*}
k_0 = c_0\frac{1}{(1-\gamma)^3\epsilon}\log\crbk{\frac{3d}{(1-\gamma)\eta \epsilon}}^2 ~\text{ and }~
n_0 =  c_0\frac{1}{\fr{p}_{\wedge}^3(1-\gamma)^2\epsilon}\log\crbk{3dk_0/\eta}^2
\end{align*}
will produce an output $q_{k_0+1}$ s.t. $\| q_{k_0+1}-q^*\|_\infty\leq \epsilon$ w.p. at least $1-\eta$, where $c$ is an absolute constant.
\end{corollary}
\par An immediate consequence of Corollary \ref{cor:ql_param_err_high_prob_bd} is
the following the worst-case sample complexity upper bound of the robust
Q-learning. 
\begin{theorem}\label{thm:ql_sample_complexity}
Assume Assumptions \ref{assump:delta_small} and \ref{assump:max_rwd}. Then the distributionally robust Q-learning Algorithm \ref{alg:q-learning} with parameters specified in Corollary \ref{cor:ql_param_err_high_prob_bd} computes a solution $ q_{k_0+1}$ s.t. $\| q_{k_0+1} - q^*\|_\infty$ w.p. at least $1-\eta$ using 
\[
\tilde O\crbk{\frac{|\bd{S}||\bd{A}|}{\fr{p}_{\wedge}^3(1-\gamma)^5\epsilon^2}}
\]
number of samples. 
\end{theorem}
\begin{proof}
The total number of samples used is $|\bd{S}||\bd{A}|n_0k_0$, implying the sample complexity upper bound. 
\end{proof}
Theorem \ref{thm:ql_sample_complexity} provides a near-optimal worst-case
sample complexity guarantee that matches and beats the expected sample
complexity upper bound in \citet{Wang2023MLMCDRQL} in all parameter
dependence. In particular, we have shown that the dependence on $\delta$ is $O(1)$ as $\delta\da 0$. This resolves the issue of the worst-case complexity
bound blowing up as $\delta\da 0$ for KL divergence based DR-RL that present in all prior works \citep{yang2021,Panaganti2021,ShiChi2022,Wang2023MLMCDRQL}.

\subsection{The Variance-Reduced Distributionally Robust Q-learning}

\label{subsec:dr_vrql} We adapt Wainwright's variance-reduced Q-learning \citep{wainwright2019} to
the robust RL setting.  This is outlined in Algorithm \ref{alg:vr_q-learning}. 
\begin{algorithm}[ht]
   \caption{Variance-Reduced Distributionally Robust Q-Learning}
   \label{alg:vr_q-learning}
\begin{algorithmic}
	\STATE {\bfseries Input:} the number of epochs $l_\mathrm{vr}$,  a sequence of recentering sample size $\{m_l\}_{l=1}^{l_\mathrm{vr}}$, an epoch length $k_\mathrm{vr}$  and a batch size $n_\mathrm{vr}$. 
   \STATE {\bfseries Initialization:}  $\hat q_0\equiv0$; $l = 1$; $k = 1$.
   \FOR{$1\leq l\leq l_\mathrm{vr}$}
   \STATE Compute $\widetilde{\mathbf{T}}_{l}$, $m_l$-sample empirical DR Bellman operator as in Definition \ref{def:empirical_Bellman_operator}.  
   \STATE Set $q_{l,1} = \hat q_{l-1}$.
   \FOR{$1\leq k\leq k_\mathrm{vr}$}
   \STATE Sample $\mathbf{T}_{l,k+1}$ an $n_\mathrm{vr}$-sample empirical Bellman operator. 
   \STATE Compute the recentered Q-learning update
   \begin{equation}\label{eqn:vr_q-learning_update}
   q_{l,k+1} = (1-\lambda_{k})q_{l,k} + \lambda_k\crbk{\mathbf{T}_{l,k+1}(q_{l,k}) - \mathbf{T}_{l,k+1}(\hat q_{l-1}) + \widetilde {\mathbf{T}}_{l}(\hat q_{l-1})}
   \end{equation}
   with stepsize $\lambda_k = 1/(1+(1-\gamma)k)$. 
   \ENDFOR
   \STATE Set $\hat q_{l} = q_{l,k_\mathrm{vr}+1}$.
   \ENDFOR
   \RETURN $\hat q_{l_\mathrm{vr}}$
\end{algorithmic}
\end{algorithm}

As in the Q-learning case, the update rule \eqref{eqn:vr_q-learning_update} can be implemented separately for each state-action pair. Thus, Algorithm \ref{alg:vr_q-learning} does not require storing or performing computations using the entire empirical kernel and reward distribution.

Before delving into the convergence rate theory of the DR variance-reduced
Q-learning, we provide an intuitive description of this variance reduction
scheme. The basic idea is to partition the algorithm into epochs. During each epoch, we perform a ``recentered'' version of stochastic approximation recursions with the aim of eliminating the variance component in the SA iteration (\eqref{eqn:sa_update}). Specifically, instead of approximating $q^*$ by one stochastic approximation, in each epoch, starting with an estimator $\hat q_{l-1}$, we recenter the SA procedure so that it approximates $\cT(q_{l-1})$. However, since $\cT$ is not known, we use $\widetilde {\mathbf{T}}_l(q_{l-1})$ as an natural estimator. By choosing a sequence of empirical DR Bellman operators with exponentially increasing sample sizes, we expect that the errors $\|\hat q_{l}-q^*\|_\infty$  decrease exponentially with high probability.

This indeed holds true for Algorithm \ref{alg:vr_q-learning}. The outer loop
produces a sequence of estimators ${\hat{q}_l,l\geq 1}$. We will show that if 
$\hat{q}_{l-1}$ is within some error from the optimal $q^*$, then $\hat{q}_{l} $ will satisfy a better concentration bound by a geometric factor. This
result is summarized in Proposition \ref{prop:one_vr_iter_high_prob_bd}. 
\par Denote the $\sigma$-field generated by the random samples used until the end of epoch $l$ by $\cF_{l}$. We define the conditional expectation $E_{l-1}[\cd] := E[\cd|\cF_{l-1}]$ and probability measure $P_{l-1}(\cd) := E_{l-1}[\1\set{\cd}]$.
\begin{proposition}\label{prop:one_vr_iter_high_prob_bd}
Assuming that Assumptions \ref{assump:delta_small} and \ref{assump:max_rwd} are satisfied. On $\set{\omega:\|\hat q_{l-1}-q^*\|_\infty\leq b}$ for some  $b\leq 1/(1-\gamma)$, under measure $P_{l-1}(\cd)(\omega)$, we have that there exists numerical constant $c$ s.t.
\begin{align*}
\|\hat q_{l} - q^*\|_\infty &\leq c\crbk{\frac{b}{(1-\gamma)^2k_\mathrm{vr}} +\frac{b}{\fr{p}_{\wedge}^{3/2}(1-\gamma)^{3/2}\sqrt{n_\mathrm{vr}k_\mathrm{vr}}} + 
\frac{b}{\fr{p}_{\wedge}^{3/2}(1-\gamma)\sqrt{n_\mathrm{vr}}} }\log\crbk{3 d k_\mathrm{vr}/\eta}^2 \\
&\quad +  c\frac{1}{\fr{p}_{\wedge}^{3/2} (1-\gamma)^2\sqrt{m_l}}\sqrt{\log(3 d/\eta)}
\end{align*}
w.p. at least $1-\eta$, provided that $m_l\geq 8\fr{p}_{\wedge}^{-2}\log(24d/\eta)$ and $n_\mrm{vr}\geq \fr{p}_\wedge\inv$. 
\end{proposition}

\par Proposition \ref{prop:one_vr_iter_high_prob_bd} implies that if the variance
reduced algorithm finds an approximation of $q^*$ with infinity norm $b$,
then the error after one epoch is improved accordingly with high
probability. This and the Markovian nature of the sequence $\set{\hat q_{l}}$
would imply a high probability bound for trajectories satisfying the
pathwise property $\set{\omega:\forall{l}\leq l_\mathrm{vr}:\|\hat q_{l}-q^*\|\leq b_l}
$. This is formalized by the next theorem where we use $b_l = 2^{-l}(1-\gamma)^{-1}$. 

Let us define the parameter choice: for sufficiently large $c_{\mrm{vr}}$ absolute constant that doesn't depend on the complexity metric parameters $\epsilon,\fr{p}_\wedge,(1-\gamma)\inv,\eta,d$,
define 
\begin{equation}  \label{eqn:param_choice_for_vrql}
\begin{aligned} 
l_\mathrm{vr} &=
\ceil{\log_2\crbk{\frac{1}{\epsilon(1-\gamma)}}},\\ 
k_{\mathrm{vr}} &= c_{\mrm{vr}}\frac{1}{(1-\gamma)^2}\log\crbk{\frac{3dl_\mathrm{vr}}{(1-\gamma)\eta}}^2,\\ n_\mathrm{vr} &=
c_{\mrm{vr}}\frac{1}{\fr{p}_{\wedge}^3(1-\gamma)^2}\log(3 d  k_\mathrm{vr} l_\mathrm{vr}/\eta)^4,\\ m_l &=
c_{\mrm{vr}}\frac{4^l}{\fr{p}^3_\wedge(1-\gamma)^2}\log(3 dl_\mathrm{vr}/\eta)^2. \end{aligned}
\end{equation}
Notice that evidently $m_l\geq 8\fr{p}_{\wedge}^{-2}\log(24d/\eta)$ and  $n_\mrm{vr}\geq \fr{p}_\wedge\inv$, satisfying the
requirement of Proposition \ref{prop:one_vr_iter_high_prob_bd}.
\begin{proposition}
  \label{prop:vr_algo_err_high_prob_bd}
Assume Assumptions \ref{assump:delta_small} and \ref{assump:max_rwd}.  For $\epsilon < (1-\gamma)^{-1} $, define parameters according to\eqref{eqn:param_choice_for_vrql}. Then, the sequence $\set{\hat q_l,0\leq l\leq l_\mathrm{vr}}$ produced by Algorithm \ref{alg:vr_q-learning} satisfies the pathwise property that $\|\hat q_l-q^*\|_\infty\leq 2^{-l}(1-\gamma)^{-1}$ for all $0\leq l\leq l_\mathrm{vr}$ w.p. at least $1-\eta$. In particular, the final estimator $\hat q_{l_\mathrm{vr}}$ satisfies $\|\hat q_{l_\mathrm{vr}}-q^*\|_\infty\leq 2^{-l_\mathrm{vr}}(1-\gamma)^{-1}$ w.p. at least $1-\eta$.   
\end{proposition}
\begin{remark}\label{remark:geometric_stepsize}
The base of geometric growth in our choice of \(m_l\) in \eqref{eqn:param_choice_for_vrql} can be modified. The same proof as in Proposition \ref{prop:vr_algo_err_high_prob_bd} suggests that with \(m_l = \alpha^{2l}\tilde\Theta(\fr{p}^{-3}_\wedge(1-\gamma)^{-2})\) and \(l_{\mrm{vr}} = \lceil \log_{\alpha}\left(\epsilon^{-1}(1-\gamma)^{-1}\right) \rceil\) for some \(\alpha > 1\), we have \(\|\hat q_l - q^*\|_\infty \leq \alpha^{-l}(1-\gamma)^{-1}\) for all \(0 \leq l \leq l_\mathrm{vr}\) with probability at least \(1-\eta\). Running Algorithm \ref{alg:vr_q-learning} with this new parameter choice will yield the same sample complexity as in Theorem \ref{thm:vrql_sample_complexity}. The choice of base 4 in \eqref{eqn:param_choice_for_vrql} was made only for clarity in our presentation.
\end{remark}
\par Proposition \ref{prop:vr_algo_err_high_prob_bd} immediately implies the
following worst-case sample complexity upper bound. 
\begin{theorem}\label{thm:vrql_sample_complexity}
Assume Assumptions \ref{assump:delta_small} and \ref{assump:max_rwd}. For $\epsilon < (1-\gamma)^{-1} $, the variance-reduced DR Q-learning Algorithm \ref{alg:vr_q-learning} with parameters specified in \eqref{eqn:param_choice_for_vrql} computes a solution $\hat q_{l_\mathrm{vr}}$ s.t. $\|\hat q_{l_\mathrm{vr}} - q^*\|_\infty\leq \epsilon$ w.p. at least $1-\eta$ using 
\[
\tilde O\crbk{\frac{|\bd{S}||\bd{A}|}{\fr{p}_{\wedge}^3(1-\gamma)^4\min(1,\epsilon^2)}}
\]
number of samples.
\end{theorem}
\begin{proof}
Given the specified parameters, the total number of samples used is
    \begin{align*}
     |\bd{S}||\bd{A}|\crbk{l_\mathrm{vr}n_\mathrm{vr}k_\mathrm{vr} + \sum_{l=1}^{l_\mathrm{vr}}m_l}= \tilde O\crbk{|\bd{S}||\bd{A}|\crbk{\frac{1}{\fr{p}_{\wedge}^3(1-\gamma)^4} + \frac{4^{l_\mathrm{vr}}}{\fr{p}_{\wedge}^3(1-\gamma)^2}}}
    \end{align*}
This simplifies to the claimed result. 
\end{proof}
Theorem \ref{thm:vrql_sample_complexity} establishes an upper bound of \(\tilde{O}\left(|\mathbf{S}||\mathbf{A}|(1-\gamma)^{-4}\epsilon^{-2}\fr{p}_{\wedge}^{-3}\right)\) when \(\epsilon \leq 1\), which is superior to the upper bound \(\tilde{O}\left(|\mathbf{S}||\mathbf{A}|(1-\gamma)^{-5}\epsilon^{-2}\fr{p}_{\wedge}^{-3}\right)\) for Algorithm \ref{alg:q-learning} (see Theorem \ref{thm:ql_sample_complexity}) in terms of \(1-\gamma\). This represents the best-known upper bound for DR-RL problems in the KL case, including both model-free and model-based algorithms \citep{ShiChi2022}. Although \citet{ShiChi2022} achieve a similar rate of \(\tilde{O}\left((1-\gamma)^{-4}\right)\), their result suffers from a \(\tilde{O}\left(\delta^{-2}\right)\) dependence, which becomes problematic as \(\delta \to 0\). In contrast, our upper bound is free from \(\delta\)-dependence.

\par We recall that the information-theoretical lower bound for the sample complexity of the classical tabular RL problem is \(\tilde{\Omega}\left(|\mathbf{S}||\mathbf{A}|(1-\gamma)^{-3}\epsilon^{-2}\right)\) \citep{azar2013}. In this setting, the variance-reduced Q-learning algorithm in \citet{wainwright2019} is minimax optimal.  For distributionally robust RL, \citet{ShiChi2022} recently showed that the minimax lower bound dependence on $|\mathbf{S}||\mathbf{A}|$, $(1-\gamma)\inv$, and $\epsilon$ remains \(\tilde{\Omega}\left(|\mathbf{S}||\mathbf{A}|(1-\gamma)^{-3}\epsilon^{-2}\right)\) when $\delta$ is small. Furthermore, \citet{shi2024curious} showed the information-theoretical lower bound may be \(\tilde{\Omega}\left(|\mathbf{S}||\mathbf{A}|(1-\gamma)^{-4}\epsilon^{-2}\right)\) when $\delta = O(1)$ for $\chi^2$-divergence uncertainty sets. However, their construction of hard instances violates our Assumption \ref{assump:delta_small}.  It is currently unknown whether variance-reduced DR Q-learning can achieve those rates. Further refinement of this bound is left for future research.

\par Notice that the variance-reduced Algorithm \ref{alg:vr_q-learning} has the property that \(k_{\mathrm{vr}}, n_{\mathrm{vr}},\) and \(m_l\) only depend on \(\frac{1}{\epsilon}\) through \(\log(l_{\mathrm{vr}}) = \Theta(\log \log \frac{1}{\epsilon})\). Therefore, within a reasonable range of \(\epsilon\), the algorithm can operate with the sample complexity guarantee in Theorem \ref{thm:vrql_sample_complexity} without needing to tune \(k_{\mathrm{vr}}, n_{\mathrm{vr}},\) and \(m_l\) based on \(\epsilon\). This introduces significant versatility in application: for example, we can continue to run the algorithm beyond termination epoch \(l_{\mathrm{vr}}\) without losing sample efficiency.
\subsection{Overview of the Analysis of Algorithms}

\label{sec:analysis_of_algo_overview} In this section, we provide a road map
to proving the key results, Proposition \ref{prop:q-learning_err_high_prob_bd}
and \ref{prop:one_vr_iter_high_prob_bd}. 
\begin{definition}
We say that $\cL$ is a monotonic $\gamma$-quasi-contraction with center $q'$ if
\begin{equation}
\|\cL(q)-\cL(q')\|_\infty \leq \gamma \|q-q'\|_\infty,
\label{Def:contraction:quasi}
\end{equation}
and entrywise
\begin{equation}
q_1\geq q_2 \implies \cL(q_1)\geq\cL(q_2)
\end{equation}
for all $q,q_1,q_2\in\R^{|\bd{S}|\times |\bd{A}|}$. Moreover, a monotonic $\gamma$-contraction is such that the above identities hold for all $q'\in\R^{|\bd{S}|\times |\bd{A}|}$. 
\end{definition}
The term \textit{quasi} refers to the fact that the relation \ref{Def:contraction:quasi} is only required for a single $q'$ \citep{wainwright2019l_infty}.
Therefore, a monotonic $\gamma$-contraction is a quasi-contraction with
center $q^{\prime }$ for any $q^{\prime }\in\R^{\bd{S}\times\bd{A}}$.

The successive application of monotonic $\gamma$-contractions under the
rescaled linear stepsize $\lambda_k = \frac{1}{1+(1-\gamma)k}$ will satisfy
the following deterministic bound: 
\begin{proposition}[Corollary 1, \citet{wainwright2019l_infty}]\label{prop:contraction_err_bd}
Let $\set{\cL_k,k\geq 2}$ be a family of monotonic $\gamma$-quasi-contractions with center $q'$. Let $\cH_k(q) = \cL_k(q)-\cL_k(q')$ the recentered operator. Then, for the sequence of step sizes $\set{\lambda_k,k\geq 1}$ the iterates of 
\begin{equation}\label{eqn:recentered_q_update}
q_{k+1} - q' = (1-\lambda_k)(q_{k}-q') + \lambda_{k}\sqbk{ \cH_{k+1}(q_{k}) + w_{k+1}}
\end{equation}
satisfies 
\[
\|q_{k+1} - q'\|_{\infty}\leq \lambda_k \sqbk{\frac{\|q_1-q'\|_\infty}{\lambda_1} + \gamma\sum_{j = 1}^{k} \|p_j\|_\infty } + \|p_{k+1}\|_\infty
\]
for all $k\geq 1$, where the sequence $\set{p_k,k\geq 1}$ is defined by $p_1= 0$ and
\[
p_{k+1}:= (1-\lambda_{k})p_k +\lambda_kw_{k+1}. 
\]
\end{proposition}

A key observation is that the empirical robust Bellman operators $\mathbf{T}_{k}, \widetilde{\mathbf{T}}_{l,k}$ used in the iterative updates of Algorithms \ref{alg:q-learning} and \ref{alg:vr_q-learning} are monotonic $\gamma$-contractions (see Lemma \ref{lemma:monotone_contraction}).

In the proof of the main results, we apply the deterministic bound for contraction mappings from Proposition \ref{prop:contraction_err_bd} to each sample path of the distributionally robust Q-learning and the inner loop of the variance-reduced version. We illustrate this by considering the distributionally robust Q-learning. Since \(\{\mathbf{T}_{k+1}, k \geq 0\}\) are monotonic \(\gamma\)-contractions, they are quasi-contractions with center \(q^*\). We can define $\mathbf{H}_{k+1}(q) := \mathbf{T}_{k+1}(q)-\mathbf{T}_{k+1}(q^*)$ for all $q\in\R^{\bd{S}\times\bd{A}}$. Then, the update rule of Algorithm \ref{alg:q-learning} can be written as 
\begin{align*}
q_{k+1}-q^* &= (1-\lambda_k)(q_k-q^*) + \lambda_k\sqbk{\crbk{\mathbf{T}_{k+1}(q_k) - \mathbf{T}_{k+1}(q^*) }+ \crbk{\mathbf{T}_{k+1}(q^*)-
\cT(q^*)}} \\
&= (1-\lambda_k)(q_k-q^*) + \lambda_k\sqbk{\mathbf{H}_{k+1}(q_k)+ W_{k+1}}.
\end{align*}
where $W_{k+1} := \mathbf{T}_{k+1}(q^*)- \cT(q^*)$ and we used the Bellman equation \eqref{eqn:dr_bellman_eqn_q} that $q^* = \cT(q^*)$.
\par This representation allow as to apply Proposition \ref{prop:contraction_err_bd} to bound the error of the $q$-function estimation using the sequence $P_1 = 0 $ and 
\begin{equation*}
P_{k+1}:= (1-\lambda_k)P_{k} + \lambda_k W_{k+1}.
\end{equation*}
Note that the only source of randomness in $W_{k}$ is from $\mathbf{T}_{k+1}(q^*)$, which are i.i.d.. Therefore, the process $P$ is a non-stationary auto-regressive (AR) process. It follows that the concentration properties of $P_k$ can be derived from that of $\mathbf{T}_{k+1}(q^*)$.

\par While standard Q-learning updates utilize an unbiased empirical Bellman operator, the DR empirical Bellman operator is biased due to its non-linearity in the empirical measure (c.f. \eqref{eqn:empirical_Bellman_operator}), resulting in \(E[W_{k}] \neq 0\). To achieve a canonical error rate of \(O(n^{-1/2})\), it is necessary that both the bias and standard deviation of the \(n\)-sample DR empirical Bellman operator are \(O(n^{-1/2})\). However, our DR Q-learning algorithms require an additional condition: the one-step bias must be of the order \(O(n^{-1})\). This is because the final bias, which is the systematic error resulting from the repeated use of the DR Bellman estimator, is compounded by the one-step bias through the model-free Q-learning updates. This imposes significant challenges on the design and analysis of our model-free algorithms.

\par Fortunately, we are able to establish tight bounds (in \(n_0\) and \(\delta\)) on the bias, c.f. Proposition \ref{prop:bias_empirical_Bellman}, in the important regime when \(\delta\) is small, as stated in Assumption \ref{assump:delta_small}. These bounds are central to our sample complexity analysis. We summarize the relevant bounds on the variance and bias of the empirical DR Bellman operator in Section \ref{sec:empirical_Bellman_operator}. By utilizing these variance and bias bounds, we can efficiently allocate samples such that the systematic error due to bias is balanced with the stochasticity in the estimator at the termination of the algorithm. With this optimal sample allocation, we can establish the worst-case sample complexity bounds as claimed.

The theory for the convergence rate of the variance-reduced DR Q-learning is more complex. In order to achieve the geometric convergence in Proposition \ref{prop:vr_algo_err_high_prob_bd}, an $O(n\inv)$ bias bound of the empirical DR Bellman operator is not enough. However, by introducing a recentered dynamics, a similar recursion can be derived in this context if we consider the \textit{conditionally recentered noise} \(\mathbf{H}_{l,k+1}(\hat q_{l-1}) - E[\mathbf{H}_{l,k+1}(\hat q_{l-1}) | \hat q_{l-1}]\) and a ``random bias" (denoted by $D_l$ in Appendix \ref{a_sec:proof_drvrql}). For details, please refer to Appendix \ref{a_sec:proof_drvrql}.

\section{Numerical Experiments}\label{sec:numerical_experiments}

This section presents a numerical validation of our theoretical findings regarding the convergence properties of the proposed algorithms. We conduct a comparative analysis between our algorithms and MLMC DR Q-learning, as studied in \citet{Wang2023MLMCDRQL}. Additionally, we investigate the complexity of Algorithm \ref{alg:vr_q-learning} as the adversary's power $\delta\da 0$.

Section \ref{sec:hard_MDP} demonstrates convergence and compares the proposed algorithms with multilevel Monte Carlo distributionally robust (MLMC DR) Q-learning. We use the hard MDP instances constructed in \citet{li2021QL_minmax}, where standard Q-learning performs at its worst-case complexity dependence of \(\tilde \Omega((1-\gamma)^{-4}\epsilon^{-2})\). Both algorithms in this paper show the canonical convergence rate of \(O(\epsilon^{-2})\), with the variance-reduced version displaying superior performance.

In Section \ref{sec:mixing_MDP}, we test the stability of sample complexity of the variance-reduced DR Q-learning Algorithm \ref{alg:vr_q-learning} as \(\delta \da 0\) using a simple DRMDP instance.

\par In the subsequent developments, we use \(m_l = 2^{l}(1-\gamma)^{-2}\) for the variance-reduced Algorithm \ref{alg:vr_q-learning}. As explained in Remark \ref{remark:geometric_stepsize}, this choice (up to a log factor) yields the same complexity guarantee as stated in Theorem \ref{thm:vrql_sample_complexity}. An advantage of this parameter choice is that it allows us to run more epochs for the plots, thereby clarifying the convergence behavior.

\subsection{Hard MDPs for the Q-learning}

\label{sec:hard_MDP} 
\begin{figure}[ht]
\centering
\includegraphics[width = 0.35\linewidth]{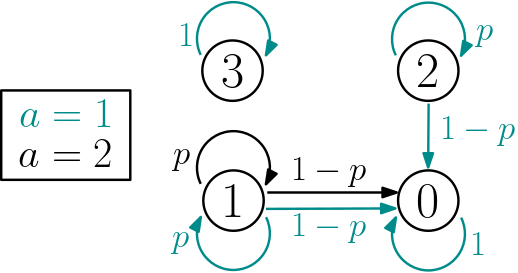}
\caption{Hard MDP for the Q-learning transition diagram.}
\label{fig:hard_mdp_instance}
\end{figure}
First, we demonstrate the convergence of the proposed algorithms using the MDP instance shown in Figure \ref{fig:hard_mdp_instance}. This MDP has 4 states and 2 actions, with transition probabilities given for actions 1 and 2 labeled on the arrows between states. Constructed in \citet{li2021QL_minmax}, it is shown in that when \( p = \frac{4\gamma - 1}{3\gamma} \), standard non-robust Q-learning will have a sample complexity of \(\tilde \Theta((1-\gamma)^{-4}\epsilon^{-2})\). 
\begin{figure}[ht]
\label{fig:hard_conv_test} \centering
\begin{subfigure}{0.48\linewidth}
    \centering
    \includegraphics[width =\linewidth]{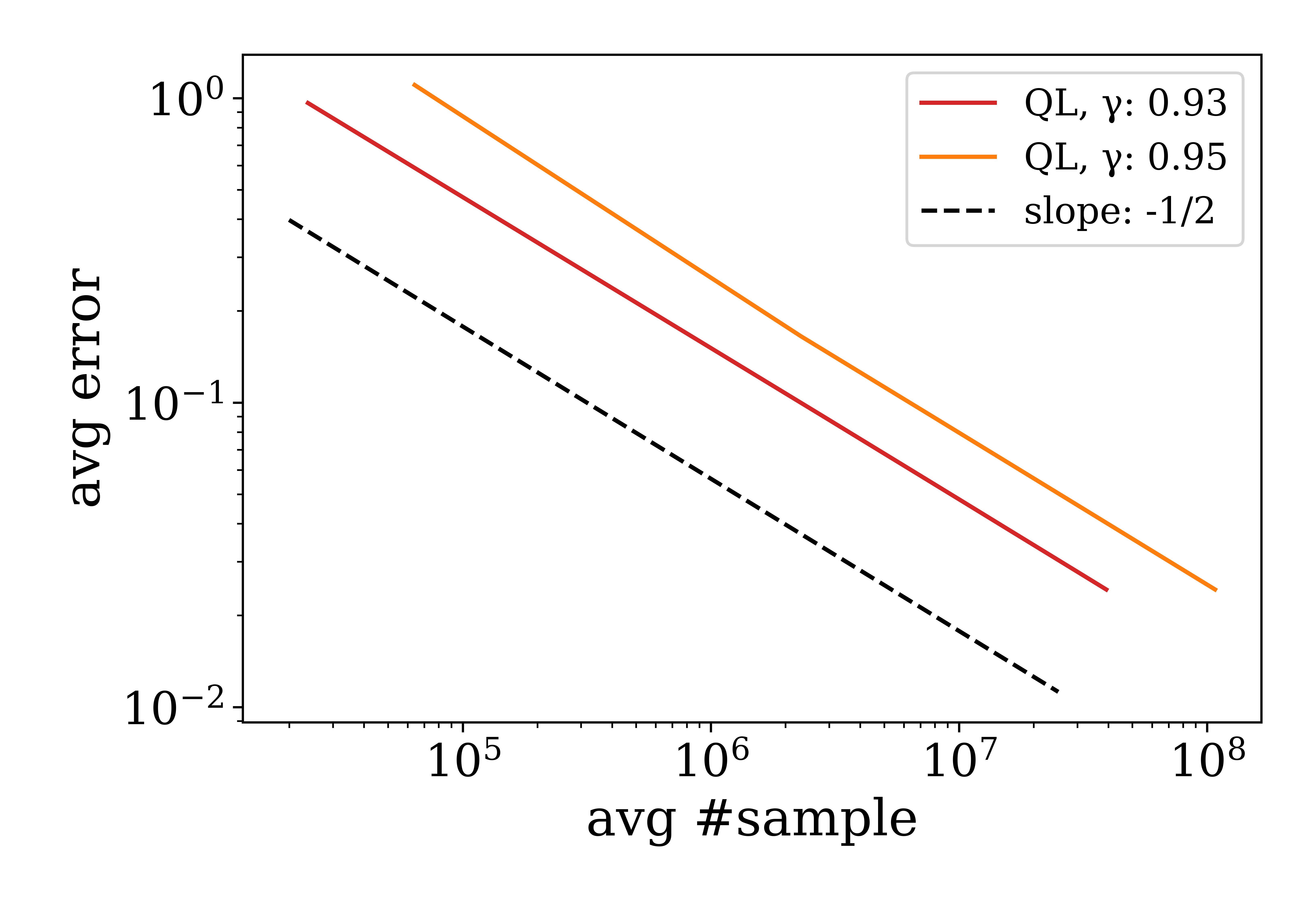}
    \caption{DR Q-learning}
    \label{fig:ql_hard_mdp_convergence}
\end{subfigure}  
\begin{subfigure}{0.48\linewidth}
    \centering
    \includegraphics[width = \linewidth]{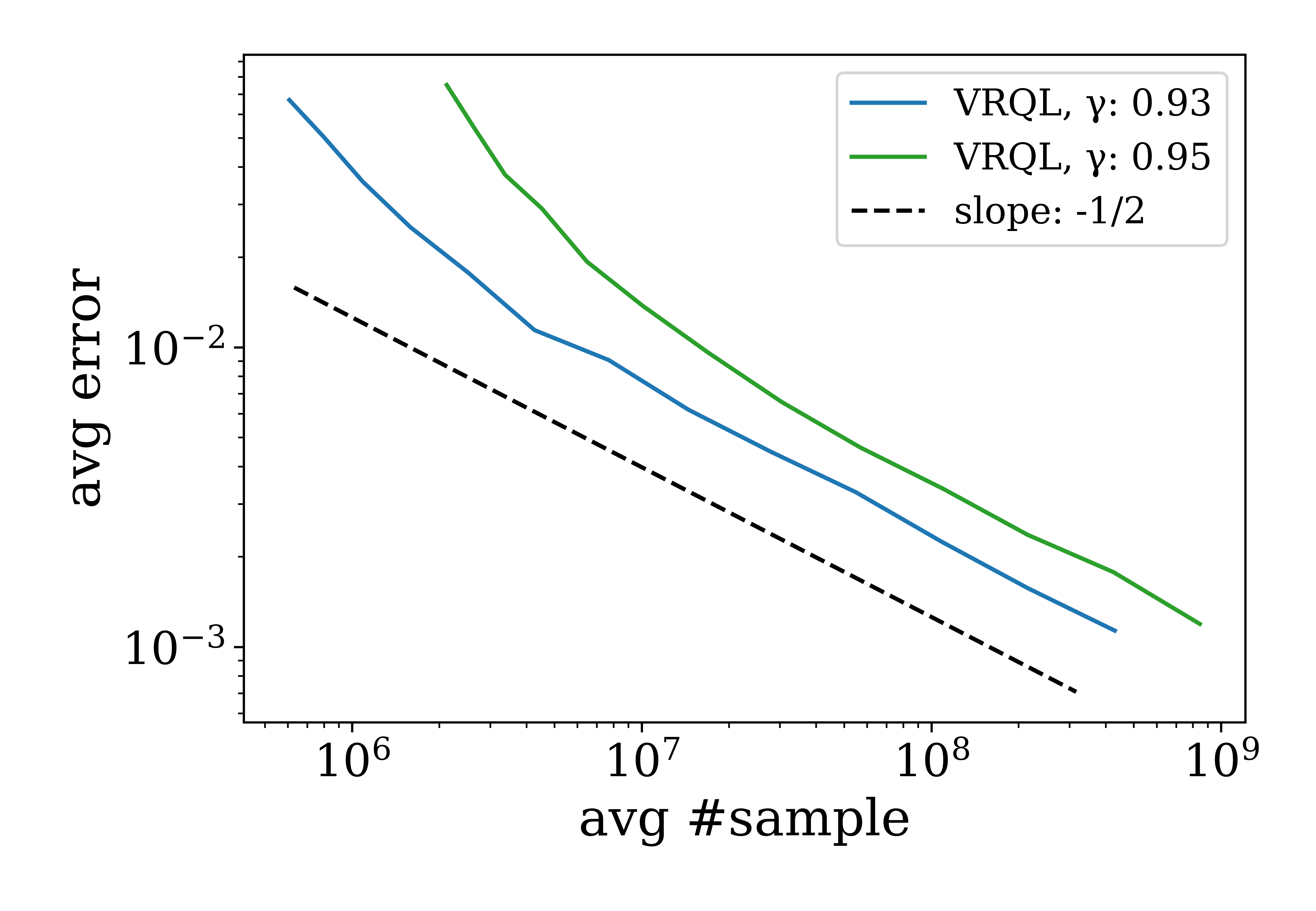}
    \caption{ variance-reduced DR Q-learning}
    \label{fig:vrql_hard_mdp_convergence}
\end{subfigure}
\caption{Convergence of Algorithm \protect\ref{alg:q-learning} and \protect
\ref{alg:vr_q-learning} on the MDP instance \protect\ref{fig:hard_mdp_instance}}
\end{figure}

\par Figures \ref{fig:ql_hard_mdp_convergence} and \ref{fig:vrql_hard_mdp_convergence} depict the convergence properties of the two algorithms for $\gamma = \set{0.93, 0.95}$ and $\delta = 0.1$. These figures show the (4000 samples) averaged error of the output $q$-function in the infinity norm plotted against the (4000 samples) averaged number of samples used, both on a log-log scale. The parameters for DR Q-learning in Figure \ref{fig:ql_hard_mdp_convergence} are set according to \ref{cor:ql_param_err_high_prob_bd}. On the other hand, Figure \ref{fig:vrql_hard_mdp_convergence} plots the averaged error achieved by the variance-reduced algorithm after each epoch against the total number of samples used. 
\par The figures indicate that both algorithms converge to the optimal robust \( q^* \), with the variance-reduced algorithm outperforming DR Q-learning. Additionally, when comparing the log-log error plot with a reference line having a slope of \(-1/2\), we observe that the log error for both algorithms decays at a rate of \(-1/2\) as the log of the samples increases. This behavior aligns with the \(\epsilon^{-2}\) dependence of the sample complexity bounds in Theorems \ref{thm:ql_sample_complexity} and \ref{thm:vrql_sample_complexity}, corresponding to the canonical convergence rate of Monte Carlo estimations, which is \(O(n^{-1/2})\). 

\begin{remark}
    With \(\delta = 0.1\) and \(\gamma = 0.93\) or \(0.95\), the DRMDP instances do not satisfy Assumption \ref{assump:delta_small}. However, the figures still show the canonical \(n^{-1/2}\) convergence rate, suggesting that our proposed algorithms might perform well even outside the regime prescribed by Assumption \ref{assump:delta_small}. 
\end{remark}

\begin{figure}[ht]
\includegraphics[width = \textwidth]{compare_conv_gamma.png}
\caption{Comparing the performance of Algorithm \protect\ref{alg:q-learning}, \protect\ref{alg:vr_q-learning} and the MLMC DR Q-learning on the MDP \protect
\ref{fig:hard_mdp_instance}. }
\label{fig:compare_conv}
\end{figure}
Figure \ref{fig:compare_conv} compares the performance of the algorithms proposed in this paper with the MLMC DR Q-learning in \citet{Wang2023MLMCDRQL}. We observe the performance comparison of three Q-learning methods: MLMC DR, DR, and DR-VR, for \(\gamma \in \{0.6, 0.7\}\). The results indicate that the distributionally robust variance-reduced Q-learning approach achieves the smallest errors. Although our DR Q-learning method shows slightly lower expected performance than the MLMC DR Q-learning, it is worth noting that the line corresponding to MLMC DR Q-learning is considerably rougher. This suggests that the MLMC DR Q-learning approach has a higher degree of variability in terms of performance.

\subsection[Testing the Small delta Regime]{Testing the Small $\delta$ Regime}\label{sec:mixing_MDP}

\par We proceed to empirically demonstrate the stability of the sample complexity of Algorithm \ref{alg:vr_q-learning} as \(\delta \da 0\). 
\par First, we introduce a family of MDPs instance. Define reference MPDs with $S = \set{1,2}$, $A = \set{a_1,a_2}$, transition kernel
\begin{equation}\label{eqn:simple_MDP}
P_{0,a_1} = P_{0,a_2} = \bmx{1/2 & 1/2\\ 1/2 & 1/2},
\end{equation}
and deterministic reward function $r(1,\cd) = 1$ and $r(2,\cd) = 0$. For any positive adversarial power level $\delta$, the worst-case transition kernel chosen by the adversary is
\begin{equation*}
P_{\delta,a_1} = P_{\delta,a_2} = \bmx{q(\delta) & 1-q(\delta)\\ q(\delta) & 1-q(\delta)}
\end{equation*}
where $q(\delta) < 1/2$ and $q(\delta)\ua 1/2$ as $\delta\da 0$. In a classical tabular RL setting, this worst-case MDP ($\delta > 0$) should be easier to learn compared to \eqref{eqn:simple_MDP}, c.f. \citep{khamaru2021,wang2023mixing}. 

\begin{figure}[t]
\centering
\includegraphics[width = \textwidth]{test_delta.png}
\caption{Testing the sample complexity behavior as $\delta\da 0$. }
\label{fig:test_delta}
\end{figure}
\par Using this DRMDP instance, we plot the average number of samples required to achieve a fixed error \(\epsilon\) while varying \(\delta\), as shown in Figure \ref{fig:test_delta}. We observe that the average number of samples increases as \(\delta \da 0\), because the worst-case MDP converges to the instance in \eqref{eqn:simple_MDP}, which is more challenging to learn. Additionally, the number of samples needed to reach the target error level becomes insensitive to increasingly small \(\delta\) when \(\delta \leq 10^{-2}\), confirming the theoretical results presented in this paper.

\section[Extension: chi2 Divergence Ambiguity Sets]{Extension: $\chi_2$ Divergence Ambiguity Sets}

We extend the variance-reduced version of the Q-learning Algorithm \ref{alg:vr_q-learning} to the setting where the adversary is constrained to perturbations within $\chi_2$ divergence balls of radius $\delta$. The $\chi_2$ divergence is defined for $Q\ll P$ as 
\begin{equation}\label{eqn:def_chi2}
    D_{\chi_2}(Q\|P) := \frac{1}{2} \int_\Omega\crbk{\frac{dQ}{dP}(\omega)-1}^2P(d\omega).
\end{equation} 
Note that we follow the convention in \citet{duchi2021learning} to include an $1/2$ in \eqref{eqn:def_chi2}. 
\par We reuse the notation for the KL case in the discussion of this section. In particular, for each $(s,a)\in \bd{S}\times\bd{A}$ and $\delta > 0$, we define $\chi_2$ ambiguity sets analogous to \eqref{eqn:def_KL_delta_balls} as
\begin{equation}\label{eqn:def_chi2_delta_balls}\begin{aligned}
\cP_{s,a}(\delta)&\coloneqq \left\lbrace p : D_{\chi_2}\left(p\|p_{s,a}\right)\leq \delta\right\rbrace, \\
\cN_{s,a}(\delta)&\coloneqq \left\lbrace \nu: D_{\chi_2}(\nu\|\nu_{s,a})\leq\delta\right\rbrace. \end{aligned}
\end{equation}
\par For $\chi_2$ divergence defined in \eqref{eqn:def_chi2}, we have the following strong duality. 
\begin{lemma}[\citet{duchi2021learning}, Lemma 1]\label{lemma:strong_dual_chi2_duchi}
Let $X$ be a random variable and $\mu_0$ a probability measure on $(\Omega,\cF)$.  Then, for any $\delta > 0$, 
\begin{equation}\label{eqn:strong_dual_chi2_duchi}
\inf_{\mu:D_{\chi_2}(\mu\|\mu_0)\leq \delta}E_\mu X = \sup_{\alpha\in\R}\set{\alpha-c(\delta)E_{\mu_0}\sqbk{(\alpha-X)_+^2}^{\frac{1}{2}}}
\end{equation}
where $c(\delta) =\sqrt{1+2\delta}$ and $(\cd)_+:= \max\set{\cd,0}$. 
\end{lemma}
We note that the dual variable $\alpha$ can be optimized within $\alpha \geq \essinf_{\mu_0}X$. 
\par We wish to learn the optimal $q$-function as defined in \eqref{eqn:def_q*_func}. To achieve this, we use the DR Bellman equation for the $q$-function \eqref{eqn:dr_bellman_eqn_q} where the dual form of the Bellman operator $\cT:\R^{\bd{S}\times \bd{A}}\ra \R^{\bd{S}\times \bd{A}}$ in the $\chi_2$ case is given by
\begin{equation}\label{eqn:def_bellman_op_chi2_dual}
    \cT(q)(s,a):= \sup_{\alpha\in\R }\left\lbrace \alpha-c(\delta) E_{\nu_{s,a}}\sqbk{(\alpha-R)_+^2}^{\frac{1}{2}}\right\rbrace+ \gamma\sup_{\beta\in\R }\left\lbrace \beta - c(\delta) E_{p_{s,a}}\sqbk{(\beta-v(q)(S))_+^2}^{\frac{1}{2}} \right\rbrace.
\end{equation}
Then, the empirical Bellman operator $\bd{T}$ is similarly defined as in \eqref{eqn:empirical_Bellman_operator} using this dual representation as
\begin{equation}\label{eqn:empirical_Bellman_operator_chi2}
\begin{aligned}
\mathbf{T}(q)(s,a)&\coloneqq \sup_{\alpha\in\R }\left\lbrace \alpha-c(\delta) E_{\nu_{s,a,n}}\sqbk{(\alpha-R)_+^2}^{\frac{1}{2}}\right\rbrace\\
&\quad + \gamma\sup_{\beta\in\R }\left\lbrace \beta - c(\delta) E_{p_{s,a,n}}\sqbk{(\beta-v(q)(S))_+^2}^{\frac{1}{2}} \right\rbrace,
\end{aligned}
\end{equation}
where the empirical measures and expectations are defined in \eqref{eqn:empirical_meas_exp}.

\par Recall the definition of the minimum support probability $\fr{p}_{\wedge}$ in \eqref{eqn:min_supp_prob}. As in the KL case, we also consider the regime $\delta = O(\delta)$:
\begin{assumption}[Limited Adversarial Power] \label{assump:delta_small_chi2}
Suppose the adversary's power $\delta<\frac{1}{2}\fr{p}_{\wedge}$.
\end{assumption}
\par In this context, we will apply the variance-reduced Q-learning Algorithm \ref{alg:vr_q-learning} with the following parameters.
\begin{equation}  \label{eqn:param_choice_for_vrql_chi2}
\begin{aligned} 
l_\mathrm{vr} &=
\ceil{\log_2\crbk{\frac{1}{\epsilon(1-\gamma)}}},\\ 
k_{\mathrm{vr}} &= c_{\mrm{vr}}\frac{1}{(1-\gamma)^2}\log\crbk{\frac{3dl_\mathrm{vr}}{(1-\gamma)\eta}}^2,\\ n_\mathrm{vr} &=
c_{\mrm{vr}}\frac{1}{\fr{p}_{\wedge}^2(1-\gamma)^2}\log(3 d  k_\mathrm{vr} l_\mathrm{vr}/\eta)^4,\\ m_l &=
c_{\mrm{vr}}\frac{4^l}{\fr{p}^2_\wedge(1-\gamma)^2}\log(3 dl_\mathrm{vr}/\eta)^2. \end{aligned}
\end{equation}

\par Notice that, compare to the specifications in \eqref{eqn:param_choice_for_vrql}, \eqref{eqn:param_choice_for_vrql_chi2} has a $\fr{p}_\wedge^{-2}$ dependence instead of $\fr{p}_\wedge^{-3}$. Running Algorithm \ref{alg:vr_q-learning} with these parameters will yield an estimate \(\hat{q}_{l_\mathrm{vr}}\) of \(q^*\) with an error of at most \(\epsilon\) with high probability, leading to the following theorem.
\begin{theorem}\label{thm:vrql_sample_complexity_chi2}
Assume Assumptions \ref{assump:max_rwd} and \ref{assump:delta_small_chi2}. For $\epsilon < (1-\gamma)^{-1} $, the variance-reduced DR Q-learning Algorithm \ref{alg:vr_q-learning} with parameters specified in \eqref{eqn:param_choice_for_vrql} computes a solution $\hat q_{l_\mathrm{vr}}$ s.t. $\|\hat q_{l_\mathrm{vr}} - q^*\|_\infty\leq \epsilon $ w.p. at least $1-\eta$ using 
\[
\tilde O\crbk{\frac{|\bd{S}||\bd{A}|}{\fr{p}_{\wedge}^2(1-\gamma)^4\min(1,\epsilon^2)}}
\]
number of samples.
\end{theorem}
The proof of Theorem \ref{thm:vrql_sample_complexity_chi2} closely follows the proof of Theorem \ref{thm:vrql_sample_complexity}. We first establish the analog of Proposition \ref{prop:one_vr_iter_high_prob_bd} and then apply it to achieve the statement in Proposition \ref{prop:vr_algo_err_high_prob_bd} using the parameters in \eqref{eqn:param_choice_for_vrql_chi2} for the \(\chi_2\) divergence ambiguity set case. The sample complexity is then derived by summing the number of samples used in each epoch. This procedure is carried out in Appendix \ref{a_sec:analysis_of_vrql_alg_chi2}.
\subsection*{Acknowledgement}
The work is generously supported by the funding of NSF grants 2312204 and 2312205.

Material in this paper is based upon work supported by the Air Force Office of Scientific Research under award number FA9550-20-1-0397. Additional support is gratefully acknowledged from NSF grants 1915967 and 2118199.

This work is supported in part by National
Science Foundation grant CCF-2106508. Zhengyuan Zhou acknowledges the generous
support from New York University’s 2022-2023 Center for Global Economy and Business
faculty research grant.

\bibliographystyle{apalike}
\bibliography{bibs/mdp_complexity,bibs/dr_mdp_rl,bibs/dro,bibs/stochastic_opt}

\begin{thebibliography}{}

\bibitem[Agarwal et~al., 2020]{agarwal2020}
Agarwal, A., Kakade, S., and Yang, L.~F. (2020).
\newblock Model-based reinforcement learning with a generative model is minimax optimal.
\newblock In Abernethy, J. and Agarwal, S., editors, {\em Proceedings of Thirty Third Conference on Learning Theory}, volume 125 of {\em Proceedings of Machine Learning Research}, pages 67--83. PMLR.

\bibitem[Azar et~al., 2013]{azar2013}
Azar, M.~G., Munos, R., and Kappen, H.~J. (2013).
\newblock Minimax pac bounds on the sample complexity of reinforcement learning with a generative model.
\newblock {\em Mach. Learn.}, 91(3):325–349.

\bibitem[Chen et~al., 2020]{chen2020}
Chen, Z., Maguluri, S.~T., Shakkottai, S., and Shanmugam, K. (2020).
\newblock Finite-sample analysis of contractive stochastic approximation using smooth convex envelopes.
\newblock In Larochelle, H., Ranzato, M., Hadsell, R., Balcan, M., and Lin, H., editors, {\em Advances in Neural Information Processing Systems}, volume~33, pages 8223--8234. Curran Associates, Inc.

\bibitem[Chen et~al., 2022]{chen2022}
Chen, Z., Zhang, S., Doan, T.~T., Clarke, J.-P., and Maguluri, S.~T. (2022).
\newblock Finite-sample analysis of nonlinear stochastic approximation with applications in reinforcement learning.
\newblock {\em Automatica}, 146:110623.

\bibitem[Duchi and Namkoong, 2021]{duchi2021learning}
Duchi, J.~C. and Namkoong, H. (2021).
\newblock Learning models with uniform performance via distributionally robust optimization.
\newblock {\em The Annals of Statistics}, 49(3):1378--1406.

\bibitem[Even-Dar et~al., 2003]{even2003learning}
Even-Dar, E., Mansour, Y., and Bartlett, P. (2003).
\newblock Learning rates for q-learning.
\newblock {\em Journal of machine learning Research}, 5(1).

\bibitem[Fran{\c{c}}ois-Lavet et~al., 2018]{franccois2018introduction}
Fran{\c{c}}ois-Lavet, V., Henderson, P., Islam, R., Bellemare, M.~G., Pineau, J., et~al. (2018).
\newblock An introduction to deep reinforcement learning.
\newblock {\em Foundations and Trends{\textregistered} in Machine Learning}, 11(3-4):219--354.

\bibitem[Gonz\'{a}lez-Trejo et~al., 2002]{gonzalez-trejo2002}
Gonz\'{a}lez-Trejo, J.~I., Hern\'{a}ndez-Lerma, O., and Hoyos-Reyes, L.~F. (2002).
\newblock Minimax control of discrete-time stochastic systems.
\newblock {\em SIAM Journal on Control and Optimization}, 41(5):1626--1659.

\bibitem[Hu and Hong, 2013]{Hu2012KLDRO}
Hu, Z. and Hong, L.~J. (2013).
\newblock Kullback-leibler divergence constrained distributionally robust optimization.
\newblock {\em Available at Optimization Online}.

\bibitem[Iyengar, 2005]{Iyengar2005}
Iyengar, G. (2005).
\newblock Robust dynamic programming.
\newblock {\em Math. Oper. Res.}, 30:257--280.

\bibitem[Karimi et~al., 2019]{karimi2019biased_SA}
Karimi, B., Miasojedow, B., Moulines, E., and Wai, H.-T. (2019).
\newblock Non-asymptotic analysis of biased stochastic approximation scheme.

\bibitem[Khamaru et~al., 2021]{khamaru2021}
Khamaru, K., Xia, E., Wainwright, M.~J., and Jordan, M.~I. (2021).
\newblock Instance-optimality in optimal value estimation: Adaptivity via variance-reduced q-learning.

\bibitem[Kushner and Yin, 2013]{kushner2013SA}
Kushner, H. and Yin, G. (2013).
\newblock {\em Stochastic Approximation and Recursive Algorithms and Applications}.
\newblock Stochastic Modelling and Applied Probability. Springer New York.

\bibitem[Li et~al., 2021]{li2021QL_minmax}
Li, G., Cai, C., Chen, Y., Gu, Y., Wei, Y., and Chi, Y. (2021).
\newblock Is q-learning minimax optimal? a tight sample complexity analysis.
\newblock {\em arXiv preprint arXiv:2102.06548}.

\bibitem[Li et~al., 2022]{li2022settling}
Li, G., Shi, L., Chen, Y., Chi, Y., and Wei, Y. (2022).
\newblock Settling the sample complexity of model-based offline reinforcement learning.

\bibitem[Li et~al., 2023]{li2023statistical}
Li, X., Yang, W., Liang, J., Zhang, Z., and Jordan, M.~I. (2023).
\newblock A statistical analysis of polyak-ruppert averaged q-learning.
\newblock In {\em International Conference on Artificial Intelligence and Statistics}, pages 2207--2261. PMLR.

\bibitem[Lin et~al., 2017]{lin2017tactics}
Lin, Y.-C., Hong, Z.-W., Liao, Y.-H., Shih, M.-L., Liu, M.-Y., and Sun, M. (2017).
\newblock Tactics of adversarial attack on deep reinforcement learning agents.
\newblock {\em arXiv preprint arXiv:1703.06748}.

\bibitem[Liu et~al., 2022]{liu22DRQ}
Liu, Z., Bai, Q., Blanchet, J., Dong, P., Xu, W., Zhou, Z., and Zhou, Z. (2022).
\newblock Distributionally robust $q$-learning.
\newblock In Chaudhuri, K., Jegelka, S., Song, L., Szepesvari, C., Niu, G., and Sabato, S., editors, {\em Proceedings of the 39th International Conference on Machine Learning}, volume 162 of {\em Proceedings of Machine Learning Research}, pages 13623--13643. PMLR.

\bibitem[Nilim and El~Ghaoui, 2005]{nilim2005robust}
Nilim, A. and El~Ghaoui, L. (2005).
\newblock Robust control of markov decision processes with uncertain transition matrices.
\newblock {\em Operations Research}, 53(5):780--798.

\bibitem[Pan et~al., 2019]{pan2019characterizing}
Pan, X., Xiao, C., He, W., Yang, S., Peng, J., Sun, M., Yi, J., Yang, Z., Liu, M., Li, B., et~al. (2019).
\newblock Characterizing attacks on deep reinforcement learning.
\newblock {\em arXiv preprint arXiv:1907.09470}.

\bibitem[Panaganti and Kalathil, 2021]{Panaganti2021}
Panaganti, K. and Kalathil, D. (2021).
\newblock Sample complexity of robust reinforcement learning with a generative model.

\bibitem[Quinonero-Candela et~al., 2008]{quinonero2008dataset}
Quinonero-Candela, J., Sugiyama, M., Schwaighofer, A., and Lawrence, N.~D. (2008).
\newblock {\em Dataset shift in machine learning}.
\newblock Mit Press.

\bibitem[Rigollet, 2015]{rigollet201518}
Rigollet, P. (2015).
\newblock “18. s997: High dimensional statistics lecture notes.

\bibitem[Shapiro, 2022]{shapiro2022}
Shapiro, A. (2022).
\newblock Distributionally robust modeling of optimal control.
\newblock {\em Operations Research Letters}, 50(5):561--567.

\bibitem[Shapiro et~al., 2014]{sharpioBookSP}
Shapiro, A., Dentcheva, D., and Ruszczyński, A. (2014).
\newblock {\em Lectures on Stochastic Programming: Modeling and Theory, Second Edition}.
\newblock Society for Industrial and Applied Mathematics, Philadelphia, PA.

\bibitem[Shi and Chi, 2022]{ShiChi2022}
Shi, L. and Chi, Y. (2022).
\newblock Distributionally robust model-based offline reinforcement learning with near-optimal sample complexity.

\bibitem[Shi et~al., 2024]{shi2024curious}
Shi, L., Li, G., Wei, Y., Chen, Y., Geist, M., and Chi, Y. (2024).
\newblock The curious price of distributional robustness in reinforcement learning with a generative model.
\newblock {\em Advances in Neural Information Processing Systems}, 36.

\bibitem[Si et~al., 2020]{si2020}
Si, N., Zhang, F., Zhou, Z., and Blanchet, J. (2020).
\newblock Distributionally robust policy evaluation and learning in offline contextual bandits.
\newblock In {\em International Conference on Machine Learning}, pages 8884--8894. PMLR.

\bibitem[Sidford et~al., 2018]{sidford2018near_opt}
Sidford, A., Wang, M., Wu, X., Yang, L., and Ye, Y. (2018).
\newblock Near-optimal time and sample complexities for solving markov decision processes with a generative model.
\newblock In Bengio, S., Wallach, H., Larochelle, H., Grauman, K., Cesa-Bianchi, N., and Garnett, R., editors, {\em Advances in Neural Information Processing Systems}, volume~31. Curran Associates, Inc.

\bibitem[Sutton and Barto, 2018]{sutton2018reinforcement}
Sutton, R.~S. and Barto, A.~G. (2018).
\newblock {\em Reinforcement learning: An introduction}.
\newblock MIT press.

\bibitem[Wainwright, 2019a]{wainwright2019l_infty}
Wainwright, M.~J. (2019a).
\newblock Stochastic approximation with cone-contractive operators: Sharp $\ell_\infty$-bounds for $q$-learning.

\bibitem[Wainwright, 2019b]{wainwright2019}
Wainwright, M.~J. (2019b).
\newblock Variance-reduced $q$-learning is minimax optimal.

\bibitem[Wang, 2022]{Wang2022biased_SA}
Wang, G. (2022).
\newblock Finite-time error bounds of biased stochastic approximation with application to td-learning.
\newblock {\em IEEE Transactions on Signal Processing}, 70:950--962.

\bibitem[Wang et~al., 2023a]{wang2023mixing}
Wang, S., Blanchet, J., and Glynn, P. (2023a).
\newblock Optimal sample complexity of reinforcement learning for uniformly ergodic discounted markov decision processes.

\bibitem[Wang et~al., 2023b]{Wang2023MLMCDRQL}
Wang, S., Si, N., Blanchet, J., and Zhou, Z. (2023b).
\newblock A finite sample complexity bound for distributionally robust q-learning.

\bibitem[Wang et~al., 2024]{wang2024foundation}
Wang, S., Si, N., Blanchet, J., and Zhou, Z. (2024).
\newblock On the foundation of distributionally robust reinforcement learning.

\bibitem[Watkins and Dayan, 1992]{watkins1992q}
Watkins, C.~J. and Dayan, P. (1992).
\newblock Q-learning.
\newblock {\em Machine learning}, 8:279--292.

\bibitem[Wiesemann et~al., 2013]{wiesemann2013}
Wiesemann, W., Kuhn, D., and Rustem, B. (2013).
\newblock Robust markov decision processes.
\newblock {\em Mathematics of Operations Research}, 38(1):153--183.

\bibitem[Xu and Mannor, 2010]{huan2010}
Xu, H. and Mannor, S. (2010).
\newblock Distributionally robust markov decision processes.
\newblock In {\em Advances in Neural Information Processing Systems}, pages 2505--2513.

\bibitem[Yang et~al., 2023]{yang2023avoiding}
Yang, W., Wang, H., Kozuno, T., Jordan, S.~M., and Zhang, Z. (2023).
\newblock Avoiding model estimation in robust markov decision processes with a generative model.

\bibitem[Yang et~al., 2021]{yang2021}
Yang, W., Zhang, L., and Zhang, Z. (2021).
\newblock Towards theoretical understandings of robust markov decision processes: Sample complexity and asymptotics.

\bibitem[Zhou et~al., 2021]{zhou21}
Zhou, Z., Zhou, Z., Bai, Q., Qiu, L., Blanchet, J., and Glynn, P. (2021).
\newblock Finite-sample regret bound for distributionally robust offline tabular reinforcement learning.
\newblock In Banerjee, A. and Fukumizu, K., editors, {\em Proceedings of The 24th International Conference on Artificial Intelligence and Statistics}, volume 130 of {\em Proceedings of Machine Learning Research}, pages 3331--3339. PMLR.

\end{thebibliography}
\newpage

\appendix
\appendixpage

\section{The Empirical Robust Bellman Operator: KL Case}

\label{sec:empirical_Bellman_operator}

In this section, we establish the bias and concentration properties of the
empirical DR Bellman operator. As pointed out in the previous sections, they
are the key ingredients for proving our near-optimal sample complexity
bounds. Let $\widehat{\mathbf{T}}_n$ be the empirical DR Bellman operator formed by $n$ samples defined in Definition \ref{def:empirical_Bellman_operator}. To
simplify the notation, we will omit the subscript $n$ and only keep $\widehat{\mathbf{T}}$ when there is no confusion. 

Even though the main results of this paper restrict \(\bd{R} \subset [0,1]\) to simplify notation and align with convention in the literature, in this section, we consider \(\bd{R} \subset [0,\rmax]\). This allows our results to be directly applied to contexts beyond RL, such as supervised learning where \(\rmax\) may vary.

In order to employ the analysis outlined in the previous section, the
empirical Bellman operators need to be contraction mappings. Indeed, we have 
\begin{lemma}\label{lemma:monotone_contraction}
$\widehat{\mathbf{T}}$ is a monotonic $\gamma$-contraction. 
\end{lemma}
Direct consequences of $\widehat{\mathbf{T}}$ being a $\gamma$-contraction with $\gamma <
1$ are the following bounds: 
\begin{lemma}\label{lemma:prob_1_bound_empirical_Bellman}
The following two bounds hold with probability 1:
\[
\|\widehat{\mathbf{T}}(q)(s,a) - \cT(q)(s,a)\|_{\infty} \leq 2(r_{\max} + \|q\|_{\infty});
\]
and
\[
\|q_*\|_{\infty} \leq \frac{r_{\max}}{1-\gamma}.
\]
\end{lemma}

As motivated in the paper, to obtain a desired
complexity dependence on problem primitives, we need to develop good bounds
on the bias and the variance of the empirical Bellman operator. We define
the span seminorm of the $q$-function as $\spnorm{q} := \max_{s,a}q(s,a) -
\min_{s,a}q(s,a)$ and $\spnorm{q}\leq (1-\gamma)^{-1}$. 
The proofs of the following propositions are in Appendix \ref{proof:propostions}.
\begin{proposition}\label{prop:variance_empirical_Bellman}  The empirical DR Bellman operator satisfies the following variance bound:
\[
\var(\widehat{\mathbf{T}}(q)(s,a))\leq 104\frac{\rmax^2 + \gamma^2\spnorm{q}^2}{\fr{p}_{\wedge}^2n}(\log(e(|\bd{R}|\vee|\bd{S}|))).
\]
\end{proposition}
We note that here $\fr{p}_{\wedge}$ can be replaced by $\min_{s'\in S} \min\set{p_{s,a}(s'),\nu_{s,a}(s')}$. In particular, the variance upper bound can depend on the state and action. However, since we are interested in a minimax complexity bound, such distinction will not make a difference if we consider an example with only $O(1)$ number of states and actions. 
\par We also have the following bound on the bias: 
\begin{proposition}\label{prop:bias_empirical_Bellman} 
Under Assumption \ref{assump:delta_small}, the empirical DR Bellman Operator satisfies the following bias bound: 
\[
|\bias(\widehat{\mathbf{T}}(q)(s,a))|:= | E[\widehat{\mathbf{T}}(q)(s,a)]-\cT(q)(s,a)|\leq 4480\frac{\rmax + \gamma\spnorm{q}}{\fr{p}_{\wedge}^3 n}\log(e|\bd{S}|\vee |\bd{R}|).
\]
\end{proposition}
Again, the dependence on  $\fr{p}_{\wedge}$ can be  replaced by $\min_{s'\in S} \min\set{p_{s,a}(s'),\nu_{s,a}(s')}$.
\par By Lemma \ref{lemma:prob_1_bound_empirical_Bellman}, the DR empirical
Bellman operator is bounded. This along with the uniform (across $s,a\in
\bd{S}\times\bd{A}$) variance bound in Proposition \ref{prop:variance_empirical_Bellman} yields: 
\begin{proposition}\label{prop:high_prob_bound_empirical_Bellman}
The empirical DR Bellman operator 
\[
\|\widehat{\mathbf{T}}(q)-\cT(q)\|_\infty \leq \frac{17(\rmax+\gamma\spnorm{q})}{\fr{p}_{\wedge}\sqrt{n}}\sqrt{\log\crbk{6|\bd{S}||\bd{A}|(|\bd{S}|\vee |\bd{R}|)/\eta}}
\]
w.p. at least $1 - \eta$, provided that $n\geq 8\fr{p}_{\wedge}^{-2}\log\crbk{12|\bd{S}||\bd{A}|(|\bd{S}|\vee |\bd{R}|)/\eta}$. 
\end{proposition}

Recall that for fixed $\hat q$, we haved defined the recentered DR Bellman
operators
\begin{equation}\label{eqn:def_recentered_op}
\widehat{\mathbf{H}}(\hat q) := \widehat{\mathbf{T}}(\hat q)-\widehat{\mathbf{T}}(q^*) \quad \text{and}\quad \cH(\hat q) := \cT(\hat q)-\cT(q^*).
\end{equation}
For the variance-reduced algorithm, we instead consider the bias and
concentration properties of the recentered operator $\widehat{\mathbf{H}}$. As we will
observe, the recentering allows the concentration bounds to depend on the
residual error in the $q$-function $\|\hat q-q^*\|_\infty$ instead of $\|\hat q\|_\infty$. As a consequence, one can imagine that as the algorithm
progresses, $\|\hat q_l-q^*\|_\infty$ will progressively become smaller,
making $\widehat{\mathbf{H}}$ having much better concentration properties than $\widehat{\mathbf{T}}$.

We start with bias and variance bounds. 
\begin{proposition}\label{prop:recentered_op_bias_and_var_KL}
Suppose Assumption \ref{assump:delta_small} is enforced.
Then
\[
|E[\widehat{\mathbf{H}}(\hat q)(s,a)-\cH(\hat q)(s,a)]|\leq  \frac{2^{6}\|\hat q - q_*\|_\infty}{ \fr{p}_{\wedge}^{3/2}\sqrt{n}}\sqrt{\log(e|\bd{S}|)},
\]
provided $n\geq \fr{p}_{\wedge}\inv$, and
\[
\var(\widehat{\mathbf{H}}(q^*))(s,a)
\leq \frac{2^{12}\|\hat q - q_*\|_\infty^2}{\fr{p}_{\wedge}^3n}\log(e|\bd{S}|)
\]
for all $n\geq 1$. 
\end{proposition}

Similar to the extension from Proposition \ref{prop:variance_empirical_Bellman} to Proposition \ref{prop:high_prob_bound_empirical_Bellman}, we can obtain the following
concentration bound for the recentered operator by extending the variance
bound in Propositon \ref{prop:recentered_op_bias_and_var_KL}. 
\begin{proposition}
\label{prop:recentered_op_high_prob_KL}
Assume Assumption \ref{assump:delta_small}. Then w.p. at least $1-\eta$
\[
\|\cH(\hat q) - \widehat{\mathbf{H}}(\hat q)\|_\infty \leq \frac{8\|\hat q-q^*\|_\infty}{\fr{p}_{\wedge}^{3/2}\sqrt{n}}\sqrt{\log(4|\bd{S}|^2|\bd{A}|/\eta)}
\]
provided that $n\geq 8\fr{p}_{\wedge}^{-2}\log(4|\bd{S}|^2|\bd{A}|/\eta)$
\end{proposition}

We emphasize that all of the propositions are $O(1)$ when $\delta\da 0$. This is due to a more thorough analysis, which allows us to remove the dual variable $\alpha$ (see Lemma \ref{lemma:strong_dual}) in the bounds, as explained  in Lemma \ref{lemma:bias_alpha^3_term} in detail.

\section{Proofs for the Analysis of Algorithms: KL Case}

With the key bias and concentration bounds, we are ready to carry out the
proofs of the worst-case sample complexity bounds for Algorithm \ref{alg:q-learning} and \ref{alg:vr_q-learning}. We will follow the proof
outlined in Section \ref{sec:analysis_of_algo_overview}.

\subsection{The Distributionally Robust Q-learning Algorithm \protect\ref{alg:q-learning}}

\label{subsec:proof_drql}

\subsubsection{Proof of Proposition \protect\ref{prop:q-learning_err_high_prob_bd}}

\begin{proof}
\par Recall that the update rule for Algorithm \ref{alg:q-learning} can be written as
\begin{equation}\label{eqn:ql_update_err}
\begin{aligned}
q_{k+1}-q^* &= (1-\lambda_k)(q_{k}-q^*) + \lambda_{k}\sqbk{\mathbf{T}_{k+1}(q_{k}) - \mathbf{T}_{k+1}(q^*)+\mathbf{T}_{k+1}(q^*)- \cT(q^*) }\\
&= (1-\lambda_k)(q_{k}-q^*) + \lambda_{k}\sqbk{\mathbf{H}_{k+1}(q_{k}) + W_{k+1}}
\end{aligned}
\end{equation}
where we define $W_{k+1}:= \mathbf{T}_{k+1}(q^*)- \cT(q^*)$. Since $ \mathbf{T}_{k+1}(q^*)$ is a i.i.d. sequence of estimators to $\cT(q^*)$, 
\[
\beta := \bias(\mathbf{T}_{k}(q^*))=E[\mathbf{T}_{k}(q^*)]- \cT(q^*) 
\] is independent of $k$. We can write $W_{k+1} = \beta + U_{k+1}$ where  $U_{k+1} := \mathbf{T}_{k+1}(q^*)- \cT(q^*) -\beta$ has zero mean. 
\par Next, we would like to apply Proposition \ref{prop:contraction_err_bd}. Define the auxiliary sequences
\begin{align}
    P_{k+1} &= (1-\lambda_k)P_{k} + \lambda_kW_{k+1}\label{eqn:P_seq_ql}\\
    Q_{k+1} &= (1-\lambda_k)Q_{k} + \lambda_kU_{k+1}\label{eqn:Q_seq_ql}
\end{align}
with $Q_{0} = P_{0} = 0$. Notice that since $\set{U_{k},k\geq 1}$ has mean zero,  $E[Q_{k}] = 0$ for all $k\geq 0$. It is easier to analyze the process $\set{Q_{k},k\geq 0}$ than $\set{P_{k},k\geq 0}$ which correspond to $\set{p_k}$ in Proposition \ref{prop:contraction_err_bd}. 
\par To use $\set{Q_{k},k\geq 0}$, we first show that 
\[
P_{k} = Q_{k} + \beta.
\]
 We prove this by induction. The base case $P_{1} = \lambda_0W_{1} = Q_{1} + \beta$ as $\lambda_0=1$. Next we check the induction step. By the iterative updates \eqref{eqn:P_seq_ql} and \eqref{eqn:Q_seq_ql} and the induction hypothesis,  we have that 
\begin{align}
P_{k+1} &= (1-\lambda_k)P_{k} + \lambda_kW_{k+1} \notag \\ 
&= (1-\lambda_k)(Q_{k} + \beta) + \lambda_k(U_{k+1} +  \beta) \notag\\
&= Q_{k+1} + \beta. \label{Q-learning:difference}
\end{align}
\par By Algorithm \ref{alg:q-learning}, $q_{1} = 0$. We have that by Lemma \ref{lemma:prob_1_bound_empirical_Bellman}, $\|q_1-q^*\|_\infty \leq (1-\gamma)\inv$. Therefore, by Proposition \ref{prop:contraction_err_bd}
\begin{equation}
\begin{aligned}
\|q_{k+1} - q^*\|_\infty &\leq\lambda_k\sqbk{\frac{\|q_{1}-q^*\|_\infty}{\lambda_1} + \gamma \sum_{j=1}^k\|P_{j}\|_\infty }+ \|P_{k+1}\|_\infty\\
 &\leq \lambda_k\sqbk{\frac{1}{\lambda_1(1-\gamma)} + \left(\gamma \sum_{j=1}^k\|Q_{j}\|_\infty\right) +\gamma k\|\beta\|_\infty} + \|Q_{l,k+1}\|_\infty + \|\beta\|_\infty.\\
 &\leq \lambda_k\sqbk{\frac{2}{1-\gamma}+ \gamma \sum_{j=1}^k\|Q_{j}\|_\infty }+ \|Q_{k+1}\|_\infty + \frac{2\|\beta\|_\infty}{1-\gamma}.
\end{aligned}\label{eqn:ql_iter_high_prob_bd}
\end{equation}
w.p.1, where we used $k\lambda_k = 1/(1/k + (1-\gamma))\leq 1/(1-\gamma)$. 
\par Next, we bound the sequence $\set{Q_k,k\geq 1}$.
\begin{lemma}\label{lemma:ql_Q_seq_bd}
The $\set{Q_k,k\geq 1}$ sequence satisfies 
\[
    P(\|Q_{k+1}\|_\infty > t)\leq 2|\bd{S}||\bd{A}|\exp\crbk{-\frac{t^2}{\lambda_k(8\gamma(1-\gamma)\inv  t+ 4\|\sigma^2(q^*)\|_\infty)}}.
\]
where $\sigma^2(q^*)(s,a) = \var(\mathbf{T}_{k}(q^*)(s,a)) $. 
\end{lemma}
The proof of Lemma \ref{lemma:ql_Q_seq_bd} is in Appendix \ref{a_sec:proof:lemmas_for_kl_alg}.
By applying Lemma \ref{lemma:ql_Q_seq_bd}, we have that 
\begin{align*}
    \|Q_{j}\|_\infty&\leq \frac{8\lambda_j}{1-\gamma}\log\crbk{2|\bd{S}||\bd{A}|/\eta} + 2\sqrt{\lambda_j}\|\sigma(q^*)\|_\infty\sqrt{\log\crbk{2|\bd{S}||\bd{A}|/\eta}} \\
    &\leq \left(\frac{8\lambda_j}{1-\gamma}
    + 2\sqrt{\lambda_j}\|\sigma(q^*)\|_\infty\right){\log\crbk{2|\bd{S}||\bd{A}|/\eta}}
\end{align*}
w.p. at least $1-\eta$. 
\par To establish high probability bound using \eqref{eqn:ql_iter_high_prob_bd}, we also need the following properties of the stepsize: 
\begin{lemma}[Proof of Corollary 3,\citet{wainwright2019l_infty}]\label{lemma:sum_lambda_k}
    The following inequalities hold:
    \[
    \sum_{j=1}^{k}\sqrt{\lambda_j}\leq \frac{2}{(1-\gamma)\sqrt{\lambda_k}};\qquad \sum_{j=1}^k\lambda_j\leq \frac{\log(1+(1-\gamma)k)}{1-\gamma}.
    \]
\end{lemma}
We have that by Lemma \ref{lemma:sum_lambda_k} and the union bound, 
\begin{align*}
    &\gamma \lambda_{k_0}\sum_{j=1}^{k_0}\|Q_{j}\|_\infty + \|Q_{k_0+1}\|_\infty \\
    &\leq 8\gamma\crbk{\frac{\lambda_{k_0}\log(1+(1-\gamma)k_0)}{(1-\gamma)^2} + \frac{\|\sigma(q^*)\|_\infty\sqrt{\lambda_{k_0}}}{1-\gamma}}\log\crbk{4|\bd{S}||\bd{A}|k_0/\eta}\\
   &\qquad +\crbk{\frac{8\lambda_{k_0}}{1-\gamma} + 2\|\sigma(q^*)\|_\infty\sqrt{\lambda_{k_0}}}\log\crbk{4|\bd{S}||\bd{A}|k_0/\eta}. \\
   &\leq 16\crbk{\frac{\lambda_{k_0}\log(1+(1-\gamma)k_0)}{(1-\gamma)^2} + \frac{\|\sigma(q^*)\|_\infty\sqrt{\lambda_{k_0}}}{1-\gamma}}\log\crbk{4|\bd{S}||\bd{A}|k_0/\eta}\\
   &\leq 16\crbk{\frac{1}{(1-\gamma)^3k_0} + \frac{20}{\fr{p}_{\wedge}(1-\gamma)^{5/2}\sqrt{n_0k_0}}}\log\crbk{4dk_0/\eta}^2
\end{align*}
w.p. at least $1-\eta$, where we utilize Proposition \ref{prop:variance_empirical_Bellman} to bound $\|\sigma(q^*)\|_\infty$.
\par We use Proposition \ref{prop:bias_empirical_Bellman} to bound $\beta$. Then, from \eqref{eqn:ql_iter_high_prob_bd} we conclude that there exists constant $c$ s.t.
\begin{align*}
\|q_{k_0+1}-q^*\|_\infty &\leq c\crbk{\frac{1}{(1-\gamma)^3k_0} + \frac{1}{\fr{p}_{\wedge}(1-\gamma)^{5/2}\sqrt{n_0k_0}}}\log\crbk{4dk_0/\eta}^2 + c\frac{\rmax + \gamma\spnorm{q}}{(1-\gamma)\fr{p}_{\wedge}^3 n_0}\log(e|\bd{S}|\vee |\bd{R}|)\\  
&\leq c\crbk{\frac{1}{(1-\gamma)^3k_0} + \frac{1}{\fr{p}_{\wedge}(1-\gamma)^{5/2}\sqrt{n_0k_0}}+ \frac{1}{\fr{p}_{\wedge}^3(1-\gamma)^2 n_0}}\log\crbk{4dk_0/\eta}^2
\end{align*}
where $c$ can change from line to line. 
\par Finally, note that for $C_1\geq 1,C_2\geq e$, $\log(C_1C_2)= \log(C_1)+\log(C_2)\leq C_1\log(C_2)$. So, $\log\crbk{4dk_0/\eta}^2\leq \frac{16}{9}\log\crbk{3dk_0/\eta}^2$. This completes the proof. 
\end{proof}

\subsection{The Variance-Reduced Distributionally Robust Q-learning Algorithm \protect\ref{alg:vr_q-learning}}
\label{a_sec:proof_drvrql}

\subsubsection{Proof of Proposition \protect\ref{prop:vr_algo_err_high_prob_bd}}\label{a_sec:proof:prop:one_vr_iter_high_prob_bd}

\begin{proof}
Recall that  $\cF_{l}$ be the $\sigma$-field generated by the random samples used until the end of epoch $l$ and 
\begin{equation*}
E_{l}[\cd] = E[\cd|\cF_l],P_{l}[\cd] = P[\cd|\cF_l],\text{ and } \var_{l}(\cd)=\var(\cd|\cF_l) . 
\end{equation*}
In the following proof, the probabilities are w.r.t. $P_{l-1}(\cd)$. Recall that
\[
\mathbf{H}_{l,k} = \mathbf{T}_{l,k}(q)-\mathbf{T}_{l,k}(q^*)\text{ and } \widetilde{\mathbf{H}}_{l} = \widetilde{\mathbf{T}}_{l,k}(q)-\widetilde{\mathbf{T}}_{l,k}(q^*).
\]
\par From the variance-reduced DR-RL (Algorithm \ref{alg:vr_q-learning}) update rule, we have at epoch $l$,
\begin{equation}\label{eqn:vr_ql_update_err}
\begin{aligned}
q_{l,k+1}-q^* &= (1-\lambda_k)(q_{l,k}-q^*) + \lambda_{k}\sqbk{\mathbf{T}_{l,k+1}(q_{l,k})- \mathbf{T}_{l,k+1}(\hat q_{l-1})+ \widetilde{\mathbf{T}}_{l}(\hat q_{l-1}) - \cT(q^*) }\\
&= (1-\lambda_k)(q_{l,k}-q^*) + \lambda_{k}\sqbk{\mathbf{H}_{l,k+1}(q_{l,k}) + \mathbf{T}_{l,k+1}(q^*)- \mathbf{T}_{l,k+1}(\hat q_{l-1})+ \widetilde{\mathbf{T}}_{l}(\hat q_{l-1}) - \cT(q^*)}
\\
&= (1-\lambda_k)(q_{l,k}-q^*) + \lambda_{k}\sqbk{\mathbf{H}_{l,k+1}(q_{l,k}) + W_{l,k+1}}
\end{aligned}
\end{equation}
where we define $W_{l,k+1} =  \mathbf{T}_{l,k+1}(q^*)- \mathbf{T}_{l,k+1}(\hat q_{l-1})+ \widetilde{\mathbf{T}}_{l}(\hat q_{l-1}) - \cT(q^*)$. Notice that only the first two terms is dependent on $k$. We can write
\begin{equation}
\label{eqn:vr_ql_noise_decomposition}
\begin{aligned}
W_{l,k+1} &=  \mathbf{T}_{l,k+1}(q^*)- \mathbf{T}_{l,k+1}(\hat q_{l-1})+ \widetilde{\mathbf{T}}_{l}(\hat q_{l-1}) - \cT(q^*) \\
&=  - \mathbf{H}_{l,k+1}(\hat q_{l-1})+ \widetilde{\mathbf{H}}_{l}(\hat q_{l-1})  + \widetilde{\mathbf{T}}_{l}(q^*)- \cT(q^*)\\
&=  - [\mathbf{H}_{l,k+1}(\hat q_{l-1})-E_{l-1}[\mathbf{H}_{l,k+1}(\hat q_{l-1})]]+ \widetilde{\mathbf{H}}_{l}(\hat q_{l-1})  + \widetilde{\mathbf{T}}_{l}(q^*)- \cT(q^*)-E_{l-1}[\mathbf{H}_{l,k+1}(\hat q_{l-1})]\\
&=  - U_{l,k+1} + D_l
\end{aligned}
\end{equation}
where 
\begin{align}
    U_{l,k+1} &:
= \mathbf{H}_{l,k+1}(\hat q_{l-1})-E_{l-1}[\mathbf{H}_{l,k+1}(\hat q_{l-1})]\label{eqn:U_l_k_def}\\
D_l &:= 
\widetilde{\mathbf{H}}_{l}(\hat q_{l-1})  + \widetilde{\mathbf{T}}_{l}(q^*)- \cT(q^*)-E_{l-1}[\mathbf{H}_{l,k+1}(\hat q_{l-1})]\label{eqn:D_l_def}.
\end{align} 
Note that $E_{l-1}[\mathbf{H}_{l,k+1}(\hat q_{l-1})]$ is constant in $k$. 
\par We will apply Proposition \ref{prop:contraction_err_bd}. Define the auxiliary sequences
\begin{align}
    P_{l,k+1} &= (1-\lambda_k)P_{l,k} + \lambda_kW_{l,k+1}\label{eqn:P_seq}\\
    Q_{l,k+1} &= (1-\lambda_k)Q_{l,k} + \lambda_k(-U_{l,k+1})\label{eqn:Q_seq}
\end{align}
with $Q_{l,0} = P_{l,0} = 0$. Note that $U_{l,k+1}$ under $E_{l-1}$ are i.i.d. and has mean 0. So $E_{l-1}[Q_{l,k}] = 0$ for any $k\geq 0$. It is easier to analyze the process $\set{Q_{l,k},k\geq 0}$ than $\set{P_{l,k},k\geq 0}$ which correspond to $\set{p_k}$ in Proposition \ref{prop:contraction_err_bd}. 
\par As in the DR Q-learning case (Equation \eqref{Q-learning:difference}), the same induction argument implies that
$P_{l,k} = Q_{l,k} + D_l$. 
\par By the algorithm, $q_{l,1} = \hat q_{l-1}$, we have that $\|q_{l,1}-q^*\|_\infty \leq b$. Therefore, by Proposition \ref{prop:contraction_err_bd}
\begin{equation}
\begin{aligned}
\|q_{l,k+1} - q^*\|_\infty &\leq\lambda_k\sqbk{\frac{\|q_{l,1}-q^*\|_\infty}{\lambda_1} + \gamma \sum_{j=1}^k\|P_{l,j}\|_\infty }+ \|P_{l,k+1}\|_\infty\\
 &\leq \lambda_k\sqbk{\frac{b}{\lambda_1} + \left(\gamma \sum_{j=1}^k\|Q_{l,j}\|_\infty +\gamma k\|D_l\|_\infty\right)}+ \|Q_{l,k+1}\|_\infty + \|D_l\|_\infty.\\
 &\leq \lambda_k\sqbk{2b + \gamma \sum_{j=1}^k\|Q_{l,j}\|_\infty }+ \|Q_{l,k+1}\|_\infty + \frac{2\|D_l\|_\infty}{1-\gamma}.
\end{aligned}\label{eqn:q_inner_prob_1_bd}
\end{equation}
w.p.1, where we used $k\lambda_k = 1/(1/k + (1-\gamma))\leq 1/(1-\gamma)$. 
\par Next, we prove bounds for $\set{\|Q_{l,k}\|_\infty,k\geq 0}$ and $\|D_l\|_\infty$. 
\begin{lemma}\label{lemma:vrql_Q_seq_bd}
Under measure $P_{l-1}(\cd)$, 
\[
P_{l-1}(\|Q_{l,j}\|_\infty >t)\leq 2|\bd{S}||\bd{A}|\exp\crbk{-\frac{t^2}{4\lambda_j  (\gamma \|\zeta_{l-1}\|_\infty t +\|\sigma_{l-1}^2\|_\infty)}}
\]
where $\zeta_{l-1} = \hat q_{l-1}-q^*$ and  $\sigma_{l-1}^2(s,a) = \var_{l-1}(\mathbf{H}_{l,k}(\hat q_{l-1})(s,a)) $. 
\end{lemma}
The proof of this Lemma is deferred to Appendix \ref{a_sec:proof:lemmas_for_kl_alg}. Apply Lemma \ref{lemma:vrql_Q_seq_bd}, we have that 
\[
    \|Q_{l,j}\|_\infty\leq 4\lambda_j\|\zeta_{l-1}\|_\infty\log\crbk{2|\bd{S}||\bd{A}|/\eta} + 2\sqrt{\lambda_j}\|\sigma_{l-1}\|_\infty\sqrt{\log\crbk{2|\bd{S}||\bd{A}|/\eta}}.
\]
w.p. at least $1-\eta$. 

Recall the definition of $\sigma^2_{l-1}$ and Proposition \ref{prop:recentered_op_bias_and_var_KL}. We have that by Lemma \ref{lemma:sum_lambda_k} and the union bound, 
\begin{equation}
\begin{aligned}
   &\gamma \lambda_{k_\mathrm{vr}}\sum_{j=1}^{k_\mathrm{vr}}\|Q_{l,j}\|_\infty + \|Q_{l,k_\mathrm{vr}+1}\|_\infty \\
   &\leq 4\gamma\crbk{\frac{\lambda_{k_\mathrm{vr}}\log(1+(1-\gamma)k_\mathrm{vr})\|\zeta_{l-1}\|_\infty}{1-\gamma} + \frac{\|\sigma_{l-1}\|_\infty\sqrt{\lambda_{k_\mathrm{vr}}}}{1-\gamma}}\log\crbk{4|\bd{S}||\bd{A}|k_\mathrm{vr}/\eta}\\
   &\qquad +\crbk{4\lambda_{k_\mathrm{vr}}\|\zeta_{l-1}\|_\infty + 2\|\sigma_{l-1}\|_\infty\sqrt{\lambda_{k_\mathrm{vr}}}}\log\crbk{4|\bd{S}||\bd{A}|k_\mathrm{vr}/\eta}. \\
   &\leq 8\crbk{\frac{\lambda_{k_\mathrm{vr}}\log(e+(1-\gamma)k_\mathrm{vr})\|\zeta_{l-1}\|_\infty}{1-\gamma} + \frac{\|\sigma_{l-1}\|_\infty\sqrt{\lambda_{k_\mathrm{vr}}}}{1-\gamma}}\log\crbk{4|\bd{S}||\bd{A}|k_\mathrm{vr}/\eta}\\
   &\leq 8\crbk{\frac{b}{(1-\gamma)^2k_\mathrm{vr}} + \frac{2^6b}{\fr{p}_{\wedge}^{3/2}(1-\gamma)^{3/2}\sqrt{n_\mathrm{vr}k_\mathrm{vr}}}}\log\crbk{4|\bd{S}||\bd{A}|k_\mathrm{vr}/\eta}^2\\
\end{aligned}\label{eqn:Q_sum_high_prob_bd}
\end{equation}
w.p. at least $1-\eta$.
\par For $D_l$, recall the definition in \eqref{eqn:D_l_def}. We add and subtract $\cH(\hat q_{l-1}) $ and write: 
\begin{equation}\label{eqn:def_D_l}
\begin{aligned}
D_l &= \crbk{\widetilde{\mathbf{H}}_{l}(\hat q_{l-1}) -\cH(\hat q_{l-1})} + \crbk{\widetilde{\mathbf{T}}_{l}(q^*)- \cT(q^*) } + \crbk{\cH(\hat q_{l-1}) -E_{l-1}[\mathbf{H}_{l,k+1}(\hat q_{l-1})]}\\
&= \crbk{\widetilde{\mathbf{H}}_{l}(\hat q_{l-1}) -\cH(\hat q_{l-1})} + \crbk{\widetilde{\mathbf{T}}_{l}(q^*)- \cT(q^*) } + E_{l-1}\left[\cH(\hat q_{l-1})-\mathbf{H}_{l,k+1}(\hat q_{l-1})\right].
\end{aligned}
\end{equation}
 Recall Propositions \ref{prop:high_prob_bound_empirical_Bellman}, \ref{prop:recentered_op_bias_and_var_KL}, and \ref{prop:recentered_op_high_prob_KL}, we have that by union bound,
\begin{align}
\|D_l\|_\infty &\leq c\frac{\rmax+\spnorm{q^*} + \|\hat q_{l-1}-q^*\|_\infty}{\fr{p}_{\wedge}^{3/2} \sqrt{m_l}}\sqrt{\log\crbk{12d/\eta}}+E_{l-1}\crbk{\cH(\hat q_{l-1})-\mathbf{H}_{l,k+1}(\hat q_{l-1})} \notag \\
&\leq c\frac{\rmax+\spnorm{q^*} + \|\hat q_{l-1}-q^*\|_\infty}{\fr{p}_{\wedge}^{3/2} \sqrt{m_l}}\sqrt{\log\crbk{12d/\eta}}+  c\frac{\|\hat q_{l-1}-q^*\|_\infty}{\fr{p}_{\wedge}^{3/2}\sqrt{n_\mathrm{vr}}}\sqrt
{\log(e|\bd{S}|)}
\label{eqn:D_l_high_prob_bd}
\end{align}
w.p. at least $1-\eta$, provided that $c$ is a large enough constant $m_l\geq 8\fr{p}_{\wedge}^{-2}\log(24d/\eta)$, and $n_\mrm{vr}\geq \fr{p}_\wedge\inv $. 
\par Finally, recall that $\hat q_l = \hat {q}_{l,k_\mathrm{vr}+1}$ and $\rmax = 1$, combine \eqref{eqn:q_inner_prob_1_bd}, \eqref{eqn:Q_sum_high_prob_bd}, and \eqref{eqn:D_l_high_prob_bd} we conclude that there exists absolute constant $c$ s.t.
\begin{align*}
\|\hat q_{l} - q^*\|_\infty&\leq c\crbk{\frac{b}{(1-\gamma)^2k_\mathrm{vr}} +\frac{b}{\fr{p}_{\wedge}^{3/2}(1-\gamma)^{3/2}\sqrt{n_\mathrm{vr}k_\mathrm{vr}}} }\log\crbk{8dk_\mathrm{vr}/\eta}^2\\
&\qquad+  c\frac{\rmax+\spnorm{q^*} +b}{\fr{p}_{\wedge}^{3/2} (1-\gamma)\sqrt{m_l}}\log(24d/\eta)+c\frac{b}{\fr{p}_{\wedge}^{3/2}(1-\gamma)\sqrt{n_\mrm{vr}}}\sqrt{\log(e|\bd{S}|)}\\
&\leq c\crbk{\frac{b}{(1-\gamma)^2k_\mathrm{vr}} +\frac{b}{\fr{p}_{\wedge}^{3/2}(1-\gamma)^{3/2}\sqrt{n_\mathrm{vr}k_\mathrm{vr}}} + 
\frac{b}{\fr{p}_{\wedge}^{3/2}(1-\gamma)\sqrt{n_\mathrm{vr}}} }\log\crbk{8dk_\mathrm{vr}/\eta}^2 \\
&\quad +  c\frac{1}{\fr{p}_{\wedge}^{3/2} (1-\gamma)^2\sqrt{m_l}}\sqrt{\log(24d/\eta)}
\end{align*}
w.p. at least $1-\eta$, where we used $\spnorm{q^*}\leq 2\|q^*\|_\infty\leq 2/(1-\gamma)$ and $b\leq 1/(1-\gamma)$, $c$ can change from line to line.  
\par Finally, note that for $C_1\geq 1,C_2\geq e$, $\log(C_1C_2)= \log(C_1)+\log(C_2)\leq C_1\log(C_2)$. This completes the proof of Proposition \ref{prop:one_vr_iter_high_prob_bd}. 
\end{proof}

\subsubsection{Proof of Proposition \protect\ref	{prop:vr_algo_err_high_prob_bd}}

\begin{proof}
	By the definition of conditional probability
	\begin{align*}
		&P\crbk{\bigcap_{l=0}^{l_\mathrm{vr}} \set{\|\hat q_{l} - q^*\|_\infty\leq 2^{-l}(1-\gamma)^{-1}}}\\
		&= \prod_{l=0}^{l_\mathrm{vr}} P\crbk{\|\hat q_{l} - q^*\|_\infty\leq 2^{-l}(1-\gamma)^{-1}\Bigg|\bigcap_{n=0}^{l-1} \set{\|\hat q_{n} - q^*\|_\infty\leq 2^{-n}(1-\gamma)^{-1}}}\\
		&= \prod_{l=1}^{l_\mathrm{vr}} P\crbk{\|\hat q_{l} - q^*\|_\infty\leq 2^{-l}(1-\gamma)^{-1}\Bigg|\bigcap_{n=1}^{l-1} \set{\|\hat q_{n} - q^*\|_\infty\leq 2^{-n}(1-\gamma)^{-1}}}
	\end{align*}
	where we note that $\hat q_0 = 0$ and Lemma \ref{lemma:prob_1_bound_empirical_Bellman} implies that $\|\hat q_0-q^*\|\leq (1-\gamma)^{-1}$ w.p.1, so the conditioned intersection and product can start from $s=1$. Let 
	\[
	A_{l-1} = \bigcap_{s=1}^{l-1} \set{\|\hat q_{s} - q^*\|_\infty\leq 2^{-s}(1-\gamma)^{-1}}.
	\]
	
	We analyze the probability for $l\geq 1$
	\begin{align*}
		&P\crbkcond{\| \hat q_{l} - q^*\|_\infty\leq 2^{-l}(1-\gamma)^{-1}}{A_{l-1}} \\
		= & \frac{1}{P(A_{l-1})}E\left[\1\set{\|\hat q_{l} - q^*\|_\infty\leq 2^{-l}(1-\gamma)^{-1}}\1_{A_{l-1}}\right]\\
		=& \frac{1}{P(A_{l-1})}E\sqbk{\1\set{\|\hat q_{l-1} - q^*\|_\infty\leq 2^{-(l-1)}(1-\gamma)^{-1}} E\sqbkcond{\1\set{\|\hat q_{l} - q_*\|_\infty\leq 2^{-l}(1-\gamma)^{-1} }}{ \cF_{l-1}}\1_{A_{l-1}}}\\
		=& \frac{1}{P(A_{l-1})}E\sqbk{\1\set{\|\hat q_{l-1} - q^*\|_\infty\leq 2^{-(l-1)}(1-\gamma)^{-1}} P_{l-1}\crbk{\1\set{\|\hat q_{l} - q_*\|_\infty\leq 2^{-l}(1-\gamma)^{-1} }}\1_{A_{l-1}}}
		,
	\end{align*}
By Proposition \ref{prop:one_vr_iter_high_prob_bd}, we recall conditioned on $\|\hat q_{l-1} - q^*\|_\infty\leq 2^{-(l-1)}(1-\gamma)^{-1} =: b$
 \begin{align*}
\|\hat q_{l} - q^*\|_\infty&\leq c\crbk{\frac{b}{(1-\gamma)^2k_\mathrm{vr}} +\frac{b}{\fr{p}_{\wedge}^{3/2}(1-\gamma)^{3/2}\sqrt{n_\mathrm{vr}k_\mathrm{vr}}} + 
\frac{b}{\fr{p}_{\wedge}^{3/2}(1-\gamma)\sqrt{n_\mathrm{vr}}} }\log\crbk{3dk_\mathrm{vr}/\eta}^2 \\
&\quad +  c\frac{1}{\fr{p}_{\wedge}^{3/2} (1-\gamma)^2\sqrt{m_l}}\sqrt{\log(3d/\eta)}
\end{align*}
w.p. at least $1-\eta$. 
\par Therefore, by the parameter choice \eqref{eqn:param_choice_for_vrql}, we have that for sufficiently large $c_{\mrm{vr}}$ and for events $\omega\in \set{\|\hat q_{l-1} - q^*\|_\infty\leq 2^{-(l-1)}(1-\gamma)^{-1}}$,
\begin{equation}\label{eqn:one_vr_iter_geom_converge_kl}
	P_{l-1}\crbk{\1\set{\|\hat q_{l} - q_*\|_\infty\leq 2^{-l}(1-\gamma)^{-1} }}(\omega) \geq 1-\frac{\eta}{l_\mathrm{vr}};
\end{equation}
i.e.
 \[
 \1\set{\|\hat q_{l-1} - q^*\|_\infty\leq 2^{-(l-1)}(1-\gamma)^{-1}} P_{l-1}\crbk{\1\set{\|\hat q_{l} - q_*\|_\infty\leq 2^{-l}(1-\gamma)^{-1} }} \geq 1-\frac{\eta}{l_\mathrm{vr}}.
 \]
	Therefore, we have
 \[
 P\crbkcond{\| \hat q_{l} - q^*\|_\infty\leq 2^{-l}(1-\gamma)^{-1}}{A_{l-1}}\geq 1-\frac{\eta}{l_\mathrm{vr}},
 \]
which further gives us	\begin{equation}\label{eqn:pathwise_err_prob_bound}
		P\crbk{\bigcap_{l=0}^{l_\mathrm{vr}} \set{\|\hat q_{l} - q^*\|_\infty\leq 2^{-l}(1-\gamma)^{-1}}}\geq\crbk{1-\frac{\eta}{l_\mathrm{vr}}}^{l_\mathrm{vr}}.
	\end{equation}
	To finish the proof, we consider the function
	\[
	e(\eta):=\crbk{1-\frac{\eta}{l}}^{l}.
	\]
	Clearly, $e(\eta)$ is $C^2$ with derivatives
	\[
	e'(\eta) = -\crbk{1-\frac{\eta}{l}}^{l-1},\qquad e''(\eta) = \frac{l-1}{l}\crbk{1-\frac{\eta}{l}}^{l-2}. 
	\]
	Note that $e''\geq 0$ if $l\geq 1$. So $e'(\eta)$ is non-decreasing. Hence for all $\eta\geq 0$, $e'(\eta)\geq e'(0)$. 
	This implies that
	\begin{align*}
		e(\eta) &= e(0) + \int_0^{\eta} e'(t)dt\\
		&\geq 1-\eta.
	\end{align*}
	Assumption $\epsilon < (1-\gamma)^{-1}$ implies that  $l_\mathrm{vr}\geq 1$. Therefore, we plug in this to \eqref{eqn:pathwise_err_prob_bound} and conclude that
	\[
	P\crbk{ \|\hat q_{l_\mathrm{vr}} - q^*\|_\infty\leq 2^{-l_\mathrm{vr}}(1-\gamma)^{-1}}\geq P\crbk{\bigcap_{l=0}^{l_\mathrm{vr}} \set{\|\hat q_{l} - q^*\|_\infty\leq 2^{-l}(1-\gamma)^{-1}}}\geq 1-\eta.
	\]
\end{proof}

\subsection{Proof of Lemma \protect\ref{lemma:ql_Q_seq_bd} and \protect\ref{lemma:vrql_Q_seq_bd}}
\label{a_sec:proof:lemmas_for_kl_alg}
To prove these two lemma, we introduce the following result:
\begin{lemma}[\citet{wainwright2019l_infty}, Lemma 2]\label{lemma:AR_seq_bd}
Let $\set{Y_k\in\R,k\geq 1}$ be a sequence of i.i.d. zero mean $\zeta$-bounded r.v.s with variance $\sigma^2$. Define $\set{X_k,k\geq 0}$ by the recursion $X_0=0$
\[
X_{k+1} = (1-\lambda_k)X_k + \lambda_kY_{k+1},
\]
where $\lambda_k = 1/(1+(1-\gamma)k)$. Then
\[
E\left[\exp(tX_{k+1})\right]\leq \exp\crbk{\frac{t^2\sigma^2\lambda_k}{1-\zeta\lambda_k |t|}}
\]
for all $|t|<1/(\zeta\lambda_k)$. 
\end{lemma}
We first prove Lemma \ref{lemma:vrql_Q_seq_bd}.

\begin{proof} We use the same steps. Recall \eqref{eqn:Q_seq}, where $U_{l,k}$ is an i.i.d. sequence under $E_{l-1}$ given by \eqref{eqn:U_l_k_def}. By Lemma \ref{lemma:prob_1_bound_empirical_Bellman}
\begin{align*}
\|U_{l,k}\|_\infty &\leq \|\mathbf{T}_{l,k}(\hat q_{l-1}) - \mathbf{T}_{l,k}( q^*) \|_\infty + \|E_{l-1}[\mathbf{T}_{l,k}(\hat q_{l-1}) - \mathbf{T}_{l,k}( q^*)]\|_\infty\\
&\leq 2\gamma\|\hat q_{l-1} - q^*\|_\infty\\
&= 2\gamma\|\zeta_{l-1}\|_\infty.
\end{align*}
Notice that by construction, $\mathbf{T}_{l,k}(q^*)(s,a)$ are independent across $s\in\bd{S}$, $a\in\bd{A}$. Therefore, by Lemma \ref{lemma:AR_seq_bd}, 
\begin{align*}
    E_{l-1}\exp(\lambda\|Q_{l,k+1}\|_\infty)&= E_{l-1}\sup_{(s,a)\in \bd{S}\times\bd{A}}\max\set{\exp(\lambda Q_{l,k+1}(s,a)),\exp(- \lambda Q_{l,k+1}(s,a))}\\
    &\leq \sum_{(s,a)\in \bd{S}\times\bd{A}}E_{l-1}\exp(\lambda Q_{l,k+1}(s,a))+E\exp(-\lambda Q_{l,k+1}(s,a))\\
    &\leq 2|\bd{S}||\bd{A}|\exp\crbk{\frac{\lambda^2\|\sigma_{l-1}^2\|_\infty\lambda_k}{1-2\gamma\|\zeta_{l-1}\|_\infty\lambda_k|\lambda|}},
\end{align*}
for any $\lambda< 1/(2\gamma\|\zeta_{l-1}\|_\infty\lambda_k)$. Therefore, by the Chernoff bound
\begin{align*}
    P_{l-1}(\|Q_{l,k+1}\|_\infty > t)&\leq 2|\bd{S}||\bd{A}|\exp\crbk{\frac{\lambda^2\|\sigma_{l-1}^2\|_\infty\lambda_k}{1-2\gamma\|\zeta_{l-1}\|_\infty\lambda_k|\lambda|}}e^{-\lambda t},
\end{align*}
for any $\lambda \in (0,1/(2\gamma\|\zeta_{l-1}\|_\infty\lambda_k))$. Choose 
\[
\lambda = \frac{t}{2\gamma \|\zeta_{l-1} \|_\infty \lambda_k t + 2\|\sigma^2_{l-1}\|_\infty\lambda_k}, 
\]
we conclude that
\[
P_{l-1}(\|Q_{l,k+1}\|_\infty > t)\leq 2|\bd{S}||\bd{A}|\exp\crbk{-\frac{t^2}{4\lambda_k(\gamma\|\zeta_{l-1}\|_\infty t+ \|\sigma^2_{l-1}\|_\infty)}}.
\]
\end{proof}
Next, we prove Lemma \ref{lemma:ql_Q_seq_bd}. Notice that we only need to
modify the bounds on $\zeta$ and $\sigma^2$. 
\begin{proof}
    Recall that $\set{Q_{l,k},k\geq 0}$ is given by recursive relation \eqref{eqn:Q_seq_ql}, where $U_{k}$ has mean 0. By Lemma \ref{lemma:prob_1_bound_empirical_Bellman}
    \begin{align*}
    \|U_{k}\|_\infty &\leq  2\|\mathbf{T}_{k+1}(q^*)\|_\infty\\
     &\leq2\rmax + 2\gamma\|q^*\|_\infty\\
     &\leq \frac{4\gamma}{1-\gamma}. 
    \end{align*}
    and $\var(\mathbf{T}_{k+1}(q^*)(s,a)) = \sigma^2(q^*)(s,a)$. 
    Therefore, using the same arguments, we conclude that
    \[
    P(\|Q_{k+1}\|_\infty > t)\leq 2|\bd{S}||\bd{A}|\exp\crbk{-\frac{t^2}{\lambda_k(8\gamma(1-\gamma)\inv  t+ 4\|\sigma^2(q^*)\|_\infty)}}.
    \]
\end{proof}

\section{Proofs of Properties of the Empirical Bellman Operator: KL Case}\label{a_sec:proof_empirical_bellam_kl}
\label{proof:propostions}

\subsection{Glossary of Notations and Basic Properties}
Before we present our proofs, we first define some technical notations. For
finite discrete measurable space $(Y,2^Y)$, fixed $u\in m2^Y$, and signed
measure $\nu\in \cM_{\pm}(Y,2^Y)$, let 
$$\nu[u] = \sum_{y\in Y} \nu(y)u(y)$$ denotes the integral. 

\par For generic probability measure $\mu$ on $(Y,2^Y)$ and random variable $u:Y\ra \R$, let $w = w(\alpha) = e^{-u/\alpha}$; define the \textit{$\mathrm{KL}$ dual
functional} under the reference measure $\mu$ 
\begin{equation}  \label{eqn:KL_dual_functional_def}
f(\mu,u,\alpha) := -\alpha\log \mu[e^{-u/\alpha}]-\alpha\delta.
\end{equation}
We clarify that $f(\mu,u,0) = \lim_{\alpha\da 0}f(\mu,u,\alpha) = \essinf_{\mu}u$. We present two basic properties of the dual functional $f$ for which the proofs are deferred to Appendix \ref{a_sec:proof:lemma_kl}. 

\begin{lemma}\label{lemma:sup_f_bound}
For any $ \nu\ll\mu$, the dual functional is bounded 
\[
-\|u\|_{L^\infty(\mu)} \leq \sup_{\alpha\geq 0} f(\nu, u,\alpha) \leq \|u\|_{L^\infty(\mu)}
\]
\end{lemma}
\begin{lemma}\label{lemma:sup_f_diff_bound} The following bound holds w.p.1.: 
\[
\abs{\sup_{\alpha\geq 0} f(\mu, u,\alpha) - \sup_{\alpha\geq 0} f(\mu_n, u,\alpha) }\leq 2\spnorm{u},
\]
where $\spnorm{u} = \max_{s\in S} u(s) - \min_{s\in S} u(s)$. 
\end{lemma}

\par Let $\mu_n$ be the empirical measure form by $n$ i.i.d. samples drawn from $\mu$. In the following development, we need to consider the perturbation
analysis on the line of center measures $\set{t\mu + (1-t)\mu_{n}:t\in[0,1]}$. So, it is convenient to define for $\mu_{s,a} = p_{s,a},\nu_{s,a}$ 
\begin{equation}  \label{eqn:g,m,mu_def}
\begin{aligned} \mu_{s,a,n}(t) &= t\mu_{s,a} + (1-t)\mu_{s,a,n}\\ m_{s,a,n}
&= \mu_{s,a}-\mu_{s,a,n}\\ g_{s,a,n}(t,\alpha) &=
f(\mu_{s,a,n}(t),u,\alpha). \end{aligned}
\end{equation}
Note that we will not explicitly indicate the dependence of $u$ for the
function $g$, because it will always be the identity function when $\mu =
\nu $ and the value function when $\mu = p$. We will also drop the
dependence on $(s,a)$ when clear.
\par Our analysis involves many derivative computations. We use three type of derivative notations, two of which is explained here and the Radon-Nikodym derivative is introduced in the following paragraph. For a smooth function of multiple arguments $g(t,\alpha_{s,t})$ where $\alpha_{s,t}$ could be dependent on parameters $s,t$, denote the partial derivatives by $\del_t,\del_\alpha$; i.e.
$$\del_t g(t,\alpha_{s,t}):= \lim_{\epsilon\ra 0} \frac{g(t+\epsilon,\alpha_{s,t})}{\epsilon},\quad \del_\alpha g(t,\alpha_{s,t}):= \lim_{\epsilon\ra 0} \frac{g(t,\alpha_{s,t}+\epsilon)}{\epsilon}. $$
On the other hand, when $\alpha_{s,t}$ is also smooth in $t$, denote the total derivative w.r.t. $t$ by $d_t$; i.e.
\[
d_tg(t,\alpha_{s,t}):= \lim_{\epsilon\ra 0}\frac{g(t+\epsilon,\alpha_{s,t+\epsilon})}{\epsilon} = \del_tg(t,\alpha_{s,t}) + \del_\alpha g(t,\alpha_{s,t})\del_t\alpha_{s,t}
\]
\par The intuition behind our ability to remove the \(1/\delta\) dependence stems from the mutual absolute continuity (also known as equivalence) between the empirical worst-case transition kernel and reward distribution and the true ones. This holds if \(\delta\) is sufficiently small and the empirical centers of the uncertainty sets are equivalent to the true centers.
\par As a result, our techniques rely on the absolute continuity between the empirical measure $\mu_n$ and $\mu$. We say that $\mu$ is absolute continuous w.r.t. another measure $\nu$, denoted by $\nu\gg\mu $, if for $A\in2^Y$, $\nu(A) = 0$ implies that $\mu(A) = 0$. We say that $\mu$ is equivalent to $\nu$, denoted by $\mu\sim \nu$, if $\nu\gg\mu$ and $\nu\ll\mu$. Note that the empirical measure $\mu_n$ always satisfies $\mu_n\ll\mu$ w.p.1. For absolutely continuous measures $\nu\ll\mu$, the Radon-Nikodym derivative is well defined: 
$$\frac{d\nu}{d\mu}(y):= \frac{\nu(s)}{\mu(s)}\1\set{\mu(s) \neq 0}.$$
The proof strategy we will implement is to consider separately the ``good events" on which $\mu_n$ and $\mu$ are close (so that we have $\mu_n\sim \mu$) and the ``bad events" where the
empirical measure is not close to the reference model. This motivates us to
define for $p > 0$ 
\begin{equation}  \label{eqn:Omega_n,p_mu}
\Omega_{n,p}(\mu) = \set{\omega: \sup_{y} |\mu_{n }(\omega)(y)-\mu(y)|\leq p}.
\end{equation}
Then, in the DR-RL setting, define 
\begin{align*}
\Omega_{n,p} &= \bigcap_{s,a}\Omega_{n,p}(p_{s,a})\cap
\bigcap_{s,a}\Omega_{n,p}(\nu_{s,a}) \\
&= \set{\omega: \sup_{s,a}\sup_{s'} |p_{s,a,n }(\omega)(s')-p_{s,a}(s')|\leq
p, \sup_{s,a}\sup_{r} |\nu_{s,a,n }(\omega)(r)-\nu_{s,a}(r)|\leq p}.
\end{align*}
\par We frequently make use of the minimum support probability of certain measures such as $\mu,\mu_{s,a}$. This is denoted by $\mu_\wedge := \min\set{\mu(s):\mu(s)> 0}$, $\mu_{s,a,\wedge} := \min\set{\mu(s):\mu(s)> 0}$. 

It is easy to see that the following lemma holds: 
\begin{lemma}\label{lemma:equiv_of_meas_cond}
Suppose $p<  \mu_\wedge $, then on $\Omega_{n,p}(\mu)$, $\mu \sim \mu_n$ and $\inf_{y:\mu(y) > 0}\mu_n(y) > \mu_\wedge -p$.
\end{lemma}
Moreover, the empirical measures are satisfies the following concentrations: 
\begin{lemma}\label{lemma:prob_Omega_n,p_mu^c_bound} Let $\mu$ be any probability measure on finite measure space $(Y,2^Y)$. Then, for any $k = 1,2,3,,\ds$ 
\[
P(\Omega_{n,p}(\mu)^c)\leq\frac{1}{p^{2k}n^k} \log(e^{2k-1}|Y|)^k .
\]
In particular, if we choose $k=1$,
\[
P(\Omega_{n,p}(\mu)^c)\leq \frac{1}{p^2n}\log(e|Y|). 
\]
\end{lemma}
This lemma follows from the subgaussian property of empirical measures on finite measure space; i.e. Lemma \ref{lemma:sub_g_k_max} holds.  
\par For absolutely continuous empirical measures, we also have the following lemma, again as a consequence of subgaussianity and hence Lemma \ref{lemma:sub_g_k_max}. 
\begin{lemma}\label{lemma:E_RND_bounds} 
Let $\xi_n$ be another random measure on $(Y,2^Y)$. Let $(\Omega,\cF,P)$ be the probability space that supports $\xi_n,\mu_n$. Suppose that $\mu_n\ll \xi_n$, $\mu\ll \xi_n$, and $\xi_n(y) > p$ for all $y$ s.t. $\xi_n(y)\neq 0$. Then, for all $A\in\cF$, the following bounds hold:
\[
E\Linfnorm{\frac{dm_n}{d\xi_n}}{\xi_n}\1_A\leq \frac{1}{p\sqrt{n}}\sqrt{\log(e|Y|)}
\]
and
\[
E\Linfnorm{\frac{dm_n}{d\xi_n}}{\xi_n}^2\1_A\leq \frac{1}{p^2n}\log(e|Y|).
\]
\end{lemma}
The proofs of these results are deferred to Appendix \ref{a_sec:proof:concentration_of_emp_meas}.

\subsection{Proof of Proposition \protect\ref{prop:variance_empirical_Bellman}}

\begin{proof}
By definition, we have \begin{equation}\label{eqn:Bellman_difference_dual_bound}
\begin{aligned}
\abs{\widehat{\mathbf{T}}(q)(s,a) - \cT(q)(s,a)}&\leq \sup_{\beta\geq 0} \abs{f(\nu_{s,a,n},id,\beta)-f(\nu_{s,a},id,\beta)}\\
&\quad +\gamma\sup_{\alpha\geq 0} \abs{f(p_{s,a,n},v(q),\alpha)-f(p_{s,a},v(q),\alpha)}.
\end{aligned}
\end{equation}
We will drop the $s,a$ dependence for simplicity. This motivates us to look at the dual functional applied to generic measureable $u:Y\ra \R$. Let's Define $w = e^{-u/\alpha}$.
\begin{align*}
    \abs{f(\mu_n,u,\alpha) -f(\mu,u,\alpha)}
    &= \abs{g_n(0,\alpha) -g_n(1,\alpha)}\\
    &= \abs{\del_tg_n(t,\alpha)\bigg |_{t=\tau}}\\
    &= \alpha\abs{\frac{m_n[w]}{\mu_n(\tau)[w]}}
\end{align*}
for some random variable $\tau\in(0,1)$. To bound this, we introduce the following lemma for which the proof is deferred to \ref{a_sec:proof:lemma_kl}. 

\begin{lemma}\label{lemma:alpha_ratio_bound}
Let $m = \mu_1-\mu_2$ with $\mu_1,\mu_2\ll\mu$ and $w = e^{-u/\alpha}$, we have that
\begin{align*}
\sup_{\alpha\geq 0}\frac{\alpha^j m[w]^2}{\mu[w]^2}
\leq  3^j\inf_{\kappa\in\R}\|u-\kappa\|_{L^\infty(\mu)}^j \norm{\frac{dm}{d\mu}}_{L^\infty(\mu)}^2.
\end{align*}
\end{lemma}
To apply Lemma \ref{lemma:alpha_ratio_bound}, we consider $p\leq \frac{1}{2}\mu_{\wedge}$. Then, By Lemma \ref{lemma:equiv_of_meas_cond}, on $\Omega_{n,p}(\mu)$, $\mu_n(t)\sim \mu$ for all $t\in[0,1]$. So, on $\Omega_{n,p}(\mu)$, we have by Lemma \ref{lemma:alpha_ratio_bound}
\begin{align*}
\sup_{\alpha\geq 0}\abs{f(\mu_n,u,\alpha) -f(\mu,u,\alpha)}
&\leq 
\sup_{\alpha\geq 0}\alpha\abs{\frac{m_n[w]}{\mu_n(\tau)[w]}}\\
&\leq \inf_{\kappa\in\R} 3\Linfnorm{u-\kappa}{\mu}\Linfnorm{\frac{dm_n}{d\mu_n(\tau)}}{\mu}\\
&=  3\spnorm{u}\Linfnorm{\frac{dm_n}{d\mu_n(\tau)}}{\mu}.
\end{align*}

\par Therefore, by partitioning $\Omega$ into $\Omega_{n,p}(\mu)^c$ and $\Omega_{n,p}(\mu)$, we bound
\begin{equation}\label{eqn:dual_func_var_two_term_kl}
\begin{aligned}
&E\sup_{\alpha\geq 0}\abs{f(\mu_n,u,\alpha) -f(\mu,u,\alpha)}^2 \\
&\leq 9\spnorm{u}^2 E\Linfnorm{\frac{dm_n}{d\mu_n(\tau)}}{\mu}^2\1_{\Omega_{n,p}(\mu)} + 4\spnorm{u}^2P(\Omega_{n,p}(\mu)^c)
\end{aligned}
\end{equation}
where on $\Omega_{n,p}(\mu)^c$, we use the bound in Lemma \ref{lemma:sup_f_diff_bound}. 
\par By Lemma \ref{lemma:equiv_of_meas_cond}, on $\Omega_{n,p}(\mu)$ for $y$ s.t. $\mu(y) > 0$, $\mu_n(y) \geq \mu_\wedge -p\geq \frac{1}{2}\mu_\wedge \geq p$. Since $\mu(y) > 0$ implies that $\mu(y) \geq \mu_\wedge $, we have that $\mu_n(t)(y) \geq p$ for any $t\in[0,1]$. Therefore, Lemma \ref{lemma:E_RND_bounds} applies. On the other hand, Lemma \ref{lemma:prob_Omega_n,p_mu^c_bound} also applies and is used to bound $P(\Omega_{n,p}(\mu)^c)$. 
\par Therefore, continue from \eqref{eqn:dual_func_var_two_term_kl}, we have
\begin{align*}
E\sup_{\alpha\geq 0}\abs{f(\mu_n,u,\alpha) -f(\mu,u,\alpha)}^2
&\leq13\frac{\spnorm{u}^2}{p^2n}\log(e|Y|). 
\end{align*}

\par We conclude that choosing $p =  \frac{1}{2}\fr{p}_{\wedge}\leq \frac{1}{2}\min\set{\nu_{s,a,\wedge},p_{s,a,\wedge}}$, 
\begin{align*}
    \var(\widehat{\mathbf{T}}(q)(s,a)) &\leq 2E\sup_{\beta\geq 0}\abs{f(\nu_{s,a,n},id,\beta)-f(\nu_{s,a},id,\beta)}^2\\
    &\quad +2\gamma^2E\sup_{\alpha\geq 0} \abs{f(p_{s,a,n},v(q),\alpha)-f(p_{s,a},v(q),\alpha)}^2\\
    &\leq 26\frac{\norm{id}^2_{\nu_{s,a},\mathrm{span}}}{p^2n}\log(e|\bd{R}|)+26\gamma^2\frac{\norm{v(q)}^2_{p_{s,a},\mathrm{span}}}{p^2n}\log(e|\bd{S}|)\\
    &\leq 26\frac{\rmax^2 + \gamma^2\spnorm{q}^2}{p^2n}\log(e(|\bd{R}|\vee |\bd{S}|)).
\end{align*}
Plugging in $p= \frac{1}{2}\fr{p}_{\wedge}$, we obtain the claimed inequality in Proposition \ref{prop:variance_empirical_Bellman}. 
\end{proof}

\subsection{Proof of Proposition \protect\ref{prop:bias_empirical_Bellman}}

\begin{proof}
We consider for generic $u$ and measure $\mu$ on $(Y,2^Y)$. We assume $\delta < \frac{1}{24}\mu_\wedge$, which will be guaranteed by Assumption \ref{assump:delta_small}.
\par Since $\alpha\ra f(\mu,u,\alpha)$ is continuous, and from \citet{si2020} it is sufficient to optimize the Lagrange multiplier on compact set $[0,\delta\inv \|u\|_{L^\infty(\mu)}]$, there is an optimal Lagrange multiplier $\alpha^*_n(t)$ that achieves $\sup_{\alpha \geq 0}f(\mu_n(t),u,\alpha)$. 
\par The bias of the dual functional
\begin{equation}\label{eqn:bias_two_terms_kl}
\begin{aligned}
&\bias(f(\mu_n,u,\alpha_n^*)) \\
&= E(g_n(0,\alpha_n^*(0))-g_n(1,\alpha^*))\1_{\Omega_{n,p}(\mu)} + E\crbk{ g_n(0,\alpha_n^*(0))-g_n(1,\alpha^*)}\1_{\Omega_{n,p}(\mu)^c}\\
&=: E_1+E_2. 
\end{aligned}
\end{equation}
We fix $p\leq \frac{1}{4}\mu_\wedge $. Notice that by assumption, 
\begin{equation}
    \delta<\frac{1}{24}\mu_{\wedge} <\frac{1}{2}\mu_\wedge\leq  -\log\crbk{1-\frac{1}{2}\mu_\wedge}. 
\end{equation}
Then, the following Lemma \ref{lemma:dual_func_derivative} holds.
\begin{lemma}[Differentiability of the Dual Functional]\label{lemma:dual_func_derivative}
Suppose $\delta < -\log(1-\frac{1}{2}\mu_{\wedge})$ and $p\leq \frac{1}{4}\mu_\wedge $, then

\begin{itemize}
    \item On $\Omega_{n,p}(\mu)$, $t\ra \sup_{\alpha\geq 0}g_n(t,\alpha)$ is $C^2((0,1))\cap C[0,1]$.
    \item $\alpha^* = 0$ iff $u$ is $\mu$ essentially constant. So, $\alpha_n^*(t)\equiv 0$ and $\sup_{\alpha\geq 0}g_n(t,\alpha) \equiv\mu[u]$
    \item If $\alpha^*> 0$, then $\alpha_n^*(t) > 0$ for all $t\in[0,1]$ with
    \[
d_t\sup_{\alpha\geq 0}g_n(t,\alpha) = -\alpha_n^*
(t)\frac{m_n[w]}{\mu_n(t)[w]}
\]
and
\begin{equation}\label{eqn:g_2nd_derivative_kl}
\begin{aligned}
&d_td_t\sup_{\alpha\geq 0} g_n(t,\alpha)\\
&= -\alpha_n^*(t)\frac{m_{n}[w]^2}{\mu_{n}(t)[w]^2}\\
&\quad - \crbk{\frac{\alpha_n^*(t)^3}{\var_{\mu^*_n(t)}(u)}}\crbk{\frac{m_n[w]}{\mu_n(t)[w]}+ \frac{m_n[uw]}{\alpha_n^*(t)\mu_n(t)[w]}- \frac{m_n[w]\mu_n(t)[uw]}{\alpha_n^*(t)\mu_n(t)[w]^2}}^2.
\end{aligned}
\end{equation}
\end{itemize}  

\end{lemma}
\par The proof of this result is deferred to Appendix \ref{a_sec:proof:lemma_kl}. 
\par So, on $\Omega_{n,p}(\mu)$, $t\ra g_n(t,\alpha_n^*(t))$ is $C^2(0,1)\cap C[0,1]$. By the (second order) mean value theorem, there exists random variable $\tau\in[0,1]$ s.t. 
\begin{align*}
E_1
&= E\crbk{- d_tg_n(t,\alpha^*_n(t)) \bigg|_{t=1} + \frac{1}{2}d_td_tg_n(t,\alpha^*_n(t)) \bigg|_{t=\tau}}\1_{\Omega_{n,p}(\mu)} \\
&= E\crbk{ \alpha^*\frac{m_n[w]}{\mu[w]} + \frac{1}{2}d_td_tg_n(t,\alpha^*_n(t))  \bigg|_{t=\tau}}\1_{\Omega_{n,p}(\mu)}\\
&= \alpha^*\frac{Em_n[w]}{\mu[w]} -E \alpha^*\frac{m_n[w]}{\mu[w]}\1_{\Omega_{n,p}(\mu)^c} + E\sqbk{\frac{1}{2}d_td_tg_n(t,\alpha^*_n(t))  \bigg|_{t=\tau}\1_{\Omega_{n,p}(\mu)}}\\
&= 0 - E_{1,1}+E_{1,2}
\end{align*}
where $Em_n[u] = 0$ for any function $u$. Recall Lemma \ref{lemma:alpha_ratio_bound}. Since naturally $\mu\gg\mu,\mu_n$, 
\begin{align*}
|E_{1,1}| &\leq 3\spnorm{u}E\norm{\frac{dm_n}{d\mu}}_{L^\infty(\mu)} \1_{\Omega_{n,p}(\mu)^c}\\
&\leq 3\frac{\spnorm{u}}{\mu_\wedge } P(\Omega_{n,p}(\mu)^c)\\
&\leq 3\frac{\spnorm{u}}{\mu_\wedge p^2n} \log(e|Y|),
\end{align*}
where we use Lemma \ref{lemma:prob_Omega_n,p_mu^c_bound} for the last inequality. 
\par On $\Omega_{n,p}(\mu)$, by Lemma \ref{lemma:dual_func_derivative}, for all $t\in[0,1]$ either $\alpha_n^*(t) = 0$ or $\alpha_n^*(t) > 0$. In the first case, we have trivially $E_{1,2} = 0$. In the second case, 
\begin{align*}
    -d_td_tg_n(t,\alpha^*_n(t)) &= -d_t\del_tg_n(t,\alpha^*_n(t))\\
    &= \alpha_n^*(t)\frac{m_{n}[w]^2}{\mu_{n}(t)[w]^2}  \\
    &\quad +\crbk{\frac{\alpha_n^*(t)^3}{\var_{\mu^*_n(t)}(u)}}\crbk{\frac{m_n[w]}{\mu_n(t)[w]}  + \frac{m_n[uw]}{\alpha_n^*(t)\mu_n(t)[w]}- \frac{\mu_n(t)[uw] m_n[w]}{\alpha_n^*(t)\mu_n(t)[w]^2}}^2.
\end{align*}


Next, we prove a finer characteristic when $\delta$ goes to 0. We need the following Lemma: 
\begin{lemma}\label{lemma:bias_alpha^3_term}
On $\Omega_{n,p}(\mu)$ with $p < \mu_\wedge $
\[
\sup_{\alpha\geq 0}\alpha^3
    \crbk{\frac{m_n[w]}{\mu_n(t)[w]}+ \frac{m_n[uw]}{\alpha\mu_n(t)[w]}- \frac{m_n[w]\mu_n(t)[uw]}{\alpha\mu_n(t)[w]^2} }^2\leq 136\inf_{\kappa\in\R} \|u-\kappa\|_{L^\infty(\mu)}^3 \norm{\frac{dm_n}{d\mu_n(t)}}_{L^\infty(\mu)}^2.
\]
\end{lemma}
Applying Lemma \ref{lemma:alpha_ratio_bound} and \ref{lemma:bias_alpha^3_term}, we have that on $\Omega_{n,p}(\mu)$
\begin{align*}
&|d_td_tg_n(t,\alpha^*_n(t))|\1_{\Omega_{n,p}(\mu)} \\
&\leq 3\inf_{\kappa\in\R}\|u-\kappa\|_{L^\infty(\mu)} \norm{\frac{dm_n}{d\mu_n(t)}}_{L^\infty(\mu)}^2 + 136\frac{\inf_{\kappa\in\R} \|u-\kappa\|_{L^\infty(\mu)}^3}{\var_{\mu^*_n(t)}(u)} \norm{\frac{dm_n}{d\mu_n(t)}}_{L^\infty(\mu)}^2\\
&\leq 3\spnorm{u}\norm{\frac{dm_n}{d\mu_n(t)}}_{L^\infty(\mu)}^2 + 136\spnorm{u}\frac{ \|u-\mu_n^*[u]\|_{L^\infty(\mu_n^*)}^2}{\|u-\mu_n^*(t)[u]\|_{L^2(\mu_n^*)}^2} \norm{\frac{dm_n}{d\mu_n(t)}}_{L^\infty(\mu)}^2
\end{align*}
To bound the second ratio in the last inequality, we introduce the following lemma, whose proof is deferred to Appendix \ref{a_sec:proof:lemma_kl} as well. 
\begin{lemma}\label{lemma:worst_case_meas_bound}
Suppose $\delta\leq\frac{1}{24}\mu_\wedge $ and $p\leq \frac{1}{4}\mu_\wedge $. When the optimal Lagrange multiplier $\alpha^*>0$, worst-case measures $\mu_n^*(t) = \mu_n(t)[w\cd]/\mu_n(t)[w]$ satisfies $\mu_n^*(t)(y)\geq \frac{1}{2}\mu_\wedge $ on $\Omega_{n,p}(\mu)$.
\end{lemma}
For $\delta\leq \frac{1}{24}\fr{p}_{\wedge}$, by Lemma \ref{lemma:worst_case_meas_bound}, for some $y'$ s.t. $\mu_n^*(t)(y') > 0$, 
\begin{align*}
    \frac{ \|u-\mu_n^*[u]\|_{L^\infty(\mu_n^*)}^2}{\|u-\mu_n^*(t)[u]\|_{L^2(\mu_n^*)}^2}
    &= \frac{ |u(y')-\mu_n^*[u]|^2}{\mu_n^*(t)(y')|u(y')-\mu_n^*[u]|^2 + \sum_{y\neq y'}\mu_n^*(t)(y)|u(y)-\mu_n^*[u]|^2}\\
    &\leq \frac{ |u(y')-\mu_n^*[u]|^2}{\mu_n^*(t)(y')|u(y')-\mu_n^*[u]|^2 }\\
    &\leq \frac{ 2}{\mu_\wedge  }
\end{align*}
As in the proof of Propositon \ref{prop:variance_empirical_Bellman}, under the choice $p\leq\frac{1}{4}\mu_\wedge $, Lemma \ref{lemma:E_RND_bounds} applies. Therefore,
\begin{align*}
|E_{1,2}|&\leq 275\frac{\spnorm{u}}{\mu_\wedge p^2 n}\log(e|Y|)
\end{align*}
For $E_2$ in \eqref{eqn:bias_two_terms_kl}, we use Lemma \ref{lemma:sup_f_diff_bound} and previous bound on $P\crbk{\Omega_{n,p}(\mu)^c}$
\begin{align*}
    |E_2|&\leq E\abs{ f(\mu_n,u,\alpha_n^*(0))-f(\mu,u,\alpha^*)}\1_{\Omega_{n,p}(\mu)^c}\\
    &\leq 2\spnorm{u}P\crbk{\Omega_{n,p}(\mu)^c}\\
    &\leq 2\frac{\spnorm{u}}{p^2n}\log(e|Y|)\\
    &\leq \frac{\spnorm{u}}{\mu_\wedge  p^2n}\log(e|Y|)
\end{align*}
Therefore, going back to \eqref{eqn:bias_two_terms_kl}, we have
\[
\abs{\bias\crbk{\sup_{\alpha\geq 0}f(\mu_n,u,\alpha)}}\leq 280\frac{\spnorm{u}}{\mu_\wedge p^2n}\log(e|Y|).
\]
Apply this to the empirical Bellman operator with $p = \frac{1}{4}\fr{p}_{\wedge}\leq \frac{1}{4}\min\set{p_{s,a,\wedge},\mu_{s,a,\wedge}}$ and Assumption \ref{assump:delta_small} holds. So, $\delta < \frac{1}{24}\fr{p}_\wedge$ implies that $\delta < \frac{1}{24}\min\set{p_{s,a,\wedge},\mu_{s,a,\wedge}}$. Therefore, we have 
\begin{align*}
|\bias(\widehat{\mathbf{T}}(q)(s,a))|&= \abs{\bias\crbk{\sup_{\beta\geq 0}f(\nu_{s,a,n},id,\beta)} + \gamma\bias\crbk{\sup_{\alpha\geq 0}f(p_{s,a,n},v(q),\alpha)}}\\
&\leq  4480\frac{\|id\|_{\nu_{s,a},\mathrm{span} }+ \gamma\spnorm{v(q)}}{\fr{p}_{\wedge}^3 n}\log(e|\bd{S}|\vee |\bd{R}|)\\
&\leq  4480\frac{\rmax + \gamma\spnorm{q}}{\fr{p}_{\wedge}^3 n}\log(e|\bd{S}|\vee |\bd{R}|).
\end{align*}
\end{proof}

\subsection{Proof of Proposition \protect\ref{prop:high_prob_bound_empirical_Bellman}}

\begin{proof}
We recall the bound \eqref{eqn:Bellman_difference_dual_bound} and the subsequent result
\begin{align*}
\sup_{\alpha\geq 0}\abs{f(\mu_n,u,\alpha) -f(\mu,u,\alpha)}
&\leq  3\spnorm{u}\Linfnorm{\frac{dm_n}{d\mu_n(\tau)}}{\mu}.
\end{align*}
Again, we consider $p\leq \frac{1}{2}\mu_{\wedge}$. Also recall the definition \eqref{eqn:Omega_n,p_mu} of $\Omega_{n,p}(\mu)$. By Lemma \ref{lemma:equiv_of_meas_cond}, on $\Omega_{n,p}(\mu)$ for $y$ s.t. $\mu(y) > 0$, $\mu_n(y) \geq \mu_\wedge -p\geq \frac{1}{2}\mu_\wedge\geq p$. Since $\mu(y) > 0$ implies that $\mu(y) \geq \mu_\wedge$, we have that $\mu_n(t)(y) \geq p$ for any $t\in[0,1]$. Therefore, we have
\begin{align*}
&P\crbk{\sup_{\alpha\geq 0} \abs{f(\mu_n,u,\alpha) -f(\mu,u,\alpha)} > t}\\
&\leq P(\Omega_{n,p}(\mu)^c) +P\crbk{3\spnorm{u}\Linfnorm{\frac{dm_n}{d\mu_n(\tau)}}{\mu}> t,\Omega_{n,p}(\mu)}\\
&\leq P\crbk{\sup_{y}|\mu_{n}(y)-\mu(y)| > p}+P\crbk{\frac{3\spnorm{u}}{p}\sup_{y} |m_{n}(y)| > t}\\
&\leq 2\sum_{y}\crbk{\exp(-2p^2n)+\exp\crbk{-\frac{2p^2t^2n }{9\spnorm{u}^2}}}\\
&\leq 2|Y|\crbk{\exp(-2p^2n)+\exp\crbk{-\frac{2p^2t^2n }{9\spnorm{u}^2}}}
\end{align*}
where we used Hoeffding's inequality and union bound. 
\par Therefore, going back to the DR Bellman operator setting, we choose $p = \frac{1}{4}\fr{p}_{\wedge}$ and by union bound
\begin{align*}
&P(\|\widehat{\mathbf{T}}(q)-\cT(q)\|_\infty > t)\\
&\leq P\crbk{\sup_{s,a}\sup_{\beta\geq 0}|f(\nu_{s,a,n},id,\beta)-f(\nu_{s,a},id,\beta)|>\frac{t}{2}}\\
&\qquad +P\crbk{\sup_{s,a}\sup_{\alpha\geq 0}|f(p_{s,a,n},v(q),\beta)-f(p_{s,a},v(q),\beta)|>\frac{t}{2}}\\ 
&\leq 2(|\bd{S}|^2|\bd{A}|+|\bd{S}||\bd{A}||\bd{R}|)\exp\crbk{-\frac{\fr{p}_{\wedge}^2n}{8}} + 2|\bd{S}||\bd{A}||\bd{R}|\exp\crbk{-\frac{\fr{p}_{\wedge}^2t^2n}{288\rmax^2}}\\
& \qquad  + 2|\bd{S}|^2|\bd{A}|\exp\crbk{-\frac{\fr{p}_{\wedge}^2t^2n}{288\gamma^2\spnorm{q}^2}}. 
\end{align*}
We set each of the three terms to be less than $\eta/3$ and find that it suffices to have
\[
n\geq  \frac{8}{\fr{p}_{\wedge}^2}\log\crbk{12|\bd{S}||\bd{A}|(|\bd{S}|\vee |\bd{R}|)/\eta}
\]
and 
\[
t\geq \frac{17(\rmax+\gamma\spnorm{q})}{\fr{p}_{\wedge}\sqrt{n}}\sqrt{\log\crbk{6|\bd{S}||\bd{A}|(|\bd{S}|\vee |\bd{R}|)/\eta}}.
\]
This implies the statement of the proposition. 
\end{proof}

\subsection{Proof of Proposition \protect\ref{prop:recentered_op_bias_and_var_KL}}\label{a_sec:proof:prop:recentered_op_bias_and_var_KL}

\begin{proof}
We define \begin{equation}\label{eqn:def_V_double_diff_kl}
    V:=\cH(\hat q) -\widehat{\mathbf{H}}(\hat q) = (\cT(\hat q) - \cT(q_*)) - (\widehat{\mathbf{T}}(\hat q) - \widehat{\mathbf{T}}(q_*)). 
\end{equation}
Recall the dual formulation
\[
\cT(q)(s,a) = \sup_{\beta\geq 0}f(\nu_{s,a},id,\beta)+\gamma \sup_{\alpha\geq 0}f(p_{s,a},v(q),\alpha).
\]
The first term is not dependent on $q$, hence canceled in $V$. We have that
\begin{align*}
    |V(s,a)| = \gamma \abs{\sup_{\alpha\geq 0} f(p_{s,a},v(\hat q),\alpha) - \sup_{\alpha\geq 0} f(p_{s,a},v(q^*),\alpha) -  \sup_{\alpha\geq 0} f(p_{s,a,n},v(\hat q),\alpha)+ \sup_{\alpha\geq 0} f(p_{s,a,n},v(q^*),\alpha)}
\end{align*}
Note that if $v(\hat q)$ and $v( q^*)$ are both $\mu$ essentially constant, then $V = 0$, and the claim of Proposition \ref{prop:recentered_op_bias_and_var_KL} holds trivially. Therefore, moving forward, we consider the case at least one of $v(\hat q)$ and $v( q^*)$ is not $\mu$ essentially constant.
\par To analyze $V$ while keeping the consistency of our notations, we define $v_t = t v(\hat q) +(1-t)v(q_*)$, $\mu =p_{s,a}$, $\mu_n =p_{s,a,n}$, $m = \mu-\mu_{n}$, and $\mu(t) = t\mu - (1-t)\mu_n$. Because Assumption \ref{assump:delta_small} is imposed, we have that $\delta < \frac{1}{24}\mu_\wedge$. 

\par We consider the parametric function for $s,t\in[0,1]$ 
\begin{equation}\label{eqn:def_hst_kl}
    h(s,t):= \sup_{\alpha\geq 0} f(\mu(t),v_s,\alpha) =  f(\mu(t),v_s,\alpha^*_{s,t}).
\end{equation}
To motivates our  analysis, we assume that $h(s,\cd)$ is $C^1(0,1)\cap C[0,1]$ and $\del_{t}h(\cd,t)$ is $C^1(0,1)\cap C[0,1]$ as well. Then the fundamental theorem of calculus implies that
\begin{equation}\label{eqn:V_integral_rep}
\begin{aligned}
    |V(s,a)| 
    &= \gamma \abs{h(1,0) - h(0,0) -h(1,1) + h(0,1)}\\
    &= \gamma \abs{-\int_0^1\del_t h(1,t)dt + \int_0^1\del_t h(0, t)dt}\\
    &= \gamma \abs{\int_0^1\int_0^1\del_s\del_t h(s,t)ds dt }\\
    &\leq \gamma\int_0^1\int_0^1 \abs{\del_s\del_t h(s,t)}ds dt 
\end{aligned}
\end{equation}
where $\del_s\del_t h(s,t)$ is easier to analyze. We proceed to show that \eqref{eqn:V_integral_rep} is valid (with some minor modification) on $\Omega_{n,p}(\mu)$. 

\par As in the proof of Proposition \ref{prop:bias_empirical_Bellman}, Lemma \ref{lemma:dual_func_derivative} applies when we consider $p\leq \frac{1}{4}\mu_\wedge$. So, for $p\leq \frac{1}{4}\mu_\wedge$, on $\Omega_{n,p}(\mu)$, $h(s,\cd)$ is $C^2(0,1)\cap C[0,1]$ with derivative
\[
\del_th(s,t) = d_t\sup_{\alpha\geq 0}f(\mu(t),v_s,\alpha) = -\alpha_{s,t}^*\frac{m[w_s]}{\mu(t)[w_s]}.
\]
Here, by Lemma \ref{lemma:dual_func_derivative}, $\alpha_{s,t}^*$ is the unique optimal Lagrange multiplier, and $w_s = e^{-v_s/\alpha_{s,t}^*}$. 
\par Next, we show that for every fixed $t$, there is a function $ D_s\del_th$ s.t. 
\begin{equation}\label{eqn:del_t_hst_weak_derivative}
\int_0^1 D_s\del_th(s,t) ds = \del_th(1,t) - \del_t h(0,t). 
\end{equation}
\par We note that by Lemma \ref{lemma:dual_func_derivative}, $\alpha_{s,t}^* = 0$ if and only if $v_s$ is essentially constant. This can only happen at one particular $s = s^*$. Otherwise, if there are some $0\leq s_1<s_2\leq 1$, 
$s_1v(\hat q) + (1-s_1)v(q^*) = c_1e$ and $s_2v(\hat q) + (1-s_2)v(q^*) = c_2e$ w.p.1 under $\mu$, where $e$ is the vector of all ones, then for all $a,b \geq 0$,
$$\frac{as_1+bs_2}{a+b}v(\hat q) +\crbk{1- \frac{as_1+bs_2}{a+b}}v(q^*) = (ac_1 +bc_2)e.$$
This would imply that $v(\hat q)$ and $v( q^*)$ are both essentially constant. 

\par We consider two cases: 
\par \textbf{Case 1: }$v_s$ is never essentially constant for all $s\in[0,1]$. 
\par In this case, $\alpha_{s,t}^* > 0$ for all $s\in[0,1]$. Note that $s\ra e^{-v_s/\alpha}$ is clearly $C^\infty$ for $\alpha>0$. So, on $\Omega_{n,p}(\mu)$ if $\alpha_{s,t}^*$ is $C^1(0,1)$ in $s$, then $s\ra \del_th(s,t)$ is $C^1(0,1)\cap C[0,1]$. 

\par We show differentiability of $s\ra\alpha_{s,t}^*$ by invoking the implicit function theorem as in the proof of Lemma \ref{lemma:dual_func_derivative}. When $\alpha_{s,t}^* > 0$, as shown in Lemma \ref{lemma:dual_func_derivative}, it is the unique solution to the optimality condition
\begin{equation}\label{eqn:alpha*_opt_eqn_kl}
\alpha_{s,t}^*(-\log\mu(t)[w_s]-\delta) -\frac{\mu(t)[v_sw_s]}{\mu(t)[w_s]}=: F(s,\alpha_{s,t}^*) = 0. 
\end{equation}
\par Define the optimal measure 
$$\mu^*(s,t)[\cd] = \frac{\mu(t)[w_s\cd]}{\mu(t)[w_s]}.$$ Since for all fixed $t$, $\alpha_{s,t}^* > 0$ on $(0,1)$ and $F$ is infinite smooth. The implicit function theorem then implies that $\alpha_{s,t}^*$ is $C^1(0,1)\cap C[0,1]$ and $s\ra \del_th(s,t)$ is $C^1(0,1)\cap C[0,1]$. 
\par We compute the derivative $\del_s\del_t h$ in this case. Let $\Delta_v= v(\hat{q})-v(q^*)$. Rewrite the optimality equation as
\[
\alpha_{s,t}^*(-\log\mu(t)[w_s]-\delta) = \frac{\mu(t)[v_sw_s]}{\mu(t)[w_s]}.
\]
Differentiate w.r.t. $s$ on both side
\begin{align*}
\lhs &= \del_s\alpha_{s,t}^*(-\log\mu(t)[w_s]-\delta) + \frac{\mu(t)[\Delta_v w_s]}{\mu(t)[w_s]} - \del_s\alpha_{s,t}^*\frac{\mu(t)[v_sw_s]}{\alpha_{s,t}^*\mu(t)[w_s]}\\
&= \del_s\alpha_{s,t}^*(-\log\mu(t)[w_s]-\delta) + \mu^*(s,t)[\Delta_v ] - \del_s\alpha_{s,t}^*\mu^*(s,t)[v_s/\alpha_{s,t}^*]
\end{align*}
\begin{align*}
    \rhs &= \frac{\mu(t)[\Delta_v w_s]\mu(t)[v_sw_s]}{\alpha_{s,t}^*\mu(t)[w_s]^2} + \frac{\mu(t)[\Delta_v w_s]}{\mu(t)[w_s]}- \frac{\mu(t)[\Delta_v v_sw_s]}{\alpha_{s,t}^*\mu(t)[w_s]} \\
    &\quad +  \del_s\alpha_{s,t}^*\crbk{-\frac{\mu(t)[ v_sw_s]^2}{(\alpha_{s,t}^*)^2\mu(t)[w_s]^2} + \frac{\mu(t)[ v_s^2w_s]}{(\alpha_{s,t}^*)^2\mu(t)[w_s]}} \\
    &= -\cov_{\mu^*(s,t)}\crbk{\Delta_v,v_s/\alpha_{s,t}^*} + \mu^*(s,t)[\Delta_v] +  \del_s\alpha_{s,t}^*\var_{\mu^*(s,t)}(v_s/\alpha_{s,t}^*)
\end{align*}
From the optimality equation and the LHS and RHS derivatives, we have
\begin{equation}\label{eqn:del_s_alpha^*_st}
\begin{aligned}
\del_s\alpha_{s,t}^*\crbk{\log\mu(t)[w_s]+\delta+ \mu^*(s,t)[v_s/\alpha_{s,t}^*]+ \var_{\mu^*(s,t)}(v_s/\alpha_{s,t}^*)} &= \cov_{\mu^*(s,t)}(\Delta_v,v_s/\alpha_{s,t}^*)\\
\del_s\alpha_{s,t}^* \var_{\mu^*(s,t)}(v_s/\alpha_{s,t}^*)&=\cov_{\mu^*(s,t)}(\Delta_v,v_s/\alpha_{s,t}^*)\\
\del_s\alpha_{s,t}^* &=\frac{\cov_{\mu^*(s,t)}(\Delta_v,v_s/\alpha_{s,t}^*)}{\var_{\mu^*(s,t)}(v_s/\alpha_{s,t}^*)}.
\end{aligned}
\end{equation}

Therefore, when $\alpha_{s,t}^* > 0$, 
\begin{equation}\label{eqn:del_s_del_t_hst_decomp}
\begin{aligned}
\del_s\del_t h(s,t) 
&= \del_s\frac{-\alpha^*_{s,t} m[w_s]}{\mu(t)[w_s]}\\
&= -\frac{m[w_s]\mu(t)[\Delta_v w_s]}{\mu(t)[w_s]^2} + \frac{m[\Delta_v w_s]}{\mu(t)[w_s]}-\del_s\alpha_{s,t}^*\frac{ m[w_s]}{\mu(t)[w_s]} \\
&\quad + \del_s\alpha_{s,t}^*\crbk{- \frac{m[v_sw_s]}{{\alpha_{s,t}^*} \mu(t)[w_s]}+\frac{m[w_s]\mu(t)[v_sw_s]}{{\alpha_{s,t}^*} \mu(t)[w_s]^2} }
\\
&=: D_1+D_2+D_3+D_4.
\end{aligned}
\end{equation}

\par \textbf{Case 2: }There is a unique $s^*\in[0,1]$ s.t. $v_s$ is essentially constant. 
\par In this case, the previous argument implies that  $s\ra \del_th(s,t)$ is $C^1(0,s^*)$, $C^1(s^*,1)$, and continuous at $0,1$. The derivative is also given by \eqref{eqn:del_s_del_t_hst_decomp}. 
\par We need to show the existence of $D_s\del_th$ that satisfy \eqref{eqn:del_t_hst_weak_derivative}. Observe that if $s\ra \del_th(s,t)$ is continuous at $s^*$, then applying the fundamental theorem of calculus on the interval $[0,s^*]$ and $[s^*,1]$ separately, we will have that
\[
\del_th(1,t) - \del_t h(0,t) = \int_0^{s^*} \del_s\del_th(s,t) ds+ \int_{s^*}^1 \del_s\del_th(s,t) ds. 
\]
Hence, taking $D_s\del_t h(s,t) = \del_s\del_t h(s,t)$ for every $s\neq s^*$ and $D_s\del_t h(s^*,t) = 0$ will suffice to produce \eqref{eqn:del_t_hst_weak_derivative}. 
\par It is left to check the continuity at $s^*$. As analyzed in \eqref{eqn:continuity_alpha_da_0}, $$\lim_{\alpha\da_0}\alpha_{s,t}^*\frac{m[w_s]}{\mu(t)[w_s]} = 0. $$ So we can conclude the continuity of $s\ra\del_th(s,t)$ at $s^*$, if we can show that when $v_s\ra c e$ for some constant $c$, then $\alpha_{s,t}^*\da 0$. 
\par To prove this, we assume to the contrary that there is a subsequential limit $\alpha_{s_n,t}^*\ra \beta > 0$ for some sequence $s_n\ra s^*$. But since $F$ defined \eqref{eqn:alpha*_opt_eqn_kl} in $s$ and $\alpha$ when $\alpha > 0$, we must have that $$0 = \lim_{n\ra\infty}F(s_n,\alpha_{s_n,t}^*) = \beta(-\log\mu(t)[e^{-ce/\beta}]-\delta) -c = -\delta\beta$$
raising a contradiction. This implies that $s\ra\del_th(s,t)$ is continuous at $s^*$, and hence \eqref{eqn:del_t_hst_weak_derivative} holds with  $D_s\del_t h(s,t) = \del_s\del_t h(s,t)$ for every $s\neq s^*$ and $D_s\del_t h(s^*,t) = 0$. 
\par Therefore, we have shown that the bound \eqref{eqn:V_integral_rep} is valid on $\Omega_{n,p}(\mu)$ with $p\leq \frac{1}{4}\mu_\wedge$. 

\par Now we bound $\del_s\del_t h(s,t)$ using the decomposition \eqref{eqn:del_s_del_t_hst_decomp}. $|D_1|$ and $|D_2|$ can be bounded using the change of measure techniques: on $\Omega_{n,p}(\mu)$
\begin{align*}
|D_1| &\leq \frac{\mu(t)[\frac{dm}{d\mu(t)}w_s]\mu(t)[|\Delta_v| w_s]}{\mu(t)[w_s]^2} \\
&\leq \|\Delta_v\|_{\infty}\norm{\frac{dm}{d\mu(t)}}_{L^\infty(\mu)}
\end{align*}
and
\begin{align*}
|D_2| &\leq \frac{\mu(t)[\frac{dm}{d\mu(t)}|\Delta_v| w_s]}{\mu(t)[w_s]}. \\
&\leq \|\Delta_v\|_{\infty}\norm{\frac{dm}{d\mu(t)}}_{L^\infty(\mu)}
\end{align*}
To bound $|D_3|$ and $|D_4|$, recall $\del_s\alpha_{s,t}^*$ from \eqref{eqn:del_s_alpha^*_st}. 
\begin{align*}
|D_3| 
&= \abs{\del_s\alpha_{s,t}^*\frac{ m[w_s]}{\mu(t)[w_s]}}\\
&\leq \frac{\abs{\cov_{\mu^*(s,t)}(\Delta_v,v_s/\alpha_{s,t}^*)}}{\var_{\mu^*(s,t)}(v_s/\alpha_{s,t}^*)}\frac{ m[w_s]}{\mu(t)[w_s]}\\
&\leq \frac{\abs{\cov_{\mu^*(s,t)}(\Delta_v,v_s)}}{\var_{\mu^*(s,t)}(v_s)}\alpha_{s,t}^*\frac{ m[w_s]}{\mu(t)[w_s]}\\
&\stackrel{(i)}{\leq} 3\frac{\abs{\cov_{\mu^*(s,t)}(\Delta_v,v_s)}}{\var_{\mu^*(s,t)}(v_s)} \inf_{\kappa\in\R}\|v_s-\kappa\|_{L^\infty(\mu)} \norm{\frac{dm}{d\mu(t)}}_{L^\infty(\mu)}\\
&\leq 3\sqrt{\var_{\mu^*(s,t)}(\Delta_v)}\frac{\|v_s-\mu^*(s,t)[v_s]\|_{L^\infty(\mu)}}{\sqrt{\var_{\mu^*(s,t)}(v_s)}} \norm{\frac{dm}{d\mu(t)}}_{L^\infty(\mu)}\\
&\leq 3\|\Delta_v\|_\infty\frac{\|v_s-\mu^*(s,t)[v_s]\|_{L^\infty(\mu^*(s,t))}}{\|v_s-\mu^*(s,t)[v_s]\|_{L^2(\mu^*(s,t))}} \norm{\frac{dm}{d\mu(t)}}_{L^\infty(\mu)}
\end{align*}
where $(i)$ used Lemma \ref{lemma:alpha_ratio_bound} with $j = 1$. 
Since $\alpha_{s,t}^*>0$ for $s\in(0,1)$ and $v_s$ is not essentially constant, by Lemma \ref{lemma:worst_case_meas_bound}, for some $s'\in \bd{S}$ s.t. $v_s(s')-\mu^*(s,t)[v_s]\neq 0$
\begin{align*}
    &\frac{\|v_s-\mu^*(s,t)[v_s]\|^2_{L^\infty(\mu^*(s,t))}}{\|v_s-\mu^*(s,t)[v_s]\|^2_{L^2(\mu^*(s,t))}}\\
    &= \frac{ |v_s(s')-\mu^*(s,t)[v_s]|^2}{\mu^*(s,t)(s')|v_s(s')-\mu^*(s,t)[v_s]|^2 + \sum_{s''\neq s'}\mu^*(s,t)(s'')|v_s(s'')-\mu^*(s,t)[v_s]|^2}\\
    &\leq \frac{ |v_s(s')-\mu^*(s,t)[v_s]|^2}{\mu^*(s,t)(s')|v_s(s')-\mu^*(s,t)[v_s]|^2 }\\
    &\leq \frac{ 2}{\mu_\wedge  }
\end{align*}
So, 
\[
|D_3|\leq \frac{5\|\Delta_v\|_\infty}{\sqrt{\mu_\wedge }}  \norm{\frac{dm}{d\mu(t)}}_{L^\infty(\mu)}.
\]
From \eqref{eqn:del_s_alpha^*_st}, by the property of variance,
\begin{align*}
\abs{\del_s\alpha_{s,t}^*}&\leq \sqrt{\frac{\var_{\mu^*(s,t)}(\Delta_v)} {\var_{\mu^*(s,t)}(v_s/\alpha_{s,t}^*)}}\leq \frac{\|\Delta_v\|_{\infty}} {\sqrt{\var_{\mu^*(s,t)}(v_s/\alpha_{s,t}^*)}}. 
\end{align*}
Hence applying similar analysis, 
\begin{align*}
|D_4|&=\abs{\del_s\alpha_{s,t}^*\crbk{- \frac{m[v_sw_s]}{{\alpha_{s,t}^*} \mu(t)[w_s]}+\frac{m[w_s]\mu(t)[v_sw_s]}{{\alpha_{s,t}^*} \mu(t)[w_s]^2} }}\\
&= |\del_s\alpha_{s,t}^*|\abs{-\mu^*(s,t)\sqbk{\frac{dm}{d\mu(t)}v_s/\alpha_{s,t}^*}+\mu^*(s,t)\sqbk{\frac{dm}{d\mu(t)}}\mu^*(s,t)[v_s/\alpha_{s,t}^*]}\\
&\leq \frac{\|\Delta_v\|_{\infty}} {\sqrt{\var_{\mu^*(s,t)}(v_s/\alpha_{s,t}^*)}}\cov_{\mu^*(s,t)}\crbk{\frac{dm}{d\mu(t)},v_s/\alpha_{s,t}^*}\\
&\leq \|\Delta_v\|_{\infty}\sqrt{\var_{\mu^*(s,t)}\crbk{\frac{dm}{d\mu(t)}}}\\
&\leq \|\Delta_v\|_{\infty}\norm{\frac{dm}{d\mu(t)}}_{L^\infty(\mu)}.
\end{align*}

\par By \eqref{eqn:V_integral_rep}, we have
\begin{align*}
E|V| 
&\leq  E|V|\1_{\Omega_{n,p}(\mu)^c} + \gamma \int_0^1\int_0^1 E\abs{\del_s\del_th(s,t) }\1_{\Omega_{n,p}(\mu)}ds dt\\
&\leq E|V|\1_{\Omega_{n,p}(\mu)^c} + \gamma\sup_{s,t\in(0,1)} E(|D_1|+|D_2|+|D_3|+|D_4|)\1_{\Omega_{n,p}(\mu)}.
\end{align*}
Recall the definition \eqref{eqn:def_V_double_diff_kl} of $V$, 
\begin{align*}
\|V\|_\infty &= \|(\cT(\hat q) - \cT(q_*)) - (\widehat{\mathbf{T}}(\hat q) - \widehat{\mathbf{T}}(q_*))\|_\infty\\
    &\leq\|\cT(\hat q) - \cT(q_*) \|_\infty + \| \widehat{\mathbf{T}}(\hat q) - \widehat{\mathbf{T}}(q_*)\|_\infty\\
    &\leq 2 \gamma \|\hat q - q_*\|_\infty.
\end{align*}
So, by Lemma \ref{lemma:prob_Omega_n,p_mu^c_bound},
\begin{equation}\label{eqn:EV_on_Omega_c_bd_kl}
\begin{aligned}
E|V|\1_{\Omega_{n,p}(\mu)^c}
&\leq 2 \gamma\|\hat q - q_*\|_\infty P(\Omega_{n,p}(\mu)^c)\\
&\leq \frac{ 2 \gamma\|\hat q - q_*\|_\infty}{p^2n}\log(e|\bd{S}|)
\end{aligned}
\end{equation}
By the previous bounds on $|D_i|$, $i = 1,2,3,4$, 
\begin{equation}\label{eqn:EV_separate_and_bd_kl}
\begin{aligned}
E|V|&\leq \frac{2\gamma\|\hat q - q_*\|_\infty}{p^2n}\log(e|\bd{S}|)+ \gamma\sup_{s,t\in(0,1)} \frac{8}{\sqrt{ \mu_\wedge }}E\|\Delta_v\|_\infty\norm{\frac{dm}{d\mu(t)}}_{L^\infty(\mu)}\1_{\Omega_{n,p}(\mu)} \\
&\leq  \frac{2^{5}\gamma\|\hat q - q_*\|_\infty}{\mu_\wedge^2n}\log(e|\bd{S}|) + \frac{8\|\Delta_v\|_\infty}{p\sqrt{ \mu_\wedge }\sqrt{n}}\sqrt{\log(e|\bd{S}|)}\\
&\leq \frac{2^{5}\gamma\|\hat q - q_*\|_\infty}{\mu_\wedge^2n}\log(e|\bd{S}|) + \frac{2^{5}\|\hat q - q_*\|_\infty}{ \mu_\wedge^{3/2}\sqrt{n}}\sqrt{\log(e|\bd{S}|)}\\
& \leq \frac{2^{6}\|\hat q - q_*\|_\infty}{ \fr p_\wedge^{3/2}\sqrt{n}}\log(e|\bd{S}|)
\end{aligned}
\end{equation}
where we choose $p = \frac{1}{4}\mu_\wedge  \leq \frac{1}{4}\fr{p}_\wedge$ and the last inequality follows from the assumption in Proposition \ref{prop:recentered_op_bias_and_var_KL} that $n\geq \fr{p}_{\wedge}\inv$. 
\par To bound the variance, note that $\var(\widehat{\mathbf{T}}(\bar x) - \widehat{\mathbf{T}}(x_*)) \leq EV^2$ and
\begin{align*}
 EV^2\1_{\Omega_{n,p}(\mu)}
&\leq \gamma^2\int_{0}^1 \int_{0}^1 (\del_s\del_t h(s,t))^2dsdt
\end{align*}
which follows from applying Jensen's inequality to the $[0,1]^2$ integral. Therefore,
\begin{equation}\label{eqn:varV_separate_and_bd_kl}
\begin{aligned}
&\var(\widehat{\mathbf{T}}(\hat q) - \widehat{\mathbf{T}}(q^*))\\
&\leq 8\gamma^2\|\hat q - q_*\|_\infty^2  P(\Omega_{n,p}(\mu)^c) + \gamma^2 E\int_{0}^1 \int_{0}^1 4(D_1^2+D_2^2+D_3^2+D_4^2)dsdt\1_{\Omega_{n,p}(\mu)}\\
&\leq \frac{2^{7}\gamma^2\|\hat q - q_*\|_\infty^2}{\mu_\wedge^2n}\log(e|\bd{S}|)   + \frac{112}{\mu_\wedge }\|\Delta_v\|_{\infty}^2\sup_{s,t\in(0,1)}E\norm{\frac{dm}{d\mu(t)}}_{L^\infty(\mu)}^2\1_{\Omega_{n,p}(\mu)}\\
&\leq \frac{2^{7}\gamma^2\|\hat q - q_*\|_\infty^2}{\mu_\wedge^2n}\log(e|\bd{S}|)   + \frac{2^{11}\|\hat q - q_*\|_\infty^2}{\mu_\wedge^3n}\log(e|\bd{S}|).\\
&\leq  \frac{2^{12}\|\hat q - q_*\|_\infty^2}{\fr p_\wedge^3n}\log(e|\bd{S}|).
\end{aligned}
\end{equation}
\end{proof}
\subsection{Proof of Proposition \protect\ref{prop:recentered_op_high_prob_KL}}\label{a_sec:proof:prop:recentered_op_high_prob_KL}

\begin{proof}
Recall the notations and definitions in as the proof of Proposition \ref{prop:recentered_op_bias_and_var_KL} in Appendix \ref{a_sec:proof:prop:recentered_op_bias_and_var_KL} and, in particular, the definition \eqref{eqn:def_V_double_diff_kl} and bound \eqref{eqn:V_integral_rep} for $V$. We again choose $p \leq \frac{1}{4}\mu_\wedge = \frac{1}{4}p_{s,a,\wedge}$. As Appendix \ref{a_sec:proof:prop:recentered_op_bias_and_var_KL}, we have that
\begin{align*}
|V(s,a)|&\leq |V| \1_{\Omega_{n,p}(\mu)^c}+\gamma\sup_{s,t\in(0,1)} (|D_1|+|D_2|+|D_3|+|D_4|)\1_{\Omega_{n,p}(\mu)}
\end{align*}
where $\mu = p_{s,a}$. 
\par Since Assumption \ref{assump:delta_small} is assumed, the bounds on $D_1,D_2,D_3,D_4$ are still applicable. Therefore, by Hoeffding's inequality and union bound
\begin{align*}
&P(|V(s,a)| > t)\\
    &\leq P(\Omega_{n,p}(p_{s,a})^c) + P\crbk{\gamma\sup_{s,t\in(0,1)} (|D_1|+|D_2|+|D_3|+|D_4|)>t,\Omega_{n,p}(p_{s,a})}\\
    &\leq P\crbk{\sup_{s'\in \bd{S}}|p_{s,a,n}(s')-p_{s,a}(s')| > p} + P\crbk{\frac{8\|\hat q - q_*\|_\infty}{(p_{s,a,\wedge}-p)\sqrt{p_{s,a,\wedge}}}\sup_{s'\in\bd{S}}|m_n(s')|>t}\\
    &\leq\sum_{s'\in\bd{S}}\crbk{P\crbk{|m_n(s')| > p} + P\crbk{\frac{11\|\hat q - q_*\|_\infty}{p_{s,a,\wedge}^{3/2}}|m_n(s')|>t}}\\
    &\leq 2|\bd{S}|\crbk{\exp\crbk{-2p^2n} + \exp\crbk{-\frac{p_{s,a,\wedge}^{3/2}t^2n}{56\|\hat q - q_*\|_\infty^2}}}
\end{align*}
where $m_n = p_{s,a,n}-p_{s,a}$. 
Then, as $\fr{p}_{\wedge} \leq p_{s,a,\wedge}$ for all $(s,a)\in\bd{S}\times\bd{A}$, by union bound
\[
P(\|V\|_\infty > t)\leq 2|\bd{S}|^2|\bd{A}|\crbk{\exp\crbk{-\frac{\fr{p}_{\wedge}^2n}{8}} + \exp\crbk{-\frac{\fr{p}_{\wedge}^3t^2n}{56\|\hat q - q_*\|_\infty^2}}}.
\]
We first control the first term to be less than $\eta/2$, which is implied by
\[
n\geq \frac{8}{\fr{p}_{\wedge}^2}\log(4|\bd{S}|^2|\bd{A}|/\eta). 
\]
Finally, the second term less than $\eta/2$ is implied by choosing
\[
t^2 = \frac{56\|\hat q - q^*\|}{\fr{p}_{\wedge}^3 n}\log(4|\bd{S}|^2|\bd{A}|/\eta).
\]
This proves the claimed result. 
\end{proof}

\section{Proof of Technical Lemmas: Empirical Measures and Concentrations}\label{a_sec:proof:concentration_of_emp_meas}
The proofs in the rest of this section is based on the following concentration property of maximum subgaussian random variables. 
\subsection{Subgaussian Maximum Inequality}
\begin{lemma}\label{lemma:sub_g_k_max}
Let $\set{Y_i,i=1\ds n}$ be $\sigma^2$-sub-Gaussian with zero means, not necessarily independent, then 
\[
EZ:= E\max_{i=1\ds n}|Y_i|^k\leq 2^k \sigma^k \crbk{k-1+\log n}^{k/2}.
\]
\end{lemma}
\begin{proof}
For any $\lambda > 0$, consider an increasing function $\phi_\lambda(z) = \exp(\lambda z^{1/k})$ for $z\geq 0$. Since $Z\geq 0$, 
\[
\begin{aligned}
\phi_\lambda( EZ) &= \phi_\lambda( EZ\1\set{Z > u}+ EZ\1\set{Z \leq u})\\
&\leq \phi_\lambda( EZ\1\set{Z > u}+ u P(Z \leq u))\\
&\leq \phi_\lambda( EZ+ u )\\
\end{aligned}
\]
Take second derivatives, 
\[
\phi_\lambda''(z) = k^{-2}\lambda  z^{1/k-2} e^{\lambda  z^{1/k}} (\lambda  z^{1/k}-k+1);
\]
one can see that $\phi_\lambda(z)$ is convex for $z \geq (k-1)^k\lambda^{-k}$. Let $u = (k-1)^k\lambda^{-k}$. By Jensen's inequality
\[
\begin{aligned}
\phi_\lambda( EZ)&\leq E\phi_\lambda(Z+ (k-1)^k\lambda^{-k})\\
&= e^{k-1} E\exp(\lambda \max_{i=1\ds n} |Y_i|)\\
&\leq e^{k-1}\sum_{i=1}^n Ee^{\lambda |Y_i|} 
\end{aligned}
\]
Since $\set{Y_i}$ are Sub-Gaussian,
\[
P(|Y_i|> t)\leq 2\exp\crbk{-\frac{t^2}{2\sigma^2}}
\]
By \citet[Lemmas 1.4 and 1.5]{rigollet201518}, one can show that
\[
\log Ee^{\lambda |Y_i|}\leq \log( E[e^{\lambda Y_i} +  e^{-\lambda Y_i}])\leq  \log(2\exp(\sigma^2\lambda^2/2))
\leq 4\sigma^2\lambda^2.
\]
Therefore, 
\[
\begin{aligned}
\lambda \crbk{E\max_{i=1\ds n}|Y_i|^k}^{1/k} &= \log\phi_\lambda (EZ)\\
&\leq k-1+\log n + 4\sigma^2\lambda^2. 
\end{aligned}
\]
Rearrange and take infimum over $\lambda > 0$, we conclude
\[
\begin{aligned}
E\max_{i=1\ds n}|Y_i|^k&\leq \crbk{\inf_{\lambda > 0}\frac{k-1+\log n }{\lambda}+ 4\sigma^2\lambda}^k\\
&\leq 2^k \sigma^k \crbk{k-1+\log n}^{k/2}
\end{aligned}
\]
\end{proof}

\subsection{Proof of Lemma \protect\ref{lemma:prob_Omega_n,p_mu^c_bound}}

\begin{proof}
By definition and Markov's inequality
\begin{align*}
     P(\Omega_{n,p}(\mu)^c) &= P\crbk{\sup_{y}|\mu_{n}(y) - \mu(y)|>p}\\
&\leq\frac{1}{p^{2k}} E\sqbk{\sup_{y}|\mu_{n}(y) - \mu(y)|^{2k}}\\
&= \frac{1}{p^{2k}n^{2k}}E\sqbk{\sup_{y}\crbk{\sum_{i=1}^n\1\set{Y_i = y} - \mu(y)}^{2k}}
\end{align*}
Since $\sum_{i=1}^n\1\set{Y_i = y} - \mu(y)$ is $n/4$ sub-Gaussian, by Lemma \ref{lemma:sub_g_k_max}
\[
P(\Omega_{n,p}(\mu)^c)\leq \frac{1}{p^{2k}n^k} (2k-1+\log(|Y|))^k=  \frac{1}{p^{2k}n^k} \log(e^{2k-1}|Y|)^k
\]
as claimed. 
\end{proof}

\subsection{Proof of Lemma \protect\ref{lemma:E_RND_bounds}}

\begin{proof}
Note that by Jensen's inequality,
\[
E\Linfnorm{\frac{dm_n}{d\xi_n}}{\xi_n}^2\1_A\geq \crbk{E\Linfnorm{\frac{dm_n}{d\xi_n}}{\xi_n}\1_A}^2.
\]
So it suffices to show the second claim. By assumption,
\[
E\Linfnorm{\frac{dm_n}{d\xi_n}}{\xi_n}^2\1_A\leq \frac{1}{p^2}E\sup_{y}|m_n(y)|^2\1_A.
\]
Same as the proof of Lemma \ref{lemma:prob_Omega_n,p_mu^c_bound}, we use Lemma \ref{lemma:sub_g_k_max} to conclude that
\[
E\Linfnorm{\frac{dm_n}{d\xi_n}}{\xi_n}^2\1_A\leq \frac{1}{p^2n}\log(e|Y|).
\]
\end{proof}

\section{Proof of Technical Lemmas: KL Case}\label{a_sec:proof:lemma_kl}

\subsection{Proof of Lemma \protect\ref{lemma:sup_f_bound}}

\begin{proof}
\[
\sup_{\alpha\geq 0} f(\nu, u,\alpha) \geq \lim_{\alpha\da 0}f(\nu,u,\alpha) = \essinf_{\nu}u \geq\essinf_{\mu}u\geq -\|u\|_{L^\infty(\mu)}
\]
On the other hand, since the $\sup$ is achieved on compact $K$. For optimal $\alpha_\nu^* > 0$, 
\[
\begin{aligned}
\sup_{\alpha\geq 0}f(\nu,\alpha)
&\leq \|u\|_{L^\infty(\nu)} -\alpha_\nu^*\log \nu[e^{-(u-\|u\|_{L^\infty(\nu)})/\alpha_\nu^*}]\\
&\leq \|u\|_{L^\infty(\mu)}
\end{aligned}
\]
where the last line follows from that $\nu[e^{-(u-\|u\|_{L^\infty(\nu)})/\alpha_\nu^*}] > 0$ and $ \nu\ll\mu$. Also, if $\alpha_\nu^* = 0$, the above holds trivially. 
\end{proof}

\subsection{Proof of Lemma \protect\ref{lemma:sup_f_diff_bound}}

\begin{proof}
Let $\alpha^*$ and $\alpha_n^*$ Use Lemma \ref{lemma:sup_f_bound}, 
\begin{align*}
   &\abs{\sup_{\alpha\geq 0} f(\mu, u,\alpha) - \sup_{\alpha\geq 0} f(\mu_n, u,\alpha) } \\
   &= \abs{ \alpha^*_n\log \mu_n[e^{-u/\alpha_n^*}]+\alpha^*_n\delta-\alpha^*\log \mu[e^{-u/\alpha^*}]-\alpha^*\delta}\\
    &= \inf_{\kappa\in\R}\abs{ \alpha_n^*(0)\log \mu_n[e^{-(u-\kappa)/\alpha_n^*(0)}]+\alpha^*_n\delta-\alpha^*\log \mu[e^{-(u-\kappa)/\alpha^*}]-\alpha^*\delta}\\
    &\leq \inf_{\kappa\in\R}\abs{ f(\mu_n,u-\kappa,\alpha_n^*(0))}+\abs{f(\mu,u-\kappa,\alpha^*)}\\
    &\leq 2\inf_{\kappa\in\R}\|u-\kappa\|_{L^\infty(\mu)}\\
    &= 2\spnorm{u}\\
\end{align*}
\end{proof}

\subsection{Proof of Lemma \protect\ref{lemma:equiv_of_meas_cond}}

By definition, on $\Omega_{n,p}(\mu)$, $|\mu_n(y)-\mu(y)|\leq p$. So for all 
$y$ s.t. $\mu(y) > 0$ 
\begin{equation*}
0 < \mu_\wedge -p\leq \mu(y)-p\leq \mu_n(y).
\end{equation*}
Moreover, if $\mu_n(y) = 0$, then $0\geq \mu(y)-p$, we must have that $\mu(y) = 0$. So, $\mu_n\gg \mu$ and hence $\mu_n\sim\mu$.

\subsection{Proof of Lemma \protect\ref{lemma:alpha_ratio_bound}}

\begin{proof}
First we note that for any $\kappa\in\R$
\[
\frac{m[e^{-u/\alpha}]}{\mu[e^{-u/\alpha}]} = \frac{m[e^{-(u-\kappa)/\alpha}]}{\mu[e^{-(u-\kappa)/\alpha}]}. 
\]
Therefore, it suffices to show that for $m = \mu_1-\mu_2$ s.t. $\mu\gg \mu_1,\mu_2$
\[
\sup_{\alpha\geq 0}\frac{\alpha^j m[w]^2}{\mu[w]^2}
\leq  9\|u\|_{L^\infty(\mu_{s,a})}^j \norm{\frac{dm}{d\mu}}_{L^\infty(\mu)}^2.
\]
Fix any $c > 0$, write
\begin{align*}
\sup_{\alpha\geq 0}\frac{\alpha^j m[w]^2}{\mu[w]^2}&= \max\set{\sup_{\alpha\in [0,c \|u\|_\infty]}\frac{\alpha^j m[w]^2}{\mu[w]^2} , \sup_{\alpha\geq c\|u\|_\infty}\frac{\alpha^j  m[w]^2}{\mu[w]^2}}\\
&=:\max\set{J_1(c),J_2(c)}.
\end{align*}
We first bound $J_2(c)$
\[
J_2(c) = \sup_{\alpha\geq c\umax}\frac{\alpha^j m[e^{- (u+\umax)/\alpha}]^2}{\mu_{n}[e^{- (u+\umax)/\alpha}]^2}
\]
For simplicity, let $w':= e^{- (u+\umax)/\alpha}$. Recall that $m = \mu_n - \mu$, so $m[1] = 0$ and
\[
\begin{aligned}
\alpha^j m[e^{- (u+\umax)/\alpha} ]^2 & =  (m[\alpha^{j/2} (e^{- (u+\umax)/\alpha} -1)])^2.\\
\end{aligned}
\]
Define and note that $v := \alpha^{j/2}( e^{- (u+\umax)/\alpha} -1) < 0$. Then
\begin{align*}
\frac{\alpha m[w']^2}{\mu[w']^2} 
&= \frac{m[v]^2}{\mu[w']^2}\\
&= \frac{1}{\mu[w']^2}\mu\sqbk{\frac{dm}{d\mu}v}^2\\
&\leq \frac{\mu\sqbk{-v}^2}{\mu[w']^2}\norm{\frac{dm}{d\mu}}_{L^\infty(\mu)}^2\\
&\leq \norm{\frac{v}{w'}}_{L^\infty(\mu)}^2\norm{\frac{dm}{d\mu}}_{L^\infty(\mu)}^2
\end{align*}
We defer the proof of the following claim:
\begin{lemma}\label{lemma:sup_|v/w|_infty_bound} For any $j\in[0,2]$
\[
\sup_{\alpha\geq c\umax}\norm{\frac{v}{w'}}_{L^\infty(\mu)} \leq (c\umax)^{j/2}(e^{2/c}-1).
\]
\end{lemma}
Therefore, 
\[
J_2(c) \leq c^j\umax^j(e^{2/c}-1)^2\norm{\frac{dm}{d\mu}}_{L^\infty(\mu)}^2.
\]
Choose $c = 2/\log 2$
\begin{align*}
\sup_{\alpha\geq 0}\frac{\alpha m[w]^2}{\mu[w]^2}
&=\max\set{J_1(c),J_2(c)}\\
&\leq \max \set{c^j\umax^j,c^j\umax^j(e^{2/c}-1)^2}\norm{\frac{dm}{d\mu}}_{L^\infty(\mu)}^2\\
&\leq 9\umax^j \norm{\frac{dm}{d\mu}}_{L^\infty(\mu)}^2
\end{align*}
which completes the proof.
\end{proof}

\subsubsection{Proof of Lemma \protect\ref{lemma:sup_|v/w|_infty_bound}}

\begin{proof}
We bound
\[
\begin{aligned}
\norm{\frac{v}{w'}}_{L^\infty(\mu)} &= \esssup_{\mu} \alpha^{j/2}  ( e^{(u(s)+\umax)/\alpha} -1) \\
&\leq  \alpha^{j/2} (e^{2\umax/\alpha} -1)
\end{aligned}
\]

Compute derivative: let $\beta = 2\umax/\alpha$
\begin{align*}
\frac{d}{d\alpha}\alpha^{j/2} (e^{2\umax/\alpha} -1) &= \alpha^{j/2-1}((e^\beta - 1)j/2 - \beta e^\beta)
\end{align*}
Notice that when $\beta = 0$, the above expression is 0. Moreover, for $j\in[0,2]$
\[
\frac{d}{d\beta}((e^\beta - 1)j/2 - \beta e^\beta) = (j/2 -1-\beta)e^\beta < 0;
\]
i.e. $(e^\beta - 1)j/2 - \beta e^\beta$ is decreasing. Therefore, for $\alpha > 0$
\[
\frac{d}{d\alpha}\alpha^{j/2} (e^{2\umax/\alpha} -1) < 0;
\]
i.e. $\alpha^{j/2} (e^{2\umax/\alpha} -1)$ is decreasing in $\alpha$. Hence
\[
\begin{aligned}
\sup_{\alpha\geq c\umax}\norm{\frac{v}{w'}}_{L^\infty(\mu)}&\leq \sup_{\alpha\geq c\umax}\alpha^{j/2} (e^{2\umax/\alpha} -1) \\
&= (c\umax)^{j/2}(e^{2/c}-1)
\end{aligned}
\]
establishing the claim. 
\end{proof}

\subsection{Proof of Lemma \protect\ref{lemma:bias_alpha^3_term}}

\begin{proof}
Let $u' = u+\umax$ and $w' = e^{-u'/\alpha}$. 
\begin{align*}
    &\sup_{\alpha\geq 0}\alpha^3
    \crbk{\frac{m_n[w]}{\mu_n(t)[w]}+ \frac{m_n[uw]}{\alpha\mu_n(t)[w]}- \frac{m_n[w]\mu_n(t)[uw]}{\alpha\mu_n(t)[w]^2} }^2 \\
    &= \sup_{\alpha\geq 0}\alpha^3
    \crbk{\frac{m_n[w']}{\mu_n(t)[w']}+ \frac{m_n[u'w']}{\alpha\mu_n(t)[w']}- \frac{m_n[w']\mu_n(t)[u'w']}{\alpha\mu_n(t)[w']^2} }^2\\
    &\leq 2\sup_{\alpha\geq 0}\alpha^3
    \crbk{\frac{m_n[w']}{\mu_n(t)[w']}+ \frac{m_n[u'w']}{\alpha\mu_n(t)[w']}}^2+ 2\sup_{\alpha\geq 0}\alpha\frac{m_n[w']^2\mu_n(t)[u'w']^2}{\mu_n(t)[w']^4}\\
    &=:2 OPT_1 + 2 OPT_2
\end{align*}
\par We first analyze $OPT_1$. Fix $c\geq 0$, we separately consider $\alpha\geq c\umax $ and $\alpha\in [0, c\umax]$. The first two terms
\begin{align*}
    \sup_{\alpha\geq c\umax}\alpha^3
    \crbk{\frac{m_n[w']}{\mu_n(t)[w']}+\frac{m_n[u'w']}{\alpha\mu_n(t)[w']}}^2
    &= \sup_{\alpha\geq c\umax }\alpha^3\frac{m_n[(1+u'/\alpha)w']^2}{\mu_n(t)[w']^2}\\
    &= \sup_{\alpha\geq c\umax}\frac{m_n[\alpha^{3/2}((1+u'/\alpha)w'-1)]^2}{\mu_n(t)[w']^2}.
\end{align*}
Recall that $1+x\leq e^x$; i.e. $(1+u'/\alpha)w'-1 \leq 0$. Also, by Lemma \ref{lemma:equiv_of_meas_cond}, on $\Omega_{n,p}(\mu)$, $p<\mu_{\wedge}$, $\mu_n(t)\sim \mu_n\sim\mu$. So,
\begin{align*}
    \frac{m_n[\alpha^{3/2}((1+u'/\alpha)w'-1)]^2}{\mu_n(t)[w']^2} 
    &\leq \norm{\frac{\alpha^{3/2}(1-(1+u'/\alpha)w')}{w'}}^2_{L^\infty(\mu)}\norm{\frac{dm_n}{d\mu_n(t)}}_{L^\infty(\mu)}^2\\
    &\leq\norm{\alpha^{3/2}(e^{u'/\alpha}-(1+u'/\alpha))}^2_{L^\infty(\mu)} \norm{\frac{dm_n}{d\mu_n(t)}}_{L^\infty(\mu)}^2
\end{align*}
Recall the Taylor series of $e^x$. For all $s\in\bd{S}$, we have that 
\begin{align*}
\alpha^{3/2}(e^{u'(s)/\alpha}-(1+u'(s)/\alpha)) = \sum_{k=2}^\infty\frac{u'(s)^k}{\alpha^{k-3/2}k!}.
\end{align*}
Notice that $k-3/2 > 0$ for $k\geq 2$ and the terms in the sum are non-negative. So, the above expression suggests that $\alpha\ra \alpha^{3/2}(e^{u'(s)/\alpha}-(1+u'(s)/\alpha))$ is decreasing. 
Therefore, on $\Omega_{n,p}(\mu)$
\[
\sup_{\alpha\geq c\umax}\alpha^3
    \crbk{\frac{m_n[w']}{\mu_n(t)[w']}+\frac{m_n[u'w']}{\alpha\mu_n(t)[w']}}^2\leq c^3\umax^3(e^{2/c}-1)^2\norm{\frac{dm_n}{d\mu_n(t)}}_{L^\infty(\mu)}^2.
\]
Also,
\begin{align*}
    &\sup_{\alpha\in [0,c\umax]}\alpha^3
    \crbk{\frac{m_n[w']}{\mu_n(t)[w']}+\frac{m_n[u'w']}{\alpha\mu_n(t)[w']}}^2 \\
    &\leq 2\sup_{\alpha\in [0,c\umax]}
    \crbk{\frac{\alpha^3m_n[w']^2}{\mu_n(t)[w']^2}+\frac{\alpha m_n[u'w']^2}{\mu_n(t)[w']^2}}\\
    &\leq 2\sup_{\alpha\in [0,c\umax]}
    \crbk{\alpha^3\norm{\frac{dm_n}{d\mu_n(t)}}_{L^\infty(\mu)}^2+\alpha\|u'\|_{L^\infty(\mu)}^2\norm{\frac{dm_n}{d\mu_n(t)}}_{L^\infty(\mu)}^2}\\
    &\leq 2(c^3 +4c)\umax^3
    \norm{\frac{dm_n}{d\mu_n(t)}}_{L^\infty(\mu)}^2
\end{align*}
Choose $c = 2$, we conclude that
\[
OPT_1\leq 32\umax^3
    \norm{\frac{dm_n}{d\mu_n(t)}}_{L^\infty(\mu)}^2
\]
For $OPT_2$, we use Lemma \ref{lemma:alpha_ratio_bound}. 
\begin{align*}
    OPT_2&\leq 9\umax \norm{\frac{dm_n}{d\mu_n(t)}}_{L^\infty(\mu)}^2\frac{\mu_n(t)[u'w']^2}{\mu_n(t)[w']^2}\\
    &\leq 9\umax \norm{\frac{dm_n}{d\mu_n(t)}}_{L^\infty(\mu)}^2\|u'\|_{L^\infty(\mu)}^2\\
    &\leq 36\umax^3 \norm{\frac{dm_n}{d\mu_n(t)}}_{L^\infty(\mu)}^2\\
\end{align*}
Therefore, we conclude that on $\Omega_{n,p}(\mu)$
\[
\sup_{\alpha\geq 0}\alpha^3
    \crbk{\frac{m_n[w]}{\mu_n(t)[w]}+ \frac{m_n[uw]}{\alpha\mu_n(t)[w]}- \frac{m_n[w]\mu_n(t)[uw]}{\alpha\mu_n(t)[w]^2} }^2 \leq 136 \umax^3 \norm{\frac{dm_n}{d\mu_n(t)}}_{L^\infty(\mu)}^2.
\]
The lemma follows from considering $u-\kappa$, which won't change the left hand side. 
\end{proof}

\subsection{Proof of Lemma \protect\ref{lemma:dual_func_derivative}}

\begin{proof}
From \citet{si2020}, it is sufficient to consider $\alpha\in [0,\delta\inv \|u\|_{L^\infty(\mu)}]=:K$. For $\alpha > 0$ fixed, 
\[
\del_t g_n(t,\alpha) = -\alpha\frac{m_n[w]}{\mu_n(t)[w]}.
\]
Also, for $\alpha = 0$, by Lemma \ref{lemma:equiv_of_meas_cond} and $p\leq \frac{1}{4}\mu_\wedge $, $g_n(t,0) \equiv \essinf_\mu u$; hence $\del_t g_n(t,0) \equiv 0$. Again by Lemma \ref{lemma:equiv_of_meas_cond} $\mu_n(t)\sim \mu$ on $\Omega_{n,p}(\mu)$. So, the Radon-Nikodym theorem applies: For fixed $t\in[0,1]$, 
\begin{equation}
\begin{aligned}\label{eqn:continuity_alpha_da_0}
\lim_{\alpha\da 0} \sup_{s\in (t\pm\epsilon)\cap[0,1]}\abs{\del_t g_n(t,\alpha)} &\leq \lim_{\alpha\da 0}\sup_{t\in [0,1]} \alpha\abs{\frac{m_n[w]}{\mu_n(t)[w]}}\\
&= \lim_{\alpha\da 0} \sup_{t\in [0,1]}\alpha\abs{\frac{1}{\mu_n(t)[w]}\mu_n(t)\sqbk{\frac{dm_n}{d\mu_n(t)}w}}\\
&\leq \lim_{\alpha\da 0} \sup_{t\in [0,1]}\alpha \norm{\frac{dm_n}{d\mu_n(t)}}_{L^\infty(\mu)}\\
&\leq \lim_{\alpha\da 0}  \frac{\alpha}{\mu_\wedge - p}\\
&=0.
\end{aligned}
\end{equation}
where we used H\"older's inequality to get the second last line. Therefore, $\del_tg(\cd,\cd)$ is continuous on $[0,1]\times K$. 
\par Next define
\[
\Theta(t):=\argmax{\alpha\in K}g(t,\alpha). 
\]
To simplify notation, we use the $w$ to denote $w =w_n^*(t)= e^{-u/\alpha_n^*(t)}$. We discuss two cases: 
\begin{enumerate}
    \item If $u$ is $\mu$-essentially constant with $\Linfnorm{u}{\mu} = \bar u$, then
\[
\sup_{\alpha\in K} -\alpha\log e^{-\bar u/\alpha}-\alpha\delta =\sup_{\alpha\in K} \bar u-\alpha \delta;
\]
i.e. $\Theta(t) = \set{0}$. 
\item $u$ is not $\mu$-essentially constant. Note that when $\alpha > 0$, $w> 0$; we can define a new measure \[
\mu^*_n(t)[\cd] = \frac{\mu_n(t)[w\cd]}{\mu_n(t)[w]}.
\]
We have that
\begin{align*}
\del_\alpha\del_\alpha g_n(t,\alpha) &= - \frac{\mu_n(t)[u^2w] }{\alpha^{3}\mu_n(t)[w]}+\frac{\mu_n(t)[u
w]^2 }{\alpha^{3}\mu_n(t)[w]^2} \\
&= - \frac{\mu^*_n(t)[u^2] }{\alpha^{3}}+\frac{\mu^*_n(t)[u]^2 }{\alpha^3}\\
&= -\frac{\var_{\mu^*_n(t)}(u)}{\alpha^3}\\
&<0; 
\end{align*}
i.e. $g_n(t,\cd)$ is strictly concave for $\alpha > 0$. Also, recall that $g_n(t,\cd)$ is continuous at $0$. So, in this case either $\Theta(t) =\set{0}$ or $\Theta(t) = \set{\alpha_n^*(t)}$ where $\delta\inv \|u\|_{L^\infty(\mu)}\geq \alpha_n^*(t) > 0 $. 
\end{enumerate}
\par In particular, $\Theta(t)$ is a singleton which we will denote by $\alpha_n^*(t)$ in both cases. We conclude that by \citet{sharpioBookSP} Theorem 7.21, the following derivative exists
\[
d_t\sup_{\alpha\in K}g_n(t,\alpha) = \sup_{\alpha\in \Theta(t)}\del_t g_n(t,\alpha) = \del_t g_n(t,\alpha_n^*(t)).
\]
\par Next, we analyze the second derivative.  We prove that under Assumption \ref{assump:delta_small}, we have that on $\Omega_{n,p}(\mu)$ $\alpha^* = 0$ or $\alpha^* > 0$ will imply that $\alpha_n^*(t) = 0$ or $\alpha_n^*(t) > 0$ respectively. 
\par Let $\rho = \mu(\set{y :u(y) = \essinf_{\mu} u})$ and $\rho_n(t)$ the mixed version. Since $\mu_n\ll\mu$, if $\rho = 1$ (thence $\alpha^* = 0$), then we automatically have that $\rho_n(t)\equiv 1$ and $\alpha_n^*(t) \equiv 0$.
\par Now we consider the case $\rho\neq 1$. Notice that by definition of  $\Omega_{n,p}(\mu)$, $\rho-p\leq\rho_n\leq\rho+p$. There are two cases:
\begin{enumerate}
    \item $\alpha^* = 0$. From \citet{Hu2012KLDRO}, $\alpha^* = 0$ iff $ \rho \geq e^{-\delta}$. If we want $\alpha^*_n(t) = 0$ for all $t\in[0,1]$, a sufficient condition is that $\rho_{n}(t) \geq \rho - p\geq e^{-\delta}$. 
    \item $\alpha^* >0$ iff $ \rho < e^{-\delta}$. If we want $\alpha^*_n(t) > 0$ for all $t\in[0,1]$, a sufficient condition is that $\rho_{n}(t)\leq \rho + p< e^{-\delta}$. 
\end{enumerate}
Therefore, for any $e^{-\delta}\neq\rho \subset\set{\mu(\set{y:u(y)\leq t}):t\in\R}$, we can always choose $p$ small enough s.t. for $\omega\in\Omega_{n,p}(\mu)$, $\rho_n(t)$ is close to $\rho$ for all $t$ and the above sufficient conditions hold. 
\begin{remark}
While this generalizes to all but finitely many $\delta$, for simplicity of presentation, we assume Assumption \ref{assump:delta_small} that $\mu_\wedge /2 \geq 1-e^{-\delta}$. \end{remark} 
So, if $\rho\neq 1$, then $1-\rho \geq \mu_\wedge > 1-e^{-\delta}$; i.e. $\rho < e^{-\delta}$ and case 1 cannot happen. Therefore, $\alpha^* = 0$ iff $u$ is $\mu$ essentially constant. Moreover, by our choice $p\leq  \frac{1}{4}\mu_\wedge $, 
\[
\rho+p\leq 1-\frac{3}{4}\mu_\wedge  < 1-\frac{1}{2}\mu_\wedge \leq e^{-\delta}
\]
satisfying the sufficient condition in case 2.
Hence our assumption on $p$ implies that if $\alpha^* = 0$ or $\alpha^* > 0$, then on $\omega\in\Omega_{n,p}(\mu)$, $\alpha_n^*(t) = 0$ or $\alpha_n^*(t) > 0$ for all $t\in[0,1]$ respectively. 
\begin{enumerate}
    \item $\alpha^* = 0$, then $g_n(t,\alpha_n^*(t)) = g_n(t,0)$ is constant. Hence $d_td_tg_n(t,\alpha_n^*(t)) = 0$. 
    
    \item $\alpha^* >0$, then $\alpha^*_n(t_1), \alpha^*_n(t_2)> 0$. Since $g_n(t,\cd)$ is strictly convex, $\alpha^*_n(t)$ is the unique solution to the first order optimality condition
    \begin{equation}
    0 = \del_\alpha g_n(t,\alpha_n^*(t)) = -\log\mu_n(t)[w] -\delta - \frac{\mu_n(t)[uw]}{\alpha_n^*(t)\mu_n(t)[w]}. \label{eqn:alpha_opt_cond}
    \end{equation}
    Note that $\del_\alpha g_n\in C^\infty([0,1]\times \R_{++})$ and that $\del_\alpha\del_\alpha g_n(t,\alpha_n^*(t)) <0$. The implicit function theorem  implies that $\alpha_n^*(t)\in C^1((0,1))$ with derivative
    \begin{align*}
    d_t\alpha_n^*(t) &= -\frac{\del_t\del_\alpha g_n(t,\alpha_n^*(t))}{\del_\alpha\del_\alpha g_n(t,\alpha_n^*(t))}\\
    &=\crbk{\frac{\alpha_n^*(t)^3}{\var_{\mu^*_n(t)}(u)}}\crbk{-\frac{m_n[w]}{\mu_n(t)[w]}  + \frac{\mu_n(t)[uw] m_n[w]}{\alpha_n^*(t)\mu_n(t)[w]^2} - \frac{m_n[uw]}{\alpha_n^*(t)\mu_n(t)[w]}}
    \end{align*}
    We conclude that 
    \[ 
    \del_t g_n(t,\alpha_n^*(t)) = -\alpha^*_n(t)\frac{m_n[w]}{\mu_n(t)[w]}
    \]
    is $C^1((0,1))$ as a function of $t$. Therefore, $g_n(t,\alpha_n^*(t))$ is $C^2((0,1))$ with derivative 
    \begin{align*}
    &d_td_tg_n(t,\alpha_n^*(t))\\
    &=d_t\del_t g_n(t,\alpha_n^*(t)) \\
    &= -\alpha_n^*(t)\frac{m_{n}[w]^2}{\mu_{n}(t)[w]^2} + d_t\alpha_n^*(t)\crbk{\frac{m_n[w]}{\mu_n(t)[w]}+ \frac{m_n[uw]}{\alpha_n^*(t)\mu_n(t)[w]}- \frac{m_n[w]\mu_n(t)[uw]}{\alpha_n^*(t)\mu_n(t)[w]^2}}\\
    &= -\alpha_n^*(t)\frac{m_{n}[w]^2}{\mu_{n}(t)[w]^2}  - \crbk{\frac{\alpha_n^*(t)^3}{\var_{\mu^*_n(t)}(u)}}\crbk{\frac{m_n[w]}{\mu_n(t)[w]}+ \frac{m_n[uw]}{\alpha_n^*(t)\mu_n(t)[w]}- \frac{m_n[w]\mu_n(t)[uw]}{\alpha_n^*(t)\mu_n(t)[w]^2}}^2. 
    \end{align*}
\end{enumerate}
Therefore, Lemma \ref{lemma:dual_func_derivative} summarizes these two cases. 
\end{proof}

\subsection{Proof of Lemma \protect\ref{lemma:worst_case_meas_bound}}

\begin{proof}
\par First, we note that if $\mu_n^*(t)(y) \geq \mu_n(t)(y)> 0$, then by Lemma \ref{lemma:equiv_of_meas_cond}, $\mu_n^*(t)(y) \geq \mu_n(t)(y)\geq \frac{3}{4}\mu_\wedge$. So, we will only consider cases where $\mu_n^*(t)(y) < \mu_n(t)(y)$. We now fix any such $y$. 
\par By Lemma \ref{lemma:dual_func_derivative}, under the given assumptions $\alpha^* > 0$ implies that $\alpha_n^*(t) > 0$. So, the KL constraint is binding; i.e. $\delta = D_{\mathrm{KL}}(\mu_n^*(t)||\mu_n(t))$. By the log-sum inequality, 
\[
\delta = D_{\mathrm{KL}}(\mu_n^*(t)||\mu_n(t))\geq \mu_n^*(t)(y)\log\crbk{\frac{\mu_n^*(t)(y)}{\mu_n(t)(y)}}+(1-\mu_n^*(t))(y)\log\crbk{\frac{1-\mu_n^*(t)(y)}{1-\mu_n(t)(y)}}
\]
Define 
\[
kl(q,b) = q\log\crbk{\frac{q}{b}}+(1-q)\log\crbk{\frac{1-q}{1-b}}.
\]
where we think of $b = \mu_n(t)(y)$. Observe that for $q\in(0,b)$
\[
\del_q\del_qkl(q,b) = \frac{1}{q}+\frac{1}{1-q} > 0;
\]
i.e. $kl(\cd,b)$ is strictly convex and the maximum is achieved at $q=0$, $kl(0,b) = \log (1/(1-b))$. Since $b\in[\frac{3}{4}\mu_\wedge,1-\frac{3}{4}\mu_\wedge]$, we have that $\log(1/(1-b))\geq \log(1/(1-\frac{3}{4}\mu_\wedge))>\frac{3}{4}\mu_\wedge> \delta$. So, by the convexity, continuity of $kl(\cd,b)$ and $kl(b,b) = 0$, there is unique $q^*\in(0,b)$ s.t. $kl(q^*,b) =\delta$. Now we bound such $q^*$. 
\par Since $d_qkl(q,b) <0$ for $q < b$, by the fundamental theorem of calculus and convexity
\begin{align*}
kl(q,b) &= -\int_q^b \del_xkl(x,b)dx\\
&= \int_q^b\log\crbk{\frac{1-x}{1-b}}-\log\crbk{\frac{x}{b}}dx \\
&\geq \int_q^b \crbk{\frac{1}{b}+\frac{1}{1-b}} (x-b)  dx\\
&=\frac{(b-q)^2}{2b(1-b)}\\
&=:\zeta(q,b)
\end{align*}
Note that for $q < b$
\[
d_b\zeta(q,b) = \frac{(b-q)(q+b-2qb)}{2(1-b)^2b^2} > 0
\]
i.e. $\zeta(q,\cd)$ is increasing. Suppose to the contrary $q^*<\frac{1}{2}\mu_\wedge $, then
\begin{align*}
kl(q^*,b)
&\geq \zeta(q^*,b)\\
&\geq \inf_{b\in[\frac{3}{4}\mu_\wedge ,1-\frac{3}{4}\mu_\wedge ]} \zeta(q^*,b) \\
&= \zeta\crbk{q^*,\frac{3}{4}\mu_\wedge }\\
&> \frac{1}{24}\mu_\wedge .
\end{align*}
However, by assumption, $\frac{1}{24}\mu_\wedge \geq \delta\geq kl(q^*,\mu_n(t)(y)) > \mu_\wedge $. Hence $q^*\geq \frac{1}{2}\mu_\wedge $. We conclude that $\mu_n^*(t)\geq \frac{1}{2}\mu_\wedge $. 
\end{proof}

\section[The Empirical Robust Bellman Operator: chi2 Case]{The Empirical Robust Bellman Operator: $\chi_2$ Case}
To analyze the variance-reduced Q-learning for the $\chi_2$ case, we establish important statistical properties of the empirical DR ellman operator  $\widehat{\mathbf{T}}$ and its recentered version $\widehat{\bd{H}}$. We defer the proofs to Appendix \ref{a_sec:proof:props_properties_of_empirical_opt_chi2}. The proof techniques are similar to that in Appendix \ref{a_sec:proof_empirical_bellam_kl}.  

We let $\widehat{\mathbf{T}}$ be the empirical DR Bellman operator formed by $n$ samples defined in  \eqref{eqn:empirical_Bellman_operator_chi2}. Define the recentered operators $\widehat{\mathbf{H}},\cH$ as in \eqref{eqn:def_recentered_op}. We fix $\hat q\in\R^{\bd{S}\times \bd{A}}$.

\begin{proposition}\label{prop:recentered_op_bias_var_bound_chi2}
Suppose Assumption \ref{assump:delta_small_chi2} is enforced.
Then
\[
|E[\widehat{\mathbf{H}}(\hat q)(s,a)-\cH(\hat q)(s,a)]|\leq  \frac{2^{6}\|\hat q - q_*\|_\infty}{ \fr{p}_{\wedge}\sqrt{n}}\log(e|\bd{S}|),
\]
provided $n\geq \fr{p}_{\wedge}^{-2}$, and
\[
\var(\widehat{\mathbf{H}}(q^*))(s,a)
\leq \frac{2^{11}\|\hat q - q_*\|_\infty^2}{\fr{p}_{\wedge}^2n}\log(e|\bd{S}|)
\]
for all $n\geq 1$. 
\end{proposition}

\begin{proposition}
\label{prop:recentered_hp_bound_chi2}
Assume Assumption \ref{assump:delta_small_chi2}. Then w.p. at least $1-\eta$
\[
\|\cH(\hat q) - \widehat{\mathbf{H}}(\hat q)\|_\infty \leq \frac{6\|\hat q-q^*\|_\infty}{\fr{p}_{\wedge}\sqrt{n}}\sqrt{\log(4|\bd{S}|^2|\bd{A}|/\eta)}
\]
provided that $n\geq 8\fr{p}_{\wedge}^{-2}\log(4|\bd{S}|^2|\bd{A}|/\eta)$
\end{proposition}

\begin{proposition}\label{prop:high_prob_bound_empirical_Bellman_chi2}
The empirical DR Bellman operator 
\[
\|\widehat{\mathbf{T}}(q)-\cT(q)\|_\infty \leq \frac{8(\rmax+\gamma\norminf{q})}{\fr{p}_{\wedge}\sqrt{n}}\sqrt{\log\crbk{6|\bd{S}||\bd{A}|(|\bd{S}|\vee |\bd{R}|)/\eta}}
\]
w.p. at least $1 - \eta$, provided that $n\geq 8\fr{p}_{\wedge}^{-2}\log\crbk{12|\bd{S}||\bd{A}|(|\bd{S}|\vee |\bd{R}|)/\eta}$. 
\end{proposition}

\section[Analysis of the Variance-Reduced Q-Learning: chi2 Case]{Analysis of the Variance-Reduced Q-Learning: $\chi_2$ Case}\label{a_sec:analysis_of_vrql_alg_chi2}

We proceed with the analysis of the variance-reduced DR Q-learning Algorithm \ref{alg:vr_q-learning} in the $\chi_2$ divergence case, similar to the KL case. Specifically, we aim to show that if the $q$-function from the last variance-reduced algorithm epoch, $\hat{q}_{l-1}$, is within a certain error $b$ of the optimal $q^*$, then $\hat{q}_{l}$ will have a better concentration bound by a geometric factor. This is summarized in Proposition \ref{prop:one_vr_iter_high_prob_bd_chi2}, which is analogous to Proposition \ref{prop:one_vr_iter_high_prob_bd} in the KL case.

\par Recall that $\cF_{l}$ denotes the $\sigma$-field generated by the random samples used until the end of epoch $l$. We define the conditional expectation $E_{l-1}[\cd] = E[\cd|\cF_{l-1}]$. 
\begin{proposition}\label{prop:one_vr_iter_high_prob_bd_chi2}
Assuming that Assumptions \ref{assump:max_rwd} and \ref{assump:delta_small_chi2} are satisfied. On $\set{\omega:\|\hat q_{l-1}-q^*\|_\infty\leq b}$ for some  $b\leq 1/(1-\gamma)$, under measure $P_{l-1}(\cd) := E_{l-1}[\1\set{\cd}]$, we have that there exists numerical constant $c$ s.t.
\begin{align*}
\|\hat q_{l} - q^*\|_\infty &\leq c\crbk{\frac{b}{(1-\gamma)^2k_\mathrm{vr} } + \frac{b}{\fr{p}_\wedge (1-\gamma)^{3/2} \sqrt{n_{\mrm{vr}}k_{\mrm{vr}}}} +  \frac{b}{\fr{p}_\wedge (1-\gamma) \sqrt{n_{\mrm{vr}}}} }\log\crbk{3dk_\mathrm{vr}/\eta}^2   \\
&\quad +c\frac{1}{\fr{p}_{\wedge} (1-\gamma)^2\sqrt{m_l}}\sqrt{\log\crbk{3 d/\eta}}
\end{align*}
w.p. at least $1-\eta$, provided that $m_l\geq 8\fr{p}_{\wedge}^{-2}\log(24d/\eta)$ and $n_\mrm{vr}\geq \fr{p}_\wedge\inv$. 
\end{proposition}
\begin{proof}[Proof of Proposition \ref{prop:one_vr_iter_high_prob_bd_chi2}]
We recall the proof of Proposition \ref{prop:one_vr_iter_high_prob_bd} in Appendix \ref{a_sec:proof:prop:one_vr_iter_high_prob_bd}. We have that by \eqref{eqn:q_inner_prob_1_bd}, under $P_{l-1}$, on $\set{\omega:\|\hat q_{l-1}-q^*\|_\infty\leq b}$ 
\begin{equation}\label{eqn:q_inner_prob_1_bd_chi2}
   \|q_{l,k+1} - q^*\|_\infty \leq \lambda_k\sqbk{2b + \gamma \sum_{j=1}^k\|Q_{l,j}\|_\infty }+ \|Q_{l,k+1}\|_\infty + \frac{2\|D_l\|_\infty}{1-\gamma}
\end{equation}
w.p.1. The sequence $\set{Q_{l,j}:j=1,\ds, k+1}$, by \eqref{eqn:Q_sum_high_prob_bd}, satisfies
\begin{align*}
    &\gamma \lambda_{k_\mathrm{vr}}\sum_{j=1}^{k_\mathrm{vr}}\|Q_{l,j}\|_\infty + \|Q_{l,k_\mathrm{vr}+1}\|_\infty \\
    &\leq 8\crbk{\frac{\lambda_{k_\mathrm{vr}}\log(e+(1-\gamma)k_\mathrm{vr})\|\zeta_{l-1}\|_\infty}{1-\gamma} + \frac{\|\sigma_{l-1}\|_\infty\sqrt{\lambda_{k_\mathrm{vr}}}}{1-\gamma}}\log\crbk{4|\bd{S}||\bd{A}|k_\mathrm{vr}/\eta}
\end{align*}
w.p. at least $1-\eta$, where we recall that
\begin{align*}\norminf{\zeta_{l-1}} &= \norm{\hat q_{l-1}-q^*},  \\
\norminf{\sigma_{l-1}^2} &= \max_{(s,a)\in \bd{S}\times \bd{A}}\var_{l-1}(\mathbf{H}_{l,k}(\hat q_{l-1})(s,a)).
\end{align*}
Therefore, by Proposition \ref{prop:recentered_op_bias_var_bound_chi2}, we have that 
\begin{align*}
&\gamma \lambda_{k_\mathrm{vr}}\sum_{j=1}^{k_\mathrm{vr}}\|Q_{l,j}\|_\infty + \|Q_{l,k_\mathrm{vr}+1}\|_\infty\\ &\leq c\crbk{\frac{b}{(1-\gamma)^2k_\mathrm{vr} } + \frac{b}{\fr{p}_\wedge (1-\gamma)^{3/2} \sqrt{n_{\mrm{vr}}k_{\mrm{vr}}}}}\log\crbk{4|\bd{S}||\bd{A}|k_\mathrm{vr}/\eta}^2  
\end{align*}
for some constant $c$. 
\par Moreover, recall the definition of $D_l$ in \eqref{eqn:def_D_l}. By Propositions \ref{prop:recentered_op_bias_var_bound_chi2}, \ref{prop:recentered_hp_bound_chi2}, and \ref{prop:high_prob_bound_empirical_Bellman_chi2}, we have that 
\begin{align*}
    \|D_l\|_\infty 
&\leq c\frac{\rmax+\spnorm{q^*} + \|\hat q_{l-1}-q^*\|_\infty}{\fr{p}_{\wedge} \sqrt{m_l}}\sqrt{\log\crbk{12d/\eta}}+  c\frac{\|\hat q_{l-1}-q^*\|_\infty}{\fr{p}_{\wedge}\sqrt{n_\mathrm{vr}}}\sqrt
{\log(e|\bd{S}|)}\\
&\leq c\frac{1}{\fr{p}_{\wedge} (1-\gamma)\sqrt{m_l}}\sqrt{\log\crbk{12d/\eta}}+  c\frac{b}{\fr{p}_{\wedge}\sqrt{n_\mathrm{vr}}}\sqrt
{\log(e|\bd{S}|)}
\end{align*}
for some constant $c$ that can change from line to line. 
\par Combining these bound with \eqref{eqn:q_inner_prob_1_bd_chi2} and apply union bound, we conclude that 
\begin{align*}
\|q_{l,k_{\mrm{vr}}+1} - q^*\|_\infty &\leq c\crbk{\frac{b}{(1-\gamma)^2k_\mathrm{vr} } + \frac{b}{\fr{p}_\wedge (1-\gamma)^{3/2} \sqrt{n_{\mrm{vr}}k_{\mrm{vr}}}} +  \frac{b}{\fr{p}_\wedge (1-\gamma) \sqrt{n_{\mrm{vr}}}} }\log\crbk{8dk_\mathrm{vr}/\eta}^2   \\
&\quad +c\frac{1}{\fr{p}_{\wedge} (1-\gamma)^2\sqrt{m_l}}\sqrt{\log\crbk{24 d/\eta}}
\end{align*}
w.p. at least $1-\eta$. Recall  the definition in Algorithm \ref{alg:vr_q-learning} that $q_{l,k_{\mrm{vr}}+1} = \hat q_{l}$.

\par Finally, we adjust the constant in the log factor using the inequality for $C_1\geq 1,C_2\geq e$, $\log(C_1C_2)= \log(C_1)+\log(C_2)\leq C_1\log(C_2)$. This completes the proof. 
\end{proof}

Given Proposition \ref{prop:vr_algo_err_high_prob_bd_chi2}, we apply the analysis techniques for the variance-reduction iterates in the proof of \ref{prop:vr_algo_err_high_prob_bd_chi2}. This yields the following Proposition. 
\begin{proposition}
  \label{prop:vr_algo_err_high_prob_bd_chi2}
Assume Assumptions \ref{assump:max_rwd} and \ref{assump:delta_small_chi2}.  For $\epsilon < (1-\gamma)^{-1} $, define parameters according to\eqref{eqn:param_choice_for_vrql_chi2}. Then, the statement of Proposition \ref{prop:vr_algo_err_high_prob_bd} hold; i.e. the sequence $\set{\hat q_l,0\leq l\leq l_\mathrm{vr}}$ produced by Algorithm \ref{alg:vr_q-learning} satisfies the pathwise property that $\|\hat q_l-q^*\|_\infty\leq 2^{-l}(1-\gamma)^{-1}$ for all $0\leq l\leq l_\mathrm{vr}$ w.p. at least $1-\eta$. In particular, the final estimator $\hat q_{l_\mathrm{vr}}$ satisfies $\|\hat q_{l_\mathrm{vr}}-q^*\|_\infty\leq 2^{-l_\mathrm{vr}}(1-\gamma)^{-1}$ w.p. at least $1-\eta$.   
\end{proposition}
\begin{proof}[Proof of Proposition \ref{prop:vr_algo_err_high_prob_bd_chi2}]
Follow the proof of Proposition \ref{prop:vr_algo_err_high_prob_bd}, we only to validate \eqref{eqn:one_vr_iter_geom_converge_kl} given the parameter choice in \eqref{eqn:param_choice_for_vrql_chi2}. By Proposition \ref{prop:one_vr_iter_high_prob_bd_chi2}, conditioned on $\|\hat q_{l-1} - q^*\|_\infty\leq 2^{-(l-1)}(1-\gamma)^{-1} =: b$
 \begin{align*}
\|\hat q_{l} - q^*\|_\infty&\leq c\crbk{\frac{b}{(1-\gamma)^2k_\mathrm{vr}} +\frac{b}{\fr{p}_{\wedge}(1-\gamma)^{3/2}\sqrt{n_\mathrm{vr}k_\mathrm{vr}}} + 
\frac{b}{\fr{p}_{\wedge}(1-\gamma)\sqrt{n_\mathrm{vr}}} }\log\crbk{3dk_\mathrm{vr}/\eta}^2 \\
&\quad +  c\frac{1}{\fr{p}_{\wedge} (1-\gamma)^2\sqrt{m_l}}\sqrt{\log(3d/\eta)}
\end{align*}
w.p. at least $1-\eta$. 
\par Therefore, it is easy to see that by the parameter choice \eqref{eqn:param_choice_for_vrql_chi2}, we have that for sufficiently large $c_{\mrm{vr}}$ and for events $\omega\in \set{\|\hat q_{l-1} - q^*\|_\infty\leq 2^{-(l-1)}(1-\gamma)^{-1}}$,
\[
	P_{l-1}\crbk{\1\set{\|\hat q_{l} - q_*\|_\infty\leq 2^{-l}(1-\gamma)^{-1} }}(\omega) \geq 1-\frac{\eta}{l_\mathrm{vr}};
\]
validating \eqref{eqn:one_vr_iter_geom_converge_kl}. Following the same arguments as in proof of Proposition \ref{prop:vr_algo_err_high_prob_bd} will yield Proposition \ref{prop:vr_algo_err_high_prob_bd_chi2}. 
\end{proof}

Now, we prove Theorem \ref{thm:vrql_sample_complexity_chi2}. 
\begin{proof}[Proof of Theorem \ref{thm:vrql_sample_complexity_chi2}]
    By Proposition \ref{prop:vr_algo_err_high_prob_bd_chi2}, under the parameter choice \eqref{eqn:param_choice_for_vrql_chi2}, $\|\hat q_{l_\mathrm{vr}}-q^*\|_\infty\leq \epsilon$ w.p. at least $1-\eta$. The total number of samples used is \begin{align*}
     |\bd{S}||\bd{A}|\crbk{l_\mathrm{vr}n_\mathrm{vr}k_\mathrm{vr} + \sum_{l=1}^{l_\mathrm{vr}}m_l}= \tilde O\crbk{|\bd{S}||\bd{A}|\crbk{\frac{1}{\fr{p}_{\wedge}^2(1-\gamma)^4} + \frac{4^{l_\mathrm{vr}}}{\fr{p}_{\wedge}^2 (1-\gamma)^2}}}. 
    \end{align*}
    This yields the sample complexity bound in Theorem \ref{thm:vrql_sample_complexity_chi2}. 
\end{proof}

\section[Proofs of Properties of the Empirical Bellman Operator: chi2 Case]{Proofs of Properties of the Empirical Bellman Operator: $\chi_2$ Case}\label{a_sec:proof:props_properties_of_empirical_opt_chi2}
\label{a_sec:proof_empirical_bellam_chi2}
We first define some notations that mimic the definitions in Appendix \ref{a_sec:proof_empirical_bellam_kl}. Again, we override the notations for the KL case. For generic
probability measure $\mu$ on $(Y,2^Y)$ and random variable $u:Y\ra \R$, let $w =(\alpha - u)_+$; define the $\chi_2$ dual
functional under the reference measure $\mu$ as
\begin{equation}  \label{eqn:dual_functional_chi2}
f(\mu,u,\alpha) := \alpha - c(\delta)\mu[w^2]^{\frac{1}{2}}.
\end{equation}
Recall the dual formulation of the DR Bellman operator \eqref{eqn:def_bellman_op_chi2_dual}, we have that
\begin{equation}\label{eqn:def_dual_func_bellman_op_chi2}
\cT(q)(s,a) = \sup_{\beta\in\R}f(\nu_{s,a},id,\beta)+\gamma \sup_{\alpha\in\R}f(p_{s,a},v(q),\alpha).
\end{equation}
Next, we present two important lemmas that underlie our analysis of the DR Bellman operator in the $\chi_2$ case. First, we characterize the optimal Lagrange multiplier in the dual formulation \eqref{eqn:def_bellman_op_chi2_dual}. 
\begin{lemma}\label{lemma:alpha_opt_equation_chi2}
For $\delta > 0$, $f(\mu,u,\alpha)$ is second continuously differentiable and concave for $\alpha>\essinf_\mu u$. The supremum is achieved at $\essinf_\mu u\leq \alpha^* < \infty$, i.e.
$\sup_{\alpha\in\R}f(\mu,u,\alpha) = f(\mu,u,\alpha^*) $, satisfying \begin{equation}\label{eqn:opt_cond_chi2}
    \mu[w^2] = c(\delta)^2\mu[w]^2. 
\end{equation}
\par Moreover, if $\alpha > \essinf_\mu u$, then 
\begin{equation}\label{eqn:worst_case_meas_chi2}
    \mu^*(\cd) :=\frac{\mu[w\1\set{\cd}]}{\mu[w]}. 
\end{equation} 
is a worst-case measure satisfying 
$$\mu^*[u] = f(\mu,u,\alpha^*) = \inf_{\mu':D_{\chi_2}(\mu'\|\mu)\leq \delta}\mu'[u] = \inf_{\mu':D_{\chi_2}(\mu'\|\mu)= \delta}\mu'[u]; $$
i.e. the $\chi_2$ constraint is active. 
\par Finally, if $\alpha^* = \essinf_\mu u$, then the measure 
\begin{equation}\label{eqn:worst_case_meas_alpha_min_chi2}
    \mu^*(\cd) := \frac{\mu[\1\set{U\cap \cd}]}{\mu(U)}
\end{equation} where $U:= \set{s:\mu(s) > 0, u(s) = \essinf_\mu u}$ is a worst-case measure. 
\end{lemma}

With this lemma, we can show that under Assumption \ref{assump:delta_small_chi2}, the optimal Lagrange multiplier $\alpha^*$ is sufficiently large so that $w = (\alpha^*-v)_+ = \alpha^*-v$ a.s.$\mu$.
\begin{lemma}\label{lemma:alpha_large_chi2}
If $\delta < \frac{1}{2}\mu_\wedge:= \min_{s:\mu(s)>0}\mu(s)$, then $\alpha^*\geq \esssup_\mu u$.  Moreover, if $u$ is not $\mu$ essentially constant, then $\alpha^*> \esssup_\mu u$. 
\end{lemma}
The proofs of these Lemmas are deferred to Appendix \ref{a_sec:proof:lemma_chi2}.

\subsection{Proof of Proposition \protect\ref{prop:recentered_op_bias_var_bound_chi2}}
As in Appendix \ref{a_sec:proof:prop:recentered_op_bias_and_var_KL}, call $V:=\cH(\hat q) -\widehat{\mathbf{H}}(\hat q) = (\cT(\hat q) - \cT(q_*)) - (\widehat{\mathbf{T}}(\hat q) - \widehat{\mathbf{T}}(q_*))$.  
\par Recall the following notations in Appendix \ref{a_sec:proof:prop:recentered_op_bias_and_var_KL}: $v_t = t v(\hat q) +(1-t)v(q_*)$, $\mu =p$, $m = p-p_n$, and $\mu(t) = tp - (1-t)p_n$. Let 
$$h(s,t):= \sup_{\alpha\in\R} f(\mu(t),v_s,\alpha).$$ 
\par We consider $\Omega_{n,p}(\mu)$ with $p\leq\frac{1}{4}\mu_\wedge$. Then, by Lemma \ref{lemma:equiv_of_meas_cond}, we have that $\mu\sim\mu_n\sim\mu(t)$ on $\Omega_{n,p}(\mu)$. Also, recall that $\alpha^*_{s,t}$ is the optimal Lagrange multiplier that satisfies the conclusions of Lemma \ref{lemma:alpha_opt_equation_chi2}. 
\par First we note that if $v(\hat q)$ and $v( q^*)$ are both $\mu$ essentially constant, then $V = 0$, and the claim of Proposition \ref{prop:recentered_op_bias_var_bound_chi2} holds trivially. Moving forward, we consider the case at least one of $v(\hat q)$ and $v( q^*)$ is not $\mu$ essentially constant.
\par We proceed to show the differentiability of $h$ in this setting. This is summarized by Lemma \ref{lemma:V_integral_rep_chi2}. The proof of this result is deferred to Appendix \ref{a_sec:proof:lemma_chi2}. 
\par Note that Assumption \ref{assump:delta_small_chi2} implies that $\delta < \frac{1}{2}\mu_\wedge$. 
\begin{lemma}\label{lemma:V_integral_rep_chi2}
Suppose $\delta < \frac{1}{2}\mu_\wedge$ and $p\leq \frac{1}{4}\mu_\wedge $. If at least one of $v(\hat q)$ and $v( q^*)$ is not $\mu$ essentially constant, then on $\Omega_{n,p}(\mu)$ there exists function $s,t\ra D^2 h(s,t)$ s.t. 
\begin{equation}\label{eqn:V_mixed_partials_bd_chi2}
\begin{aligned}
    |V(s,a)| 
    &\leq \gamma\int_0^1\int_0^1 \abs{D^2 h(s,t)}ds dt 
\end{aligned}
\end{equation}
w.p.1, where 
\begin{equation}\label{eqn:del_s_del_t_hst_decomp_chi2}
\begin{aligned}
D^2 h(s,t) 
&=  \sqbk{\frac{\mu(t)[\Delta_vw_s]m[w_s^2]}{2\mu(t)[w_s]\mu(t)[w_s^2]} -  \frac{m[\Delta_vw_s]}{\mu(t)[w_s]}} +\del_s\alpha^*_{s,t}\crbk{ \frac{m[w_s^2]}{2\mu(t)[w_s^2]}- \frac{m[w_s]}{\mu(t)[w_s]}} \\
&=: D_1+D_2
\end{aligned}
\end{equation}
with
\begin{equation}\label{eqn:alpha*_derivative_chi2}
\begin{aligned}
\del_s\alpha^*_{s,t} &= \frac{c(\delta)^2\mu(t)[w_s]\mu(t)[\Delta_v] - \mu(t)[w_s\Delta_v]}{(c(\delta)^2 - 1)\mu(t)[w_s]}
\end{aligned}
\end{equation}

\end{lemma}

\par We analyze the two terms separately. Recall that $w_s\geq 0$. Similar to the techniques in Appendix \ref{a_sec:proof:prop:recentered_op_bias_and_var_KL}, we have that on $\Omega_{n,p}(\mu)$ with $\mu\sim \mu_n\sim\mu(t)$, 
$$\abs{\frac{\mu(t)[\Delta_vw_s]m[w_s^2]}{2\mu(t)[w_s]\mu(t)[w_s^2]}} \leq \abs{\frac{\norm{\Delta_v}_\infty m[w_s^2]}{2\mu(t)[w_s^2]}} \leq \frac{1}{2}\norm{\Delta_v}_\infty\Linfnorm{\frac{dm}{d\mu(t)}}{\mu}$$
and 
$$\frac{m[\Delta_vw_s]}{\mu(t)[w_s]}\leq\norm{\Delta_v}_\infty\Linfnorm{\frac{dm}{d\mu(t)}}{\mu}. $$
Hence on $\Omega_{n,p}(\mu)$,
$$\abs{D_1}\leq \frac{3}{2}\norm{\Delta_v}_\infty\Linfnorm{\frac{dm}{d\mu(t)}}{\mu}. $$
For $D_2$, we note that 
$$
\begin{aligned}
\abs{\del_s\alpha^*_{s,t} }&= \abs{\frac{c(\delta)^2\mu(t)[w_s]\mu(t)[\Delta_v] - \mu(t)[w_s\Delta_v]}{(c(\delta)^2 - 1)\mu(t)[w_s]}}\\
&= \frac{1}{c(\delta)^2 - 1}\abs{c(\delta)^2\mu(t)[\Delta_v] - \frac{\mu(t)[w_s\Delta_v]}{\mu(t)[w_s]}}\\
&\leq \frac{c(\delta)^2 + 1}{c(\delta)^2 - 1}\norm{\Delta_v}
\end{aligned}
$$
Next, we consider
 $$\begin{aligned}\frac{m[w_s^{2}]}{2\mu(t)[w_s^{2}]}- \frac{m[w_s]}{\mu(t)[w_s]}
 &=\frac{m[(w_s-\mu(t)[w_s])^2] + 2m[w_s]\mu(t)[w_s]}{2\mu(t)[w_s^{2}]}- \frac{m[w_s]}{\mu(t)[w_s]}\\
 &= \frac{m[(w_s-\mu(t)[w_s])^2] }{2c(\delta)^2\mu(t)[w_s]^2}- \frac{(c(\delta)^2-1)m[w_s]}{c(\delta)^2\mu(t)[w_s]}\end{aligned}$$
 where we use the optimality condition \eqref{eqn:alpha_opt_cond} to replace $\mu(t)[w_s^{2}]$ with $c(\delta)^2\mu(t)[w_s]^2$. Then, 
 $$\begin{aligned}\abs{m[(w_s-\mu(t)[w_s])^2]}&=\abs{\mu(t)\sqbk{\frac{dm}{d\mu(t)}(w_s-\mu(t)[w_s])^2}}\\
 &\leq\mu(t)\sqbk{(w_s-\mu(t)[w_s])^2}\norm{\frac{dm}{d\mu(t)}}_{L^\infty(\mu)}\\
 &= \crbk{\mu(t)[w_s^2] - \mu(t)[w_s]^2}\norm{\frac{dm}{d\mu(t)}}_{L^\infty(\mu)}\\
 &= (c(\delta)^2-1)\mu(t)[w_s]^2\norm{\frac{dm}{d\mu(t)}}_{L^\infty(\mu)}\end{aligned}$$
where we also apply \eqref{eqn:alpha_opt_cond} and $\mu(t)\sim\mu$. So, 
$$
\begin{aligned}\abs{\frac{m[w_s^{2}]}{2\mu(t)[w_s^{2}]}- \frac{m[w_s]}{\mu(t)[w_s]}}&\leq \abs{\frac{m[(w_s-\mu(t)[w_s])^2] }{2c(\delta)^2\mu(t)[w_s]^2}}+\abs{ \frac{(c(\delta)^2-1)m[w_s]}{c(\delta)^2\mu(t)[w_s]}}\\
&\leq \frac{3}{2}\frac{c(\delta)^2-1}{c(\delta)^2}\norm{\frac{dm}{d\mu(t)}}_{L^\infty(\mu)}.
\end{aligned}$$
Therefore, we have that $$\begin{aligned}\abs{D_2} &= \abs{\del_s\alpha^*_{s,t}}\abs{\frac{m[w_s^{2}]}{2\mu(t)[w_s^{2}]}- \frac{m[w_s]}{\mu(t)[w_s]}}\\
&\leq \frac{3(c^2+1)}{2c^2}\norminf{\Delta_v}\norm{\frac{dm}{d\mu(t)}}_{L^\infty(\mu)}\\
&\leq3\norminf{\Delta_v}\norm{\frac{dm}{d\mu(t)}}_{L^\infty(\mu)}\end{aligned}$$
as $c(\delta)^2 = 1+2\delta \geq 1$. 

\par So, on $\Omega_{n,p}(\mu)$, $$\abs{\del_s\del_t h(s,t) }\leq |D_1|+|D_2|\leq \frac{9}{2}\norm{\Delta_v}_\infty\norm{\frac{dm}{d\mu}}_{L^\infty(\mu)}.$$
Recall \eqref{eqn:EV_on_Omega_c_bd_kl} and \eqref{eqn:EV_separate_and_bd_kl}, we have that
\begin{align*}
E|V|&\leq \frac{2\gamma\|\hat q - q_*\|_\infty}{p^2n}\log(e|\bd{S}|)+ 5\gamma\|\Delta_v\|_\infty\sup_{s,t\in(0,1)} E\norm{\frac{dm}{d\mu(t)}}_{L^\infty(\mu)}\1_{\Omega_{n,p}(\mu)} \\
&\leq  \frac{2^{5}\|\hat q - q_*\|_\infty}{\mu_\wedge^2n}\log(e|\bd{S}|) + \frac{5\|\hat q - q_*\|_\infty}{p\sqrt{n}}\sqrt{\log(e|\bd{S}|)}\\
&\leq \frac{2^{5}\|\hat q - q_*\|_\infty}{\mu_\wedge^2n}\log(e|\bd{S}|) + \frac{20\|\hat q - q_*\|_\infty}{ \mu_\wedge\sqrt{n}}\sqrt{\log(e|\bd{S}|)}\\
& \leq \frac{2^{6}\|\hat q - q_*\|_\infty}{ \fr p_\wedge\sqrt{n}}\log(e|\bd{S}|)
\end{align*}
where we choose $p = \frac{1}{4}\mu_\wedge  \leq \frac{1}{4}\fr{p}_\wedge$ and the last inequality follows from the assumption that $n\geq \fr{p}_{\wedge}^{-2}$. 
\par To bound the variance, we use the same techniques as in \eqref{eqn:varV_separate_and_bd_kl} and conclude that for $n\geq 1$
\begin{align*}
\var(\widehat{\mathbf{T}}(\hat q) - \widehat{\mathbf{T}}(q^*))
&\leq 8\gamma^2\|\hat q - q_*\|_\infty^2  P(\Omega_{n,p}(\mu)^c) + \gamma^2 E\int_{0}^1 \int_{0}^1 2(D_1^2+D_2^2)dsdt\1_{\Omega_{n,p}(\mu)}\\
&\leq \frac{2^{7}\|\hat q - q_*\|_\infty^2}{\mu_\wedge^2n}\log(e|\bd{S}|)   + 24\|\Delta_v\|_{\infty}^2\sup_{s,t\in(0,1)}E\norm{\frac{dm}{d\mu(t)}}_{L^\infty(\mu)}^2\1_{\Omega_{n,p}(\mu)}\\
&\leq \frac{2^{7}\|\hat q - q_*\|_\infty^2}{\mu_\wedge^2n}\log(e|\bd{S}|)   + \frac{2^{10}\|\hat q - q_*\|_\infty^2}{\mu_\wedge n}\log(e|\bd{S}|).\\
&\leq  \frac{2^{11}\|\hat q - q_*\|_\infty^2}{\fr p_\wedge n}\log(e|\bd{S}|).
\end{align*}
This is the variance of $\bd H(\hat q)$ as $\cH(\hat q)$ is deterministic.

\subsection{Proof of Proposition \protect\ref{prop:recentered_hp_bound_chi2}}
\begin{proof}
Given Lemma \ref{lemma:V_integral_rep_chi2}, we directly apply the arguments in Appendix \ref{a_sec:proof:prop:recentered_op_high_prob_KL}. 
\par We have that w.p.1,
\begin{align*}
|V(s,a)|&\leq |V| \1_{\Omega_{n,p}(\mu)^c}+\gamma\sup_{s,t\in(0,1)} (|D_1|+|D_2)\1_{\Omega_{n,p}(\mu)}
\end{align*}
where $\mu = p_{s,a}$. Recall the choice $p \leq \frac{1}{4}\mu_\wedge = \frac{1}{4}p_{s,a\wedge}$. By Hoeffding's inequality and the union bound
\begin{align*}
P(|V(s,a)| > t)
    &\leq P(\Omega_{n,p}(p_{s,a})^c) + P\crbk{\gamma\sup_{s,t\in(0,1)} (|D_1|+|D_2|)>t,\Omega_{n,p}(p_{s,a})}\\
    &\leq P\crbk{\sup_{s'\in S}|p_{s,a,n}(s')-p_{s,a}(s')| > p} + P\crbk{\frac{5\gamma\|\hat q - q_*\|_\infty}{p_{s,a,\wedge} - p}\sup_{s\in\bd S}|m(s)|>t}\\
    &\leq \sum_{s\in\bd{S}}\crbk{P\crbk{|m(s)| > p} + P\crbk{\frac{8\|\hat q - q_*\|_\infty}{p_{s,a,\wedge}}|m(s)|>t}}\\
    &\leq 2|\bd{S}|\crbk{\exp\crbk{-2p^2n} + \exp\crbk{-\frac{p_{s,a,\wedge}^2 t^2n}{32\|\hat q - q_*\|_\infty^2}}}
\end{align*}
Then, as $\fr{p}_{\wedge} \leq p_{s,a,\wedge}$ for all $(s,a)\in\bd{S}\times\bd{A}$, by union bound
\[
P(\|V\|_\infty > t)\leq 2|\bd{S}|^2|\bd{A}|\crbk{\exp\crbk{-\frac{\fr{p}_{\wedge}^2n}{8}} + \exp\crbk{-\frac{\fr{p}_{\wedge}^2t^2n}{32\gamma^2\|\hat q - q_*\|_\infty^2}}}.
\]
We first control the first term to be less than $\eta/2$, which is implied by
\[
n\geq \frac{8}{\fr{p}_{\wedge}^2}\log(4|\bd{S}|^2|\bd{A}|/\eta). 
\]
Finally, the second term less than $\eta/2$ is implied by choosing
\[
t^2 = \frac{32\gamma^2\|\hat q - q^*\|}{\fr{p}_{\wedge}^2 n}\log(4|\bd{S}|^2|\bd{A}|/\eta).
\]
This proves the claimed result. 
\end{proof}

\subsection{Proof of Proposition \protect\ref{prop:high_prob_bound_empirical_Bellman_chi2}}

\begin{proof}
We recall the bound \eqref{eqn:Bellman_difference_dual_bound}. If $v(q)$ is essentially constant w.r.t. $p_{s,a}$, then $
\widehat{\bd{T}}(q)(s,a) = \cT(q)(s,a)$. Therefore, we then focus on the case that $v(q)$ is not essentially constant. 

\par Again, we fix $p\leq \frac{1}{4}\fr{p}_\wedge\leq \frac{1}{4}\mu_\wedge$ and thus on $\Omega_{n,p}(\mu)$, $\mu\sim \mu_n$ where $\mu = \nu_{s,a}$ or $p_{s,a}$. So, if $u$ is not essentially constant, by Assumption \ref{assump:delta_small_chi2} and Lemma \ref{lemma:alpha_large_chi2}, we have that

$$\sup_{\alpha\in\R}f(\mu_n,u,\alpha) = f(\mu_n,u,\alpha_n^*), \quad\sup_{\alpha\in\R}f(\mu,u,\alpha) = f(\mu,u,\alpha^*) $$
for some $\alpha_n^*,\alpha^* > \esssup_\mu u =: u_\vee$.  
\par Then, as in \eqref{eqn:Bellman_difference_dual_bound} we analyze 
\begin{align*}
\sup_{\alpha> u_\vee}\abs{f(\mu_n,u,\alpha) -f(\mu,u,\alpha)}.
\end{align*}
Since $\alpha> u_\vee$, $\mu[w^2]> 0$ and $f$ is differentiable in $\mu$ on $\Omega_{n,p}(\mu)$. By the mean value theorem, 
\begin{align*}
\abs{f(\mu_n,u,\alpha) -f(\mu,u,\alpha)} &=c(\delta) \frac{1}{2}\abs{\mu(\tau)[w^2]^{-\frac{1}{2}}m[w^2]}\\
&= \frac{1}{2}\abs{\frac{m[(\alpha-u)^2]}{\mu(\tau)[\alpha - u]}}
\end{align*}
for some $\tau\in[0,1]$ where we used \eqref{eqn:opt_cond_chi2} and $\mu(t) = t\mu+(1-t)\mu_n$ and $m = \mu-\mu_n$. 
\par We first consider when $\alpha > 2\norm{u}_\infty$,
\begin{align*}
&\sup_{\alpha> 2\norm{u}_\infty}\abs{f(\mu_n,u,\alpha) -f(\mu,u,\alpha)}\\
&\leq \sup_{\alpha> 2\norm{u}_\infty}\frac{1}{2}\abs{\frac{m[\alpha^2-2\alpha u + u^2]}{\alpha - \mu(\tau )[u]}}\\
&\leq \sup_{\alpha> 2\norm{u}_\infty}\abs{\frac{\alpha m[ u ]}{\alpha - \mu(\tau )[u]}} + \frac{1}{2}\abs{\frac{m[ u^2 ]}{\alpha - \mu(\tau )[u]}}\\
&\leq \sup_{\alpha> 2\norm{u}_\infty}\abs{\frac{(\alpha - \mu(\tau )[u] )m[ u ]}{\alpha - \mu(\tau )[u]}}+ \abs{\frac{\mu(\tau  )[u]m[ u ]}{\alpha - \mu(\tau )[u]}} + \frac{1}{2}\abs{\frac{m[ u^2 ]}{\alpha - \mu(\tau )[u]}}\\
&\stackrel{(i)}{\leq} |m[u]|+ \abs{\frac{\mu(\tau  )[u]}{\norm{u}_\infty}}\abs{m[ u ]} + \frac{1}{2}\frac{m[u^2]}{\norm{u}_\infty}\\
&\leq \norm{u}_\infty\sup_{y\in Y}|m(y)|
\end{align*}
where $(i)$ uses that $\alpha \geq 2\norm {u}_\infty$ and hence $\alpha - \mu(\tau  )[u]\geq \norm {u}_\infty$. 
\par On the other hand, if $u_\vee < \alpha\leq 2\norm{u}_\infty$
\begin{align*}
\sup_{u_\vee< \alpha \leq 2\norm{u}_\infty}\abs{f(\mu_n,u,\alpha) -f(\mu,u,\alpha)}& \leq \sup_{u_\vee< \alpha \leq 2\norm{u}_\infty} \frac{1}{2}\abs{\frac{m[(\alpha-u)^2]}{\mu(\tau)[\alpha - u]}}\\
&\leq \sup_{u_\vee< \alpha \leq 2\norm{u}_\infty} \frac{1}{2}\norminf{\alpha-u}\abs{\frac{m[\alpha-u]}{\mu(\tau)[\alpha - u]}}\\
&\leq \frac{3}{2}\norminf{u}\Linfnorm{\frac{dm}{d\mu(\tau)}}{\mu}\\
&\leq  \frac{3}{2}\norm{u}_\infty\frac{1}{\mu_\wedge - p}\sup_{y\in Y}|m(y)|\\
&\leq \frac{2\norminf{u}}{\mu_\wedge}\sup_{y\in Y}|m(y)|
\end{align*}
where the last two inequalities follow from Lemma \ref{lemma:equiv_of_meas_cond} and $p\leq \frac{1}{4}\mu_\wedge$. 
\par Therefore, we have
\begin{align*}
&P\crbk{\sup_{\alpha \in\R} \abs{f(\mu_n,u,\alpha) -f(\mu,u,\alpha)} > t}\\
&\leq P(\Omega_{n,p}(\mu)^c) +P\crbk{\sup_{\alpha >v_\vee} \abs{f(\mu_n,u,\alpha) -f(\mu,u,\alpha)} > t,\Omega_{n,p}(\mu)}\\
&\leq P\crbk{\sup_{y}|\mu_{n}(y)-\mu(y)| > p}+P\crbk{\frac{2\norminf{u}}{\mu_\wedge}\sup_{y\in Y} |m_{n}(y)| > t}\\
&\leq 2\sum_{y}\crbk{\exp(-2p^2n)+\exp\crbk{-\frac{\mu_\wedge^2t^2n }{2\norminf{u}^2}}}\\
&\leq 2|Y|\crbk{\exp(-2p^2n)+\exp\crbk{-\frac{\mu_\wedge^2t^2n }{2\norminf{u}^2}}}
\end{align*}
where we used Hoeffding's inequality and union bound. 
\par Therefore, going back to the DR Bellman operator setting, we choose $p = \frac{1}{4}\fr{p}_{\wedge}$. By union bound
\begin{align*} 
&P(\|\widehat{\mathbf{T}}(q)-\cT(q)\|_\infty > t)\\
&\leq P\crbk{\sup_{s,a}\sup_{\beta\in\R }|f(\nu_{s,a,n},id,\beta)-f(\nu_{s,a},id,\beta)|>\frac{t}{2}}\\
&\qquad +P\crbk{\sup_{s,a}\sup_{\alpha\in\R}|f(p_{s,a,n},v(q),\beta)-f(p_{s,a},v(q),\beta)|>\frac{t}{2}}\\ 
&\leq 2(|\bd{S}|^2|\bd{A}|+|\bd{S}||\bd{A}||\bd{R}|)\exp\crbk{-\frac{\fr{p}_{\wedge}^2n}{8}} + 2|\bd{S}||\bd{A}||\bd{R}|\exp\crbk{-\frac{\fr{p}_{\wedge}^2t^2n}{64\rmax^2}}\\
& \qquad  + 2|\bd{S}|^2|\bd{A}|\exp\crbk{-\frac{\fr{p}_{\wedge}^2t^2n}{64\gamma^2\norminf{q}^2}}. 
\end{align*}
We set each of the three terms to be less than $\eta/3$ and find that it suffices to have
\[
n\geq  \frac{8}{\fr{p}_{\wedge}^2}\log\crbk{12|\bd{S}||\bd{A}|(|\bd{S}|\vee |\bd{R}|)/\eta}
\]
and 
\[
t\geq \frac{8(\rmax+\gamma\spnorm{q})}{\fr{p}_{\wedge}\sqrt{n}}\sqrt{\log\crbk{6|\bd{S}||\bd{A}|(|\bd{S}|\vee |\bd{R}|)/\eta}}.
\]
This implies the statement of the proposition. 
\end{proof}

\section[Proof of Technical Lemmas: chi2 Case]{Proof of Technical Lemmas: $\chi_2$ Case}\label{a_sec:proof:lemma_chi2}

\subsection{Proof of Lemma \protect\ref{lemma:alpha_opt_equation_chi2}}
\begin{proof}First, we note that for every $u$ and $\mu$, $f$ is continuous in $\alpha$. Differentiate, we see that $f(\mu,u,\cd)$ is $C^1$ with derivative 
\begin{equation}\label{eqn:dual_func_alph_derivative_chi2}
\del_\alpha f(\mu,u,\alpha) = 1-c(\delta)\mu[w^2]^{-\frac{1}{2}}\mu[w]
\end{equation}
which is again continuous. Differentiate again, we get that
\begin{equation}\label{eqn:f_alpha_2_deriv_chi2}
\begin{aligned}
\del_\alpha\del_\alpha f(\mu,u,\alpha)&= c(\delta) \crbk{\mu[w^2]^{-\frac{3}{2}}\mu[w]^2 - \mu[w^2]^{-\frac{1}{2}}\mu[\1\set{\alpha > v}]}\\
&= c(\delta) \mu[w^2 ]^{-\frac{3}{2}}\crbk{\mu[w\1\set{\alpha > v}]^2 - \mu[w^2]\mu[\1\set{\alpha > v}^2]}\\
&\stackrel{(i)}{\leq} 0
\end{aligned}
\end{equation}
when $\alpha>\essinf_{\mu} u$, where $(i)$ follows from Jensen's inequality. Moreover, this expression is continuous for $\alpha > $ Therefore, $f$ is second differentiable and convex in $\alpha$ when $\alpha>\essinf_{\mu} u$. 
\par As we commented after Lemma \ref{lemma:strong_dual_chi2_duchi}, it suffices to optimize over $\alpha\geq \essinf_{\mu} u$. By the continuity of $f$ and $\del_\alpha f$ in $\alpha$ and convexity, if the optimizer $\essinf_{\mu} u < \alpha^*<\infty$, it must satisfies
$$ 0 = \del_\alpha f(\mu,u,\alpha^*) = 1-c(\delta)\mu[w^2]^{-\frac{1}{2}}\mu[w];$$ which is \eqref{eqn:opt_cond_chi2}.
\par Next, we handle the boundary cases $\alpha^* = \infty$ and $\alpha^* = \essinf_\mu u$. Notice that rewriting \eqref{eqn:opt_cond_chi2} as
$$\mu\sqbk{\crbk{\frac{w}{\mu[w]}}^2} = c(\delta)^2$$
we see that for $\delta > 0$, $\alpha^*\neq \infty$, because otherwise $\frac{w}{\mu[w]} = 1$ a.s.$\mu$. and the above equality cannot hold. 
\par On the other hand, if $\alpha^* =  \essinf_{\mu} u$, then \eqref{eqn:opt_cond_chi2} holds trivially with $w=0$. 
\par Then, we show that \eqref{eqn:worst_case_meas_chi2} is a worst-case measure. It suffices to check that $\mu^*[u] = f(\mu,u,\alpha^*)$ and $D_{\chi_2}(\mu^*\|\mu)= \delta$. We have that
\begin{align*}
    \mu^*[u] &= \frac{\mu[wu]}{\mu[w]}\\
    &= \alpha^* - \frac{\mu[(\alpha^*-u)\1\set{\alpha > u}(\alpha^*-u)]}{\mu[w]}\\
    &= \alpha^* - \frac{\mu[w^2]}{\mu[w]}\\
    &\stackrel{(i)}{=}\alpha^* - c(\delta)^2\mu[w] \\
    &\stackrel{(ii)}{=}\alpha^* - c(\delta)\mu[w^2]^{\frac{1}{2}} \\
    &= f(\mu,u,\alpha^*)
\end{align*} 
where $(i)$ and $(ii)$ follows from \eqref{eqn:opt_cond_chi2}. Moreover, by definition \eqref{eqn:def_chi2},
\begin{align*}
D_{\chi_2}(\mu^*\|\mu) 
&= \frac{1}{2}\mu\sqbk{\crbk{\frac{d\mu^*}{d\mu} -1}^2}\\
&= \frac{1}{2}\mu\sqbk{\crbk{\frac{w}{\mu[w]} -1}^2}\\
&= \frac{1}{2}\crbk{\frac{\mu[w^2]}{\mu[w]^2}+1-2}\\
&\stackrel{(i)}{=} \frac{1}{2}\crbk{c(\delta)^2-1}\\
&=\delta
\end{align*} 
again $(i)$ follows from \eqref{eqn:opt_cond_chi2}. 
\par Finally, clearly $\mu^*$ defined in \eqref{eqn:worst_case_meas_alpha_min_chi2} satisfies $\mu^*[u] = \essinf_\mu u = f(\mu,u,\alpha^*)$. So, to show that $\mu^*$ is a worst-case measure, it suffices to check that $D_{\chi_2}(\mu^*\|\mu)\leq \delta$. 
\par To show this, we observe that if $\alpha^* = \essinf_\mu u$, then by convexity we must have that for all sufficiently small $\epsilon > 0$, $\del_\alpha f(\mu,u,\alpha^*+\epsilon)\leq 0$. Otherwise, $\alpha^* = \essinf_\mu u$ cannot be optimal. In particular, let $w(\epsilon) = (\alpha^*+\epsilon - u)_+$, then by \eqref{eqn:dual_func_alph_derivative_chi2}, we have that 
$$ \mu[w(\epsilon)^2] \leq  c(\delta)^2\mu[w(\epsilon)]^2. $$
Note that if $\epsilon$ is sufficiently small, i.e. when $\alpha^*+\epsilon < u(s)$ for all $s\notin U$ and $\mu(s) > 0$, then $w(\epsilon) = \epsilon\1_U$. Therefore, we must have that 
$$ \mu[\1_U] \leq  c(\delta)^2\mu[\1_U]^2;$$
i.e. $\mu(U)^{-1} \leq c(\delta)^2$. 
With this bound, we now compute 
\begin{align*}
    D_{\chi_2}(\mu^*\|\mu) &= \frac{1}{2}\mu\sqbk{\crbk{\frac{\1_{U}}{\mu(U)} -1}^2} \\
    &= \frac{1}{2}\crbk{\frac{1}{\mu(U)}-1}\\
    &\leq \frac{1}{2}\crbk{c(\delta)^2-1}\\
    &= \delta. 
\end{align*}
\par Therefore, this proves Lemma \ref{lemma:alpha_opt_equation_chi2}. 
\end{proof}

\subsection{Proof of Lemma \protect\ref{lemma:alpha_large_chi2}}
\begin{proof}
\par If $u$ is $\mu$ essentially constant, then $\essinf_\mu u = \esssup_\mu u = \alpha^*$; i.e. the statement of Lemma \ref{lemma:alpha_large_chi2} holds. 

\par Next, we prove that if $u$ is not $\mu$ essentially constant, then $\delta < \frac{1}{2}\mu_\wedge$ implies $\alpha^*\geq\esssup_\mu u$.  To achieve this, we first show that $\alpha^* > \essinf_\mu u$ under these assumptions. 
\par We prove this by assuming $\alpha^* = \essinf_\mu u$ and raising a contradiction. By Lemma \ref{lemma:alpha_opt_equation_chi2}, $\mu^*$ defined in \eqref{eqn:worst_case_meas_alpha_min_chi2} is a worst-case measure. Hence, 
$$\begin{aligned}
    \delta &\geq D_{\chi_2}(\mu^*\|\mu) \\
    &= \frac{1}{2}\mu\sqbk{\crbk{\frac{\1_{U}}{\mu(U)} -1}^2} \\
    &= \frac{1}{2}\crbk{\frac{1}{\mu(U)}-1}\\
    &\stackrel{(i)}{\geq}\frac{1}{2}\frac{\mu_\wedge}{ 1-\mu_\wedge}
\end{aligned}$$
where $(i)$ follows from the assumption that $u$ is not $\mu$ essentially constant, so $$U = \set{s:\mu(s) > 0, u(s) = \essinf_\mu u}$$ cannot be of probability 1. In particular, by the definition of $\mu_\wedge$, $\mu(U)\leq 1-\mu_\wedge$. Therefore, rearrange terms, we have that
$$\frac{\delta}{\mu_\wedge}\geq \frac{1}{2}\frac{1}{1-\mu_\wedge}\geq \frac{1}{2};$$
i.e. $\delta\geq \frac{1}{2}\mu_\wedge$, contradicting our assumption. Therefore, $\alpha^* > \essinf_\mu u$. 
\par Using this, we then show that if $u$ is not $\mu$ essentially constant, $\delta < \frac{1}{2}\mu_\wedge$, and $\alpha^*>\essinf_\mu u$, then $\alpha^*\geq\esssup_\mu u$. 
\par We prove by contradiction, assuming that $\essinf_\mu u<\alpha^*\leq\esssup_\mu u$. Since $\alpha^* \leq\esssup_\mu u$, we must have that for some $s'\in \bd{S}$ s.t. $\mu(s') > 0$, $w(s') = (\alpha^*-u(s'))_+ = 0$. By Lemma \ref{lemma:alpha_opt_equation_chi2}, $\mu^*$ defined in \eqref{eqn:worst_case_meas_chi2} is a worst-case measure when $\alpha^*>\essinf_\mu u$. Moreover,
$$\begin{aligned}
    \delta &=D_{\chi_2}(\mu^*\|\mu) \\
    &= \frac{1}{2}\mu\sqbk{\crbk{\frac{w}{\mu[w]} -1}^2} \\
    &\geq \frac{1}{2}\mu(s')
\end{aligned}$$
contradicting the assumption. Therefore, $\alpha^* > \esssup_\mu u$. This completes the proof of Lemma \ref{lemma:alpha_large_chi2}. 
\end{proof}

\subsection{Proof of Lemma \ref{lemma:V_integral_rep_chi2}}
\begin{proof}
\par By assumption, we are interested in empirical measures that satisfy $\Omega_{n,p}(\mu)$ (c.f. \eqref{eqn:Omega_n,p_mu}) with $p\leq\frac{1}{4}\mu_\wedge$. Then, by Lemma \ref{lemma:equiv_of_meas_cond}, we have that $\mu\sim\mu_n\sim\mu(t)$ on $\Omega_{n,p}(\mu)$. 
\par We first fix $s\in[0,1]$. Let us denote $v_{s,\vee} := \esssup_\mu v_s$. Recall that by Lemma \ref{lemma:alpha_large_chi2}, when $\delta < \frac{1}{2}\mu_\wedge$, it suffices to optimize the Lagrange multiplayer in $[v_{s,\vee},\infty)$.  We have
\[
\del_t  f(\mu(t),v_s,\alpha) = -\frac{1}{2}c(\delta)m[w_s^2]\mu(t)[w_s^2]^{-\frac{1}{2}}
\]
where $w_s = (\alpha-v_s)_+ = \alpha-v$. It is not hard to see that $\del_t  f(\mu(t),v_s,\alpha)$ is continuous on $[0,1]\times [v_{s,\vee},\infty)$ even if $v_s$ is essentially constant (in this case we note that $\del_t  f(\mu(t),v_s,v_{s,\vee}) = 0$). 
\par Next define
\[
\Theta(t):=\argmax{\alpha > v_{s,\vee}}f(\mu(t),v_s,\alpha). 
\]
We discuss two cases: 
\begin{enumerate}
    \item If $v_s$ is $\mu$ essentially constant, then for $\alpha > v_{s,\vee}$
\[
f(\mu(t),v_s,\alpha) = \alpha - c(\delta)(\alpha - v_{s,\vee}) = (1-c(\delta) )\alpha + c(\delta)v_{s,\vee}.
\]
Since $c(\delta) = 1+2\delta > 0$, this is maximized at $\Theta(t) = \set{v_{s,\vee}}$. 
\item $v_s$ is not $\mu$ essentially constant. Note that then by Lemma \ref{lemma:alpha_large_chi2}, $\alpha > \esssup_\mu v_s$, $\alpha > v_s$ a.s.$\mu$ (hence $\mu(t)$). Recall that the second derivative in \eqref{eqn:f_alpha_2_deriv_chi2},
\begin{equation}\label{eqn:f_strict_convex}
\begin{aligned}
&\del_\alpha\del_\alpha f(\mu(t),v_s,\alpha)\\
&= c(\delta) \mu(t)[w_s^2 ]^{-\frac{3}{2}}\crbk{\mu(t)[w_s\1\set{\alpha > v_s}]^2 - \mu(t)[w_s^2]\mu(t)[\1\set{\alpha > v_s}^2]}\\
&= c(\delta) \mu(t)[w_s^2 ]^{-\frac{3}{2}}\crbk{\mu(t)[w_s]^2 - \mu(t)[w_s^2]}\\
&< 0
\end{aligned}
\end{equation}
where the last inequality follows from that $w_s$ is not $\mu(t)$ constant, hence the variance is positive. So, in this case $f(\mu(t),v_s,\cd)$ is strictly concave. Thus, $\Theta(t)$ is a singleton. 
\end{enumerate}
\par Therefore, in both case, $\Theta(t)$ is a singleton. We conclude that by \citet[Theorem 7.21]{sharpioBookSP}, the following derivative exists
\begin{equation}\label{eqn:hst_t_deriv_chi2}
\begin{aligned}
d_t\sup_{\alpha > v_{s,\vee}} f(\mu(t),v_s,\alpha) &= \sup_{\alpha\in \Theta(t)}\del_t f(\mu(t),v_s,\alpha) \\
&= \del_t f(\mu(t),v_s,\alpha^*_{s,t})\\
&= -\frac{1}{2}c(\delta)m[w_s^2]\mu(t)[w_s^2]^{-\frac{1}{2}}.
\end{aligned}
\end{equation}
where it is understood that $w_s = (\alpha^*_{s,t}-v_s)_+ = \alpha^*_{s,t}-v_s$. Therefore, we have shown that $t\ra h(s,t)$ is $C^1(0,1)\cap C[0,1]$. Hence,
\begin{align*}|V(s,a)| 
    &= \gamma \abs{h(1,0) - h(0,0) -h(1,1) + h(0,1)}\\
    &= \gamma \abs{\int_0^1\del_t h(1,t) - \del_t h(0, t)dt}
\end{align*}
\par Next, we show that for any fixed $t$, there exists a mapping $s\ra D_s\del_th(s,t)$ s.t. \eqref{eqn:del_t_hst_weak_derivative} holds. 
\par We note that by Lemma \ref{lemma:alpha_large_chi2}, $\alpha_{s,t}^* = v_{s,\vee}$ only when $v_s$ is essentially constant. Again, assuming that at least one of $v(\hat q)$ and $v( q^*)$ is not $\mu$ essentially constant, as in the proof of Proposition \ref{prop:recentered_op_bias_and_var_KL}, this can only happen at one particular $s = s^*$. 
\par We separately consider these two cases: 
\par \textbf{Case 1: }$v_s$ is never essentially constant for all $s\in[0,1]$. 
\par In this case, $\alpha_{s,t}^* > v_{s,\vee}$ for all $s\in[0,1]$.  Note that $w_s = \alpha_{s,t}^* - v_{s} > 0$. So, if on $\Omega_{n,p}(\mu)$,  $\alpha_{s,t}^*$ is $C^1(0,1)\cap C[0,1]$ in $s$, then by chain rule, $s\ra \del_th(s,t) $ in \eqref{eqn:hst_t_deriv_chi2} is $C^1(0,1)\cap C[0,1]$. 

\par As in the proof of Proposition \ref{prop:recentered_op_bias_and_var_KL}, we show differentiability of $s\ra\alpha_{s,t}^*$ by invoking the implicit function theorem. By the strict convexity \eqref{eqn:f_strict_convex}, $\alpha_{s,t}^*$ is the unique solution to the optimality condition \eqref{eqn:opt_cond_chi2}
$$ 0 = c(\delta)^2\mu(t)[w_s]^2 -\mu(t)[w_s^2] =: F(s,\alpha^*_{s,t}). $$
\par Since $F$ is infinite smooth, the implicit function theorem implies that $\alpha_{s,t}^*$ is $C^1(0,1)\cap C[0,1]$ and $s\ra \del_th(s,t)$ is $C^1(0,1)\cap C[0,1]$. 
\par We compute the derivative $\del_s\del_t h$ in this case. Recall $\Delta_v= v(\hat{q})-v(q^*)$. Differentiate w.r.t. $s$ on both side, we have
$$
\begin{aligned}
0 &= c(\delta^2)2\mu(t)[w_s]\mu(t)[\del_s\alpha^*_{s,t} - \Delta_v] - 2\mu(t)[w_s(\del_s\alpha^*_{s,t} - \Delta_v)]. 
\end{aligned}
$$
Rearranging terms, we have 
\begin{equation*}
\begin{aligned}
\del_s\alpha^*_{s,t} &= \frac{c(\delta)^2\mu(t)[w]\mu(t)[\Delta_v] - \mu(t)[w_s\Delta_v]}{(c(\delta)^2 - 1)\mu(t)[w_s]}
\end{aligned}
\end{equation*}
This gives \eqref{eqn:alpha*_derivative_chi2}. Moreover, when $\alpha_{s,t}^* > 0$, 
\begin{equation*}
\begin{aligned}
\del_s\del_t h(s,t) &= \frac{1}{2} c(\delta)  \mu(t)[w_s^2]^{-\frac{3}{2}}\mu(t)[\Delta_vw_s^{2}]m[w_s^{2}] - c(\delta)\mu(t)[w_s^{2}]^{-\frac{1}{2}} m[\Delta_vw_s]\\
&\quad +\del_s\alpha^*_{s,t}\crbk{-c(\delta)\mu(t)[w_s^2]^{-\frac{1}{2}} m[w_s] +\frac{1}{2}c(\delta) \mu(t)[w_s^2]^{-\frac{3}{2}}\mu(t)[w_s]m[w_s^2]}\\
&= c(\delta) \mu(t)[w_s^{2}]^{-\frac{1}{2}}\mu(t)[w_s]\crbk{ \frac{\mu(t)[\Delta_v w_s]m[w_s^2]}{2\mu(t)[w_s]\mu(t)[w_s^2]} -  \frac{m[\Delta_vw_s]}{\mu(t)[w_s]}}\\
&\quad +c(\delta) \mu(t)[w_s^{2}]^{-\frac{1}{2}}\mu(t)[w_s] \del_s\alpha^*_{s,t}\crbk{ \frac{m[w_s^2]}{2\mu(t)[w_s^2]}- \frac{m[w_s]}{\mu(t)[w_s]}}\\
&\stackrel{(i)}{=}  \sqbk{\frac{\mu(t)[\Delta_vw_s]m[w_s^2]}{2\mu(t)[w_s]\mu(t)[w_s^2]} -  \frac{m[\Delta_vw_s]}{\mu(t)[w_s]}} +\del_s\alpha^*_{s,t}\crbk{ \frac{m[w_s^2]}{2\mu(t)[w_s^2]}- \frac{m[w_s]}{\mu(t)[w_s]}} 
\end{aligned}
\end{equation*}
where $(i)$ uses the optimality equation \eqref{eqn:opt_cond_chi2}. This is consistent with  \eqref{eqn:del_s_del_t_hst_decomp_chi2}.
\par \textbf{Case 2: }There is a unique $s^*\in[0,1]$ s.t. $v_s$ is essentially constant. 
\par As in the proof of Proposition \ref{prop:recentered_op_bias_and_var_KL}, in this case, the previous argument implies that  $s\ra \del_th(s,t)$ is $C^1(0,s^*)$, $C^1(s^*,1)$, and continuous at $0,1$. The derivative is also given by \eqref{eqn:del_s_del_t_hst_decomp_chi2}. 
\par Again, we show the existence of $D_s\del_th$ that satisfy \eqref{eqn:del_t_hst_weak_derivative}. Observe that if $s\ra \del_th(s,t)$ is continuous at $s^*$, then applying the fundamental theorem of calculus on the interval $[0,s^*]$ and $[s^*,1]$ separately, we will have that
\[
\del_th(1,t) - \del_t h(0,t) = \int_0^{s^*} \del_s\del_th(s,t) ds+ \int_{s^*}^1 \del_s\del_th(s,t) ds. 
\]
Hence, taking $D_s\del_t h(s,t) = \del_s\del_t h(s,t)$ for every $s\neq s^*$ and $D_s\del_t h(s^*,t) = 0$ will suffice to produce \eqref{eqn:del_t_hst_weak_derivative}. 
\par It is left to check the continuity at $s^*$ of 
$$ \del_th(s,t) = \del_t f(\mu(t),v_s,\alpha^*_{s,t})= -\frac{1}{2}c(\delta)m[w_s^2]\mu(t)[w_s^2]^{-\frac{1}{2}}$$
from \eqref{eqn:hst_t_deriv_chi2}. Note that on $\Omega_{n,p}(\mu)$, for all $s\in[0,1],\alpha\geq v_{s,\vee}$, \begin{align*}\abs{-\frac{1}{2}c(\delta)m[\alpha-v_s]\mu(t)[(\alpha-v_s)^2]^{-\frac{1}{2}}} &\leq \abs{\frac{1}{2}\mu(t)[(\alpha-v_s)^2]^{\frac{1}{2}}\Linfnorm{\frac{dm}{d\mu(t)}}{\mu}}\\
& \stackrel{(i)}{\leq} \abs{\frac{1}{2}\Linfnorm{\alpha-v_s}{\mu}\frac{1}{\mu_\wedge - p}}\\
&\leq  \abs{\frac{1}{2}\Linfnorm{\alpha-v_s}{\mu}\frac{1}{\frac{3}{4}\mu_\wedge}}
\end{align*}
where $(i)$ follows from Lemma \ref{lemma:equiv_of_meas_cond}. 
Also, $\del_t h(s^*,t) = 0$. Therefore, if $\del_s\alpha^*_{s,t}\ra v_{s^*,\vee}$ as $s\ra s^*$, then $\Linfnorm{\alpha-v_s}{\mu}\ra 0$ as $s\ra s^*$, implying continuity at $s^*$. 
\par It is left to check that $\del_s\alpha^*_{s,t}\ra v_{s^*,\vee}$ as $s\ra s^*$. To prove this, we assume to the contrary that there is a subsequential limit $\alpha_{s_n,t}^*\ra \beta + v_{s^*,\vee}$ for some sequence $s_n\ra s^*$ and $\beta > 0$. But by Lemma \ref{lemma:alpha_opt_equation_chi2}, we must have that $$0 = \lim_{n\ra\infty}c(\delta)^2\mu(t)[\alpha^*_{s_n,t}-v_{s_n}]^2 -\mu(t)[(\alpha^*_{s_n,t}-v_{s_n})^2] = \delta\beta$$
raising a contradiction. This implies that $s\ra\del_th(s,t)$ is continuous at $s^*$, and hence \eqref{eqn:del_t_hst_weak_derivative} holds with  $D_s\del_t h(s,t) = \del_s\del_t h(s,t)$ for every $s\neq s^*$ and $D_s\del_t h(s^*,t) = 0$. 
\par Therefore, in both cases \eqref{eqn:V_mixed_partials_bd_chi2} holds with $\abs{D^2 h(s,t)}$ is given by \eqref{eqn:del_s_del_t_hst_decomp_chi2}. This gives the claim of the lemma. 
\end{proof}

\end{document}